\documentclass[mnsc,sglanonrev]{preprint}
\renewcommand{\theARTICLETOP}{}
\RRHSecondLine{}
\LRHSecondLine{} 

\usepackage[T1]{fontenc}
\usepackage{lmodern}

\RequirePackage{endnotes}

\OneAndAHalfSpacedXI

\usepackage{algorithm}
\usepackage[
    noEnd=false,
    indLines=true
]{algpseudocodex}
\algrenewcommand\algorithmiccomment[1]{\quad\quad\quad\quad{$\triangleright$ \it \color{gray} #1}}

\usepackage{tikz}
\usepackage{fontawesome5}
\definecolor{mainblue}{RGB}{31,96,135}
\definecolor{accentred}{RGB}{178,20,20}
\definecolor{softgreen}{RGB}{135,190,120}

\usepackage{multirow}
\usepackage{makecell}
\usepackage{subcaption}
\usepackage{tabularx}

\usepackage{natbib}
 \bibpunct[, ]{(}{)}{,}{a}{}{,}\def\bibfont{\small}

\usepackage{amsmath, amssymb, amsfonts}
\usepackage{mathtools}
\DeclarePairedDelimiter{\norm}{\lVert}{\rVert}

\usepackage{bm}

\usepackage{bbm}
\newcommand{\indc}{{\mathbbm{1}}}
\usepackage{tikz}
\usetikzlibrary{shapes.geometric, shapes.arrows, arrows.meta, positioning, calc, shadows}

\usepackage{booktabs}

\usepackage{etoolbox}
\makeatletter
\patchcmd{\th@TH}{\hspace*{1em}}{}{}{}
\patchcmd{\th@TH}{\hspace*{1em}}{}{}{}
\patchcmd{\th@EX}{\hspace*{1em}}{}{}{}
\patchcmd{\th@EX}{\hspace*{1em}}{}{}{}
\patchcmd{\th@THkey}{\hspace*{1em}}{}{}{}
\patchcmd{\th@THkey}{\hspace*{1em}}{}{}{}
\patchcmd{\th@EXkey}{\hspace*{1em}}{}{}{}
\patchcmd{\th@EXkey}{\hspace*{1em}}{}{}{}
\patchcmd{\proof}{\hspace*{1em}}{}{}{}
\makeatother
\usepackage{enumitem}

\usepackage[dvipsnames]{xcolor}
\usepackage{colortbl}
\definecolor{tablebestbg}{HTML}{D7EAF8}
\definecolor{tablesecondbg}{HTML}{FBE5C8}
\newcommand{\bestnum}[1]{\cellcolor{tablebestbg}#1}
\newcommand{\secondnum}[1]{\cellcolor{tablesecondbg}#1}

\makeatletter
\renewcommand\paragraph{\@startsection{paragraph}{4}{\z@}{0pt}{-0.5em}{\normalfont\normalsize\bfseries}}
\makeatother

\definecolor{refblue}{HTML}{003366}
\usepackage[hidelinks]{hyperref}
\DeclareRobustCommand{\ConfRO}{\texorpdfstring{\ifmmode\mathrm{ConfRO}\else\textup{ConfRO}\fi}{ConfRO}}
\DeclareRobustCommand{\ConfRS}{\texorpdfstring{\ifmmode\mathrm{ConfRS}\else\textup{ConfRS}\fi}{ConfRS}}
\hypersetup{
    pdftitle={Conformal Robustness in Prediction-Driven Decision-Making},
    pdfauthor={Lingjie Zhao, Hansheng Jiang, Wei Qi},
    colorlinks,
    breaklinks,
    linkcolor=Mahogany,
    urlcolor=black,
    anchorcolor=black,
    citecolor=refblue
}

\EquationsNumberedThrough    

\TheoremsNumberedThrough     \ECRepeatTheorems

\begin{document}

\RUNAUTHOR{Zhao, Jiang, and Qi}

\RUNTITLE{Conformal Robustness in Prediction-Driven Decision-Making}

\TITLE{Conformal Robustness in Prediction-Driven Decision-Making}

\ARTICLEAUTHORS{
\AUTHOR{Lingjie Zhao}
\AFF{Department of Industrial Engineering, Tsinghua University, Beijing 100084, China \\ \EMAIL{zhaolj25@mails.tsinghua.edu.cn}}

\AUTHOR{Hansheng Jiang}
\AFF{Rotman School of Management, University of Toronto, Toronto, Ontario M5S 3E6, Canada
\\ \EMAIL{hansheng.jiang@rotman.utoronto.ca}}

\AUTHOR{Wei Qi}
\AFF{Department of Industrial Engineering, Tsinghua University, Beijing 100084, China\\ Desautels Faculty of Management, McGill University, Montreal, Quebec H3A 1G5, Canada\\ \EMAIL{qiw@tsinghua.edu.cn}}
\vspace{-1em}

} 

\ABSTRACT{Modern prediction-driven decision systems often rely on black-box predictors, but a point forecast alone does not provide the uncertainty scale required for robust downstream decision-making. We build a score-calibrated robustness framework that converts any fixed point predictor into a decision-relevant uncertainty representation through distribution-free conformal calibration. We use the conformal score, rather than a particular uncertainty set, as the primitive unit of robustness. The same score determines coverage-calibrated uncertainty sets for reliability-based robust optimization and normalizes target violations in a target-oriented formulation, Conformal Robust Satisficing. This formulation induces a conformal fragility measure that quantifies how rapidly performance deteriorates as the realized parameter departs from the forecast on the conformal score scale. We establish decision-efficiency bounds separating the effects of prediction accuracy and score design, and data-driven target-violation certificates for satisficing decisions. For objective-uncertainty problems under standard convexity and duality conditions, we show that the reliability-based and target-oriented formulations parameterize the same score-calibrated robust decision frontier. This equivalence yields a data-driven mapping between reliability levels and acceptable targets and characterizes the marginal cost of robustness. Synthetic experiments validate the theoretical guarantees and illustrate the reliability--target correspondence. A real-data online-grocery case study demonstrates how the interface combines deep-learning demand forecasts with tractable inventory optimization, thereby improving reliability and reducing operational costs. Overall, our work shows that conformal scores endow fixed black-box predictors with an interpretable uncertainty scale for downstream decision-making while enabling reliability guarantees, acceptable-target selection, and fragility analysis within a unified framework.
}


\maketitle

\section{Introduction}

Abundant contextual data have made prediction-driven decision-making a central paradigm in modern operations. Increasingly, predictions are produced by high-capacity artificial intelligence models that can extract information at scale from large, high-dimensional, and unstructured data sources, including text, images, and videos \citep{jordan2015machine,lecun2015deep}. This predictive power, however, often comes at the cost of interpretability and decision reliability. A downstream decision maker may receive a point forecast with strong empirical predictive performance, but without an interpretable characterization of its forecast error or a tractable robustness representation that can be incorporated into an optimization model \citep{rudin2019stop}.

In many operational settings, prediction-driven decision-making is implemented through a modular pipeline: a forecasting model is placed upstream of an optimization model. Given contextual information $\bm{z}$, a possibly complex predictor $\hat f$ produces a point forecast $\hat{\bm d}=\hat f(\bm z)$ of the uncertain parameter $\bm d$, and the downstream optimization model uses $\hat{\bm d}$ to choose an operational decision. This modular design is attractive because it allows firms to exploit rich product, customer, temporal, and spatial information using modern statistical and machine-learning tools. It also preserves the functional division of labor commonly observed in cross-team collaborations within firms \citep{papier2021effect}. Yet this separation creates a fundamental robustness question: predictors are typically trained to reduce forecast loss, but forecast loss alone does not provide a decision-relevant \emph{scale} for uncertainty. \emph{How can we leverage the power of existing predictors while preserving the reliability levels and performance targets required by decision makers?}

This missing \emph{scale} is central to robust decision-making. In non-contextual and prediction-free settings, robust optimization protects a decision over an uncertainty set whose radius, budget, or geometry encodes the desired level of protection \citep{ben2009robust}. \emph{Robust satisficing} starts from a different managerial primitive: an acceptable target and a fragility measure that quantifies how quickly the target may be violated as uncertainty grows \citep{brown2009satisficing}. In both views, the protection level is meaningful only after uncertainty has been expressed on an appropriate scale. In prediction-driven applications, however, this scale is not identified by the point forecast alone. Forecast errors may be heteroscedastic, asymmetric, and context-specific; moreover, the predictor may be a black box whose internal structure does not yield either an interpretable error model or an optimization-friendly uncertainty set \citep{tyralis2024review}. Thus, the central difficulty is not merely that forecasts may be inaccurate but that point forecasts do not specify how their inaccuracies should be measured, calibrated, and used for robust decision-making.

Existing robust decision frameworks provide important building blocks, but the post-prediction decision challenge posed by arbitrary black-box predictors has yet to be fully resolved. An emerging literature uses conformal prediction, a modern statistical calibration method, to construct uncertainty sets for optimization problems \citep{vovk2012conditional,johnstone2021conformal}, a class of approaches that we collectively frame as Conformal Robust Optimization (\ConfRO{}). These advances show that predictive uncertainty can be translated into statistically calibrated protection for downstream optimization. However, existing uses of conformal calibration remain primarily reliability-based or uncertainty-set-centric: the conformal score is used to size an uncertainty set, but not as a decision primitive that \emph{jointly} governs reliability, target violation, and fragility. In particular, the integration of conformal scores with acceptable targets and fragility measures has not been explored in this stream of literature. As a result, reliability-based robust optimization and target-oriented robust satisficing remain expressed in different coordinates. What is missing is a score-calibrated frontier that translates between these coordinates in prediction-driven decision-making, assigning each reliability level an implied acceptable target, assigning each acceptable target an implied reliability level, and quantifying the marginal cost of moving along the frontier.

We bridge this gap by using the conformal score not merely as a calibration tool but as the common unit of robustness. Given a point predictor $\hat f$, the conformal score $s_{\hat f}(\bm z,\bm d)$ measures the discrepancy between the realized parameter $\bm d$ and its forecast $\hat f(\bm z)$. Standard split conformal calibration computes a quantile $\eta_\alpha$ from the empirical distribution of calibration scores, yielding the context-dependent uncertainty set $\mathcal U_\alpha(\bm z)=\{\bm d:s_{\hat f}(\bm z,\bm d)\le \eta_\alpha\}$, which enjoys finite-sample marginal coverage under exchangeability.
Our key step is to use the same score to normalize target violations in a parameter-uncertainty robust satisficing model. A decision is fragile if a small score-measured deviation from the forecast can generate a large violation of the acceptable target. Thus, the conformal score plays two roles simultaneously: it calibrates the uncertainty set used for reliable optimization and defines the error scale on which target-oriented fragility is assessed.

This perspective encompasses two complementary implementations. Conformal robust optimization takes a reliability level $\alpha$ as input and optimizes worst-case performance over the conformal uncertainty set $\mathcal U_\alpha(\bm z)$. Conformal Robust Satisficing (\ConfRS{}) instead takes an acceptable target $\tau$ as input and minimizes the fragility parameter $k$ required to ensure
\[
    a(\bm x,\bm d)-\tau
    \le
    k\,\mathfrak{s}_{\hat f}(\bm d,\hat{\bm d}),
    \qquad \forall \bm d\in\mathfrak D,
\]
where $a(\bm x,\bm d)$ is the objective value and $\mathfrak{s}_{\hat f}$ denotes the score-induced deviation from the forecast. The optimal value $k^*$ measures the worst-case target violation per unit of score deviation; smaller $k^*$ therefore corresponds to lower fragility. Although \ConfRO{} and \ConfRS{} are motivated by different managerial objectives, we show that they are connected through a common score-calibrated robustness frontier.

Our main contributions are as follows.
\begin{enumerate}
    \item
    \emph{A post-prediction robustness interface based on the conformal score function.}
    Given any fitted contextual point predictor, the conformal score evaluates predictive error through a user-defined conformal score function. We show that the finite-sample coverage guarantee delivered by conformal calibration of these scores can be translated into guarantees for the resulting optimization decisions. This perspective connects to prior work that uses conformal prediction to construct uncertainty sets, while extending the role of conformal scores beyond set construction: they serve as a modular, post-prediction interface through which predictive uncertainty is calibrated and propagated to robustness in decision-making.

    \item
    \emph{A score-calibrated, target-oriented formulation with theoretical guarantees.}
    We introduce Conformal Robust Satisficing (\ConfRS{}), a target-oriented counterpart to conformal robust optimization that can also be built around black-box forecasts. Whereas \ConfRO{} takes a reliability level as input and optimizes over a prediction-centered uncertainty set, \ConfRS{} takes an acceptable target as input and minimizes the worst-case target violation per unit of score deviation. Furthermore, we show that our \ConfRS{} formulation yields a valid conformal \emph{fragility measure} in the prediction-driven decision-making paradigm. For convex \ConfRO{} programs satisfying the stated metric-score and outer-duality conditions and allowing uncertainty in the objective and constraints, we derive a decision-efficiency bound that separates the calibrated radius from objective and dual-weighted constraint sensitivities. We also establish data-driven target guarantees for \ConfRS{} and complement the framework with a score-selection procedure that preserves finite-sample validity.

    \item
    \emph{Reliability--target translation and marginal cost of reliability.}
    For problems in which uncertainty enters only the objective function and standard regularity conditions hold, we establish matched decision correspondences between \ConfRO{} and \ConfRS{}, identify when their optimal decision sets coincide, and derive a mapping between the coverage level and the satisficing target through the score distribution and the relevant value-function sensitivity. We also derive an explicit characterization of the marginal cost of score-radius robustness and reliability level. We show that the optimal \ConfRS{} fragility is the marginal cost of enlarging the score radius in the matched \ConfRO{} problem, and that the marginal cost of increasing reliability is the matched fragility divided by the local density of the conformal score distribution.

    \item
    \emph{Numerical studies with an application to a real-world multiperiod inventory problem in online grocery.}
    We validate the framework on a fractional knapsack benchmark, confirming the theoretical guarantees and realized-utility gains of 23.2\%--97.7\% over $k$-nearest-neighbor robust optimization (KNN-RO) and $k$-means robust optimization (KMeans-RO) baselines while nearly matching predict-then-optimize (PTO) performance under robust protection. The benchmark also illustrates an interesting step-like pattern through an explicit parameter mapping between \ConfRO{} and \ConfRS{}.
    We further conduct a case study on multiperiod inventory management in online grocery using large-scale real-world data and  industrial deep-learning demand forecasting models. The case study benchmarks \ConfRO{} across score geometries against baselines and implements \ConfRS{} by leveraging the additive period-cost structure to allocate a target across the planning horizon. Numerical studies also yield several implementation guidelines: strong base forecasts should be prioritized, flexible score geometries can reduce unnecessary conservatism, and residual scaling can be counterproductive when secondary error modeling adds noise.

\end{enumerate}

The remainder of the paper is organized as follows. Section~\ref{sec:literature} reviews the related literature. Section~\ref{sec:cus_framework} introduces the score-based calibration layer and establishes its basic properties. Section~\ref{sec:cro_crs_framework} presents \ConfRO{} and \ConfRS{}, and Section~\ref{sec:ConfRO_ConfRS_equivalence} studies their decision-level correspondence and the conditions for equivalence, characterizes marginal robustness, and develops the score-selection procedure. Section~\ref{sec:experiments} presents the numerical studies, and Section~\ref{sec:conclusion} concludes.

\section{Related Literature}
\label{sec:literature}

The interaction between prediction and optimization is central to contextual decision-making. The classical predict-then-optimize pipeline first estimates unknown problem parameters and then plugs the estimates into a deterministic optimization problem. Although simple and modular, this approach can amplify prediction error through the optimization step, a phenomenon closely related to the optimizer's curse \citep{smith2006optimizer}. Prediction-driven decision-making and post-prediction decision performance have been studied in various settings, including logistics, healthcare, and pricing \citep{liu2021time,hu2025prediction,albert2025post}. Smart predict-then-optimize and decision-focused learning address such prediction-optimization mismatch by training predictors with losses that reflect downstream decision quality \citep{elmachtoub2022smart,mandi2022decision}.
Instead of modifying the predictor's training objective, a complementary prediction-to-decision perspective arises in the learning-augmented algorithms literature \citep{mitzenmacher2022algorithms}, where predictions are used to improve algorithm performance even when the predictions are inaccurate. For example, \citet{jin2022online} study the consistency-robustness tradeoff in online matching with predictions. Our approach is closer in spirit to this latter perspective: we take the predictor as given and fixed; however, rather than exploiting problem-specific algorithmic structure, our goal is to provide a general-purpose robustness interface for a broad class of downstream decision-making problems.

Toward robustness in decision-making, robust optimization provides a principled approach by optimizing against the worst case over a prescribed uncertainty set. Foundational work on robust optimization designs uncertainty sets with computationally tractable geometries, including polyhedral, ellipsoidal, cardinality-constrained, and norm-constrained sets \citep{ben2009robust}. Recent advances further incorporate contextual information into robust optimization, often by encoding context as covariate vectors that enter the predictor, the uncertainty model, or the downstream decision model. Our perspective is instead explicitly \emph{post-prediction}: contextual information is absorbed by an upstream point predictor, which may be black-box and may take unstructured inputs, but the downstream decision model need not access that context directly. In this setting, the predictor's uncertainty and accuracy, though vital to decision-making, are not known a priori. Conformal prediction is therefore particularly well suited to our post-prediction perspective, as it provides a simple, predictor-agnostic procedure that wraps around any predictor while retaining finite-sample validity \citep{angelopoulos2023conformal}. Accordingly, the \ConfRO{} presented in this paper is aligned with an emerging literature that uses conformal prediction to construct uncertainty sets for optimization (\citealp{sun2023predict,patel2024conformal, chenreddy2024end,cai2025out}). What distinguishes this work is that the conformal score is treated as the central primitive: the same score determines uncertainty-set geometry, coverage calibration, and decision fragility. This score-calibrated view extends the conformal perspective beyond uncertainty set construction to a target-oriented satisficing formulation, \ConfRS{}, and identifies when reliability-parameterized and target-parameterized decisions lie on the same frontier.

The notion of \emph{satisficing}, dating back to \cite{simon1956rational}, describes a decision rule that selects an alternative meeting an aspiration level rather than exhaustively searching for an optimum. Building on this idea, robust satisficing asks whether a decision can meet a prescribed acceptable target with low fragility \citep{brown2009satisficing}.  This target-oriented perspective has recently been developed \emph{primarily for distributional ambiguity} \citep{long2023robust}, with applications in several operational problems \citep{zhou2022advance,cui2023target,sim2024analytics,fu2025location,ding2026robust}. The idea of satisficing is also studied in the context of the exploration–exploitation trade-off in online learning \citep{feng2024satisficing}. \citet{sim2024analytics} develop a residual-based distributionally robust satisficing model based on a parametric prediction model and introduce an estimation-fortification step to address uncertainty in the prediction coefficients. Our \ConfRS{} component instead treats an arbitrary fitted point predictor as fixed, measures realized-parameter deviations through a context-indexed score, and uses conformal calibration to provide finite-sample target certificates and a reliability--target interpretation; see Section~\ref{subsubsec:crs_vs_drs} for a further comparison.

Robust satisficing under \emph{parameter uncertainty} has received comparatively limited attention relative to its distributional counterpart. Related parameter-uncertainty perspectives include info-gap decision theory \citep{benhaim2006info} and joint estimation and robustness optimization (JERO) \citep{zhu2022joint}. These approaches rely on either a prescribed uncertainty horizon or an explicit estimation model, whereas our framework operates on the prediction-error scale of an arbitrary fixed predictor. Under suitable conditions, distributionally robust optimization (DRO) and distributionally robust satisficing can share the same solution family after translating between ambiguity radii and acceptable targets \citep{wang2025equivalence}. Our correspondence result in Section~\ref{sec:ConfRO_ConfRS_equivalence} aligns in spirit with this result in the distributional setting, but it differs in two ways: it quantifies uncertainty through the conformal score and incorporates contexts explicitly with arbitrary black-box predictors. Furthermore, we build on this correspondence to show that the score-induced fragility function is exactly the marginal cost of robustness and to characterize how the score distribution affects the value function's sensitivity to changes in the reliability level. Table~\ref{tab:closest_positioning} in Appendix~\ref{app:supplementary_discussion} synthesizes these comparisons across predictor structure, robustness object, and decision-level guarantees.

Conformal prediction, also known as conformal inference, is a simple yet principled statistical calibration method that provides finite-sample, distribution-free uncertainty quantification for essentially any point predictor under an exchangeability assumption \citep{shafer2008tutorial}. It has recently attracted increasing attention with the rise of deep learning and has found applications in high-stakes settings such as medical imaging diagnosis \citep{angelopoulos2024conformal}. Its predictor-agnostic nature makes it especially attractive for robustifying contextual decision-making, where the upstream predictor may be statistically powerful but analytically opaque. Recent works in the broader data science community have started exploring conformal prediction in various settings, including predict-then-optimize \citep{patel2024conformal}, risk-sensitive linear programs \citep{sun2023predict}, end-to-end optimization \citep{chenreddy2024end}, out-of-distribution optimization \citep{cai2025out}, inverse optimization \citep{lin2024conformal}, decision optimality assessment \citep{zhou2025conformalized}, and risk-averse agents \citep{kiyani2025decision}. Our work contributes to this emerging literature by showing that conformal calibration can be used not only to construct uncertainty sets, but also to define a common robustness scale. This scale yields coverage-calibrated uncertainty sets for \ConfRO{} and normalized target-violation fragility measures for \ConfRS{}, thereby linking reliability-based and target-oriented robust decision-making with a single score-calibrated interface.

\section{Problem Setup and the Score-Calibrated Robustness Interface}
\label{sec:cus_framework}

Section~\ref{sec:cus_framework} formalizes the post-prediction robustness problem that motivates our framework. A trained predictor converts the observed context into a point forecast, but the downstream optimizer must still decide how forecast errors should be measured, calibrated, and incorporated into the decision model. We first introduce robust prediction-driven decision-making as a problem of converting a fixed predictor into a decision-relevant robustness representation. We then review split conformal calibration and show how the conformal score provides the common unit that supports both reliability-based robust optimization and target-oriented robust satisficing in Section~\ref{sec:cro_crs_framework}.

\subsection{Robust Prediction-Driven Decision-Making Problem}
\label{sec:problem_setting}

We study a prediction-driven decision-making problem indexed by an observed context $\bm z\in\mathcal Z$ and an uncertain parameter $\bm d\in\mathfrak D\subseteq\mathbb R^J$. The context $\bm z$ is observed before the decision is made, whereas $\bm d$ is realized after the decision. Given a realization of $\bm d$, the downstream model evaluates a decision $\bm x\in\mathcal X$ through an objective function $a_0(\bm x,\bm d)$ and constraints $a_i(\bm x,\bm d)\le b_i$, $i\in[I]$. A predictor $\hat f:\mathcal Z\to\mathfrak D$, which has been trained on historical training data before decision making with new context ${\bm z}$, maps each context $\bm z$ to a point forecast $\hat{\bm d}=\hat f(\bm z)$.

The classical PTO method directly plugs the forecast $\hat{\bm d}=\hat f(\bm z)$ into the downstream optimization model to obtain the decision $\bm x^{\mathrm{PTO}}(\bm z)$, i.e.,
\[
    \bm x^{\mathrm{PTO}}(\bm z)\in
    \arg\min_{\bm x\in\mathcal X}
    \left\{
        a_0(\bm x,\hat{\bm d}):
        a_i(\bm x,\hat{\bm d})\le b_i,\ i\in[I]
    \right\}.
\]

This sequential and modular design is compatible with arbitrary predictors, including black-box models trained on rich contextual data. However, it leaves the downstream optimizer without a calibrated scale for forecast error. Consequently, the PTO decision can be unreliable or fragile: the optimizer is anchored at $\hat{\bm d}$ but has no principled information about which deviations from $\hat{\bm d}$ should be protected against, how large those deviations should be, or how they affect the objective and constraints. The issue is therefore not only whether the forecast is accurate, but also how forecast error should be translated into a robustness representation that is meaningful for the downstream decision.

This post-prediction challenge motivates two complementary robustness views. In the reliability-based view, the decision maker specifies a desired reliability level and seeks protection over a calibrated region around the forecast. In the target-oriented view, the decision maker specifies an acceptable target and evaluates how fragile a decision is to deviations from the forecast. Both views require the same missing object: a calibrated, decision-relevant scale for measuring deviations from $\hat{\bm d}$. We formalize this problem below.

\begin{definition}[Robust Prediction-Driven Decision-Making]
\label{def:robust_prediction_driven_decision_making}
Given a fixed predictor $\hat f$ and a downstream decision problem with uncertain parameter $\bm d$, robust prediction-driven decision-making is the task of transforming the point forecast $\hat{\bm d}(\bm z)=\hat f(\bm z)$ into a calibrated robustness representation and using this representation to choose decisions under either reliability-based or target-oriented robustness criteria.
\end{definition}

The following examples illustrate settings in which such a post-prediction robustness representation is needed and in which decision makers pursue reliability-based or target-oriented robustness.

\begin{example}[E-commerce Replenishment]
In e-commerce replenishment, $\bm d$ may represent future demand across products, stores, or fulfillment locations, and $\bm x$ may represent replenishment quantities. A demand forecasting model $\hat{f}$ can use product descriptions, images, promotion signals, review text, and temporal patterns to generate $\hat{\bm d}(\bm z)$ \citep{salinas2019deepar, ansari2025chronos2, qi2025timehf}. The downstream inventory model, however, still needs a calibrated scale for deciding how much protection around this forecast is warranted. A reliability-based manager may specify a service reliability level, such as 95\% calibrated protection against demand deviations, whereas a target-oriented manager may specify an inventory cost or stockout cost target and seek the least fragile replenishment decision relative to that target.
\end{example}

\begin{example}[Spatio-temporal Driver Dispatch]
In ride-hailing dispatch, $\bm d$ may represent future demand-supply imbalance across zones and time periods, and $\bm x$ may represent repositioning or fleet capacity-allocation decisions. A black-box predictor $\hat{f}$ can exploit weather, traffic, transit disruptions, and platform signals \citep{uber2025forecasting,XiKumar2024LyftForecasting}. Robust dispatch nevertheless requires a calibrated scale for deviations from the forecast. The decision maker may request a dispatch plan with a prescribed reliability level, or instead evaluate fragility relative to a waiting-time or fulfillment target.
\end{example}

We next review split conformal prediction because it supplies the statistical calibration ingredient needed for this post-prediction robustness problem. Section~\ref{sec:score_calibrated_interface} then shows how the conformal score becomes a common robustness unit for both decision-making views.

\subsection{Preliminaries: Conformal Score, Calibration, and Coverage Guarantees}
\label{sec:cp_preliminary}

Split conformal prediction uses held-out calibration data to convert a fixed point predictor into an uncertainty region with finite-sample marginal coverage \citep{shafer2008tutorial,angelopoulos2023conformal}.

\begin{definition}[Conformal Score Function]
A conformal score function, also referred to as a nonconformity score function, is a function
\begin{equation}
    s_{\hat f}:\mathcal Z\times\mathfrak D\to\mathbb R_{\ge 0},
\end{equation}
that assigns to each context-realization pair $(\bm z,\bm d)$ a nonnegative measure of discrepancy between the prediction $\hat f(\bm z)$ and the realized outcome $\bm d$. The predictor $\hat f$ and all fitted quantities entering $s_{\hat f}$ are fixed before the calibration stage.
\end{definition}

Larger values of $s_{\hat f}(\bm z,\bm d)$ indicate greater disagreement between the prediction and the realization. Let $\{(\bm Z_i,\bm D_i)\}_{i=1}^n$ be the calibration dataset, and denote the calibration scores by
\begin{equation}
    S_i=s_{\hat f}(\bm Z_i,\bm D_i),\qquad i=1,\ldots,n.
\end{equation}
The calibration scores $S_1, S_2, \dots, S_n$ are sorted in non-decreasing order as $S_{(1)}\le S_{(2)} \le \cdots\le S_{(n)}$ with the convention that $S_{(n+1)}:=+\infty$. The conformal quantile index is $k_\alpha:=\lceil(n+1)\alpha\rceil$, and the corresponding calibrated score threshold is $\eta_\alpha:=S_{(k_\alpha)}$. For a desired coverage level $\alpha\in(0,1)$ and test context $\bm Z_{\mathrm{test}}$, the conformal prediction set is defined as
$
    \mathcal U_\alpha(\bm Z_{\mathrm{test}})
    =
    \{\bm D\in\mathfrak D:s_{\hat f}(\bm Z_{\mathrm{test}},\bm D)\le \eta_\alpha\}.
$
The construction of $\mathcal U_\alpha(\bm Z_{\mathrm{test}})$ separates three roles: \emph{centering} through the forecast $\hat f(\bm z)$, \emph{shaping} through the score $s_{\hat f}$, and \emph{sizing} through the conformal threshold $\eta_\alpha$.

\begin{assumption}
\label{asp:exchangeability}
The calibration data $\{(\bm{Z}_i, \bm{D}_i)\}_{i=1}^n$ and the test point $(\bm{Z}_{\mathrm{test}}, \bm{D}_{\mathrm{test}})$ (also denoted by $(\bm{Z}_{n+1}, \bm{D}_{n+1})$) are exchangeable. That is, for any permutation $\pi$ of $\{1, \dots, n+1\}$, the joint distribution of the data points is invariant under $\pi$, i.e.,
\[
    \bigl((\bm{Z}_1, \bm{D}_1), \dots, (\bm{Z}_{n+1}, \bm{D}_{n+1})\bigr) \stackrel{d}{=}\bigl((\bm{Z}_{\pi(1)}, \bm{D}_{\pi(1)}), \dots, (\bm{Z}_{\pi(n+1)}, \bm{D}_{\pi(n+1)})\bigr).
\]
\end{assumption}

The exchangeability condition in Assumption~\ref{asp:exchangeability} is standard in conformal calibration and milder than i.i.d. because it permits symmetric dependence among observations. This distributional symmetry is the key condition under which the rank of the test score among the calibration scores is controlled nonparametrically.
Under Assumption~\ref{asp:exchangeability}, the foundational result in conformal prediction yields a distribution-free finite-sample marginal coverage guarantee (formal statement in Lemma~\ref{lem:conformal_coverage}),
\begin{equation}
  \mathbb P\bigl(\bm D_{\mathrm{test}}\in \mathcal U_\alpha(\bm Z_{\mathrm{test}})\bigr)\ge \alpha, \quad \text{where } \mathcal U_\alpha(\bm Z_{\mathrm{test}})=\{\bm d\in\mathfrak D:s_{\hat{f}}(\bm Z_{\mathrm{test}},\bm d)\le \eta_\alpha \}.\label{eq:marginal_coverage}
\end{equation}

For decision-making, the important consequence is that the set coverage guarantee in \eqref{eq:marginal_coverage} transfers directly to any decision certificate that is enforced over the conformal uncertainty set.

\begin{proposition}[Calibration-Decision Transfer]
    \label{prop:calibration_transfer}
    Suppose Assumption~\ref{asp:exchangeability} holds. Let $\tilde{\bm x}$ be any decision selected before observing
    $\bm D_{\mathrm{test}}$. Let
    $
        h_m(\tilde{\bm x},\bm d),\ m=1,\ldots,M,
    $
    be performance functions representing either constraints or objective functions. If, almost surely, $h_m(\tilde{\bm x},\bm d)\le 0$ holds for all $m=1,\ldots,M$ whenever $\bm d$ is in the conformal uncertainty set $\mathcal U_\alpha(\bm Z_{\mathrm{test}})$,
    then
    \begin{equation}
        \mathbb P\left\{
            h_m(\tilde{\bm x},\bm D_{\mathrm{test}})\le 0,
            \ m=1,\ldots,M
        \right\}
        \ge \alpha .
        \label{eq:calibration-decision}
    \end{equation}
\end{proposition}

Proposition~\ref{prop:calibration_transfer} is stated in a general form so that $h_m({\bm x},\bm d)$ can represent either objective or constraint performance. Section~\ref{sec:cro_crs_framework} applies this transfer principle to the two decision-making strategies.

The sample complexity of split conformal calibration depends on the size of the calibration data. Beyond this finite-sample transfer guarantee, the calibration sample size also governs how stable the achieved coverage is across calibration realizations. Let $\mathcal D_n:=\{(\bm Z_i,\bm D_i)\}_{i=1}^n$ denote the calibration set and define the calibration-conditional coverage
\begin{equation}
\label{eq:conditional_coverage_def}
    p(\mathcal D_n):=
    \mathbb P\!\left(\bm D_{\mathrm{test}}\in \mathcal U_\alpha(\bm Z_{\mathrm{test}})\mid \mathcal D_n\right).
\end{equation}

\begin{assumption}
    \label{asp:iid_stability}
    The calibration pairs $(\bm{Z}_i,\bm{D}_i)_{i=1}^n$ and the test pair $(\bm{Z}_{\mathrm{test}},\bm{D}_{\mathrm{test}})$ are i.i.d.\ draws
    from a common distribution.
\end{assumption}

The marginal coverage guarantee in \eqref{eq:marginal_coverage} implies $\mathbb E[p(\mathcal D_n)]\ge\alpha$ under exchangeability. Under the common i.i.d. calibration regime formalized in Assumption~\ref{asp:iid_stability}, this marginal statement can be sharpened to non-asymptotic control of the lower tail of $p(\mathcal D_n)$ \citep{vovk2012conditional,duchi2025few}.
In particular, if the fixed score $S=s_{\hat f}(\bm Z,\bm D)$ has a continuous cumulative distribution function (CDF), then, as the calibration sample size grows, the lower tail of $p(\mathcal D_n)$ concentrates at the canonical nonparametric rate: with high probability, the achieved coverage is at least $\alpha-\mathcal O(n^{-1/2})$, and the probability of any fixed undercoverage gap decreases exponentially with $n$.

Beyond marginal coverage and coverage conditional on the calibration sample, one may seek the stronger pointwise guarantee
\[
    \mathbb P\!\left(\bm D_{\mathrm{test}}\in\mathcal U_\alpha(\bm z)\mid \bm Z_{\mathrm{test}}=\bm z\right)\ge \alpha
    \qquad\text{for every }\bm z.
\]
Fundamental impossibility results for conformal prediction rule out this guarantee for nontrivial distribution-free procedures: any useful conditional method must either impose additional structure on the data-generating process or weaken the guarantee to an asymptotic statement \citep{foygel2021limits}. Under such structure, localization can provide asymptotic context-conditional coverage at a fixed interior context.

Appendix~\ref{sec:conditional_cro} develops one such localized extension for smooth, low-dimensional contexts. With $n$ calibration observations and a rate-balancing bandwidth, kernel-weighted calibration yields both conditional quantile estimation error and realized context-conditional coverage error of order $\mathcal O_p\!\left((\log n/n)^{s/(2s+d_z)}\right)$, where $d_z$ denotes the context dimension. We retain marginal calibration as the main framework because it provides finite-sample, distribution-free validity and accommodates high-dimensional or unstructured contexts.

\subsection{The Score-Calibrated Robustness Interface}
\label{sec:score_calibrated_interface}

The key idea is to treat the conformal score $s_{\hat f}$, rather than any single calibrated sublevel set $\mathcal{U}_\alpha(\bm z)=\{\bm d:s_{\hat{f}}(\bm z,\bm d)\le\eta_\alpha\}$, as the robustness primitive. A sublevel set of the score supplies the uncertainty region for reliability-based robust optimization, while the same score supplies the deviation unit against which target violations are normalized in robust satisficing. This score-centered view leads to our new \ConfRS{} model and its conformal fragility measure, and it places \ConfRO{} and \ConfRS{} on a common robustness scale. Figure~\ref{fig:cro_pipeline} illustrates the resulting interface.

The interface follows a three-step Predict-Calibrate-Solve pipeline. First, given a context $\bm z$, the predictor outputs the point forecast $\hat{\bm d}(\bm z)=\hat f(\bm z)$. Second, the calibrator module evaluates held-out scores $S_i=s_{\hat f}(\bm z_i,\bm d_i)$, forms the empirical score distribution $\hat F_n(t)=n^{-1}\sum_{i=1}^n\indc\{S_i\le t\}$, and obtains calibrated thresholds such as $\eta_\alpha$ for desired reliability levels. Third, the solver uses the same score in one of two ways: \ConfRO{} takes a reliability level $\alpha$ and optimizes over the calibrated sublevel set $\mathcal U_\alpha(\bm z)$, whereas \ConfRS{} takes an acceptable target $\tau$ and defines the conformal fragility of a decision as the smallest coefficient $k$ such that target violations are bounded by $k\cdot s_{\hat f}(\bm z,\bm d)$ for all $\bm d\in\mathfrak D$.

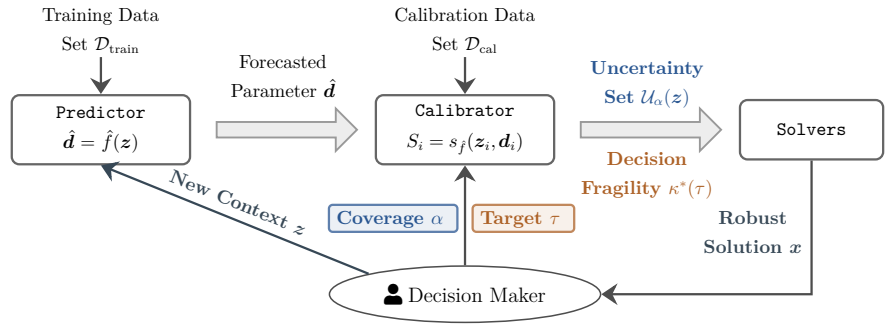
\begin{figure}[htb]
\vspace{-1.2em}
  \caption{Predict-Calibrate-Solve as a score-calibrated robustness interface.}
  \label{fig:cro_pipeline}
  \centering
  \vspace{0.8em}
  \definecolor{contextSlate}{RGB}{55, 72, 86}
\definecolor{coverageBlue}{RGB}{41, 86, 145}
\definecolor{targetCopper}{RGB}{172, 96, 37}

\def\croCoverageTagMinHeight{0.50cm}
\def\croTargetTagMinHeight{0.50cm}

\def\croCoverageTagTextHeight{2.0ex}
\def\croCoverageTagTextDepth{0.8ex}
\def\croTargetTagTextHeight{2.0ex}
\def\croTargetTagTextDepth{0.8ex}

\def\croCoverageTagInnerXSep{4pt}
\def\croCoverageTagInnerYSep{0pt}
\def\croTargetTagInnerXSep{4pt}
\def\croTargetTagInnerYSep{0pt}

\def\croCoverageTagPos{0.47}
\def\croCoverageTagLeft{0.12cm}
\def\croCoverageTagYShift{0pt}
\def\croTargetTagPos{0.47}
\def\croTargetTagRight{0.12cm}
\def\croTargetTagYShift{0pt}

\newcommand{\croInputTagStrut}{\vphantom{\textbf{Coverage $\alpha\tau$}}}

\resizebox{0.7\textwidth}{!}{\begin{tikzpicture}[
    node distance=1.8cm and 3.4cm,
    block/.style={
        rectangle, rounded corners, very thick,
        draw=black!60, fill=white,
        text centered, align=center, font=\small,
        minimum height=1.10cm, minimum width=2.4cm, text width=2.4cm
    },
    predictor_block/.style={
        block,
        minimum height=1.28cm, minimum width=3.05cm, text width=3.05cm
    },
    calibrator_block/.style={
        block,
        minimum height=1.28cm, minimum width=3.05cm, text width=3.05cm
    },
    thin_arrow/.style={
        -{Stealth[length=3mm, width=4mm]},
        line width=1.2pt, draw=black!70
    },
    fat_arrow/.style={
        single arrow, minimum height=2.60cm, shape border rotate=-90,
        line width=1.0pt, draw=black!40, fill=black!7
    },
    label_text/.style={
        text centered, align=center, font=\small
    },
    coverage_tag/.style={
        label_text,
        rounded corners=2pt,
        minimum height=\croCoverageTagMinHeight,
        text height=\croCoverageTagTextHeight,
        text depth=\croCoverageTagTextDepth,
        inner xsep=\croCoverageTagInnerXSep,
        inner ysep=\croCoverageTagInnerYSep,
        draw=coverageBlue!75,
        fill=coverageBlue!9
    },
    target_tag/.style={
        label_text,
        rounded corners=2pt,
        minimum height=\croTargetTagMinHeight,
        text height=\croTargetTagTextHeight,
        text depth=\croTargetTagTextDepth,
        inner xsep=\croTargetTagInnerXSep,
        inner ysep=\croTargetTagInnerYSep,
        draw=targetCopper!75,
        fill=targetCopper!9
    }
]

\node[predictor_block] (predictor)
    {\texttt{Predictor}\\[-1pt] $\hat{\bm{d}}=\hat f(\bm{z})$};

\node[calibrator_block, right=of predictor] (calibrator)
    {\texttt{Calibrator}\\[-1pt] {$S_i=s_{\hat f}(\bm{z}_i,\bm{d}_i)$}};

\node[block, right=of calibrator] (solvers)
    {\texttt{Solvers} \\ {}};

\node[
  draw=black!70,
  thick,
  ellipse,
  below=1.85cm of calibrator,
  minimum width=3.2cm,
  text width=3.2cm,
  align=center,
  inner sep=5pt
] (dm) {\faIcon{user}\;Decision Maker};

\node[fat_arrow] (arrow1) at ($(predictor.east)!0.5!(calibrator.west)$) {};
\node[label_text, above=0.32cm of arrow1]
    {Forecasted \\[-2pt] Parameter $\hat{\bm{d}}$};

\node[fat_arrow] (arrow2) at ($(calibrator.east)!0.5!(solvers.west)$) {};
\node[label_text, above=0.12cm of arrow2] {
    \textbf{\color{coverageBlue} Uncertainty} \\[-2pt]
    \textbf{\color{coverageBlue} Set $\mathcal{U}_{\alpha}(\bm{z})$}
};
\node[label_text, below=0.12cm of arrow2] {
    \textbf{\color{targetCopper} Decision} \\[-2pt]
    \textbf{\color{targetCopper} Fragility $\kappa^*(\tau)$}
};

\node[label_text, above=0.72cm of predictor.north, inner sep=1pt] (train_label)
    {Training Data \\[-2pt] Set $\mathcal{D}_{\text{train}}$};
\draw[thin_arrow] ([yshift=-1pt]train_label.south) -- (predictor.north);

\node[label_text, above=0.72cm of calibrator.north, inner sep=1pt] (calib_label)
    {Calibration Data \\[-2pt] Set $\mathcal{D}_{\text{cal}}$};
\draw[thin_arrow] ([yshift=-1pt]calib_label.south) -- (calibrator.north);

\draw[thin_arrow, draw=contextSlate] (dm.north west) -- (predictor.south)
    node[midway, label_text, above=0.05cm, sloped] {
        \textbf{\color{contextSlate} New Context $\bm{z}$}
    };

\draw[thin_arrow] (dm.north) -- (calibrator.south)
    node[pos=\croCoverageTagPos, coverage_tag, left=\croCoverageTagLeft, yshift=\croCoverageTagYShift] {
        \croInputTagStrut\textbf{\color{coverageBlue} Coverage $\alpha$}
    }
    node[pos=\croTargetTagPos, target_tag, right=\croTargetTagRight, yshift=\croTargetTagYShift] {
        \croInputTagStrut\textbf{\color{targetCopper} Target $\tau$}
    };

\coordinate (returncorner) at (solvers.south |- dm.east);
\coordinate (returnmid) at ($(solvers.south)!0.53!(returncorner)$);

\draw[thin_arrow] (solvers.south) -- (returncorner) -- (dm.east);
\node[label_text, left=0.04cm of returnmid]
    {\textbf{\color{contextSlate} Robust} \\[-2pt]
     \textbf{\color{contextSlate} Solution $\bm{x}$}};

\end{tikzpicture}}
  \vspace{-1.2em}
\end{figure}

This pipeline separates the statistical and optimization roles of prediction uncertainty. The predictor can be any fitted model; the conformal calibration step converts its empirical errors into a finite-sample valid score scale; and the downstream solver uses that scale through managerial inputs that are interpretable either as a reliability level $\alpha$ or as an acceptable target $\tau$. Section~\ref{sec:cro_crs_framework} formalizes the two resulting decision-making strategies, and Section~\ref{sec:ConfRO_ConfRS_equivalence} studies their decision-level correspondence and the conditions under which it strengthens to equivalence.

\section{Conformal Robustness: Two Complementary Strategies}
\label{sec:cro_crs_framework}
Section~\ref{sec:cus_framework} has introduced the conformal score as a post-prediction calibration layer: after a black-box predictor produces the context-specific forecast $\hat{\bm d}=\hat f(\bm z)$, the conformal scores $S_i = s_{\hat{f}}(\bm{z}_i, \bm{d}_i)$ on the calibration data are collected to capture the prediction uncertainty. In this section, we discuss how the calibration layer is integrated into decision-making through reliability-based and target-oriented approaches, respectively.

\subsection{Reliability-Based View: \ConfRO{}}
\label{sec:cro_framework}

Recent work has established how conformal uncertainty sets can be incorporated into optimization (e.g., \citealp{patel2024conformal, sun2023predict}). Building on this foundation, we analyze the decision consequences of the conformal score itself. We first organize representative score geometries by their induced uncertainty sets and implications for downstream optimization tractability (Section~\ref{subsubsec:cro_score_geometry}). We then derive a decision-efficiency bound for convex programs allowing uncertainty in the objective and/or constraints (Section~\ref{subsubsec:cro_value_good_prediction}). The bound separates the calibrated forecast-error scale from objective and dual-weighted constraint sensitivities. Together, these results clarify how prediction quality and score geometry determine the price of robustness in \ConfRO{}.

Given a context $\bm z$, let $\hat{\bm d}=\hat f(\bm z)$ be the corresponding point forecast. If the uncertain parameter were known to be $\bm d$, the downstream decision problem would be
\begin{equation}
\begin{aligned}
        V(\bm d)  =  \inf_{\bm{x}} &\;   a_0(\bm{x}, \bm{d}),\\
    \text{s.t.}  &\;  a_i(\bm{x}, \bm{d}) \leq b_i, \quad i \in [I],\\
    &\;  \bm{x}  \in \mathcal{X}.
\end{aligned}
\label{eq:nominal_uncertain_problem}
\end{equation}
A predict-then-optimize method replaces $\bm d$ by $\hat{\bm d}$ in \eqref{eq:nominal_uncertain_problem} and solves the resulting deterministic problem. In contrast, \ConfRO{} incorporates the calibrated conformal set
$
    \mathcal U_\alpha(\bm z)=\{\bm d\in\mathfrak D:s_{\hat f}(\bm z,\bm d)\le \eta_\alpha\},
$
where $\eta_\alpha$ is the split-conformal score quantile in Section~\ref{sec:cp_preliminary}. The robust value $V_{\ConfRO}(\alpha)$ is given by
\begin{equation}
\begin{aligned}
    V_{\ConfRO}(\alpha)=\inf_{\bm{x}} \sup_{\bm{d}\in\mathcal{U}_{\alpha}(\bm{z})} &\;  a_0(\bm{x},\bm{d}),\\
    \text{s.t.}  &\;  a_i(\bm{x}, \bm{d}) \leq b_i, \quad i \in [I],\ \forall \bm{d} \in \mathcal{U}_\alpha(\bm{z}),\\
    &\;  \bm{x} \in \mathcal{X}.
\end{aligned}
\label{eq:ConfRO_primal}
\end{equation}
Because $\mathbb P\{\bm{D}_{\mathrm{test}}\in\mathcal U_\alpha(\bm{Z}_{\mathrm{test}})\}\ge \alpha$ under the exchangeability of Assumption~\ref{asp:exchangeability} and the constraints in \eqref{eq:ConfRO_primal} are enforced over $\mathcal U_\alpha(\bm z)$, the original constraints $a_i(\bm{x},\bm D_{\mathrm{test}})\le b_i$ are satisfied with  probability at least $\alpha$, as shown in Proposition~\ref{prop:calibration_transfer}.

\subsubsection{Score Geometry and Reliability Specification.}
\label{subsubsec:cro_score_geometry}

Implementing \ConfRO{} requires choosing a score function $s(\cdot,\cdot)$ and a reliability level $\alpha$. The score determines the geometry of forecast deviations and may encode residual scale, covariance, sparsity, and component- or direction-specific weights. It need not be a metric: conformal validity permits asymmetry and does not require the triangle inequality, provided the score is fixed before calibration. The level $\alpha$ controls only the calibrated size through $\eta_\alpha$. Thus, score design should prioritize downstream tractability, whereas $\alpha$ remains an interpretable reliability input independent of that design. Table~\ref{tab:cus_designs} summarizes Box, Ellipsoid, and Budget designs for coordinate-wise, correlated, and aggregate deviations, respectively.

\begin{table}[htbp]
    \centering
    \small
    \caption{Correspondence between uncertainty set geometries and conformal score designs.}
    \label{tab:cus_designs}
    \renewcommand{\arraystretch}{1.3}
    \begin{tabular}{p{0.1\textwidth} p{0.3\textwidth} p{0.36\textwidth}}
        \toprule
        {Geometry} & {Conformal score $s(\bm z, \bm d)$} & {Resulting uncertainty set $\mathcal{U}_\alpha(\bm{z})$} \\
        \midrule
        {Box} & $\displaystyle \max_{j\in[J]}\frac{|d_j - \hat{f}_j(\bm{z})|}{\hat{h}_j(\bm{z})}$ & $\displaystyle \left\{ \bm{d} : |d_j - \hat{f}_j(\bm{z})| \le \eta_\alpha \hat{h}_j(\bm{z}), \ j\in[J] \right\}$ \\
        {Ellipsoid} & $\displaystyle \frac{\|\bm{L}(\bm{d} - \hat{f}(\bm{z}))\|_2}{\hat{g}(\bm{z})}, \, \hat{\bm\Sigma}^{-1}=\bm L^\top\bm L$ & $\displaystyle \left\{ \bm{d} : (\bm{d} - \hat{f}(\bm{z}))^\top \hat{\bm{\Sigma}}^{-1} (\bm{d} - \hat{f}(\bm{z})) \le \eta_\alpha^2 \hat{g}^2(\bm{z}) \right\}$ \\
        {Budget} & $\displaystyle \left\| \frac{\bm{d} - \hat{f}(\bm{z})}{\bm{\beta} \odot \hat{\bm{h}}(\bm{z})} \right\|_1$ & $\displaystyle \left\{ \bm{d} : \sum_{j=1}^J \frac{|d_j - \hat{f}_j(\bm{z})|}{\beta_j \hat{h}_j(\bm{z})} \le \eta_\alpha\right\}$ \\
        \bottomrule
    \end{tabular}
    \par\smallskip
    \parbox{\textwidth}{\scriptsize \textit{Note.} The table presents general data-driven forms in which $\hat{\bm h}(\bm z)$ and $\hat g(\bm z)$ are trained a priori, in addition to $\hat{f}$, to estimate residual scales; $\bm\beta\in\mathbb R_{+}^J$ specifies component-wise weights; static score designs are obtained by replacing these estimated quantities with constants. Division and multiplication by $\odot$ are component-wise.}
    \vspace{-1.5em}
\end{table}

The adaptive designs in Table~\ref{tab:cus_designs} use local scale estimators to accommodate heteroscedastic forecast errors. Furthermore, direction-specific weights extend this construction to asymmetric errors. For example, if $\hat{\bm r}(\bm z)\in\mathbb R_{++}^J$ estimates component-wise residual magnitudes and $\bm\omega^+,\bm\omega^-\in\mathbb R_+^J$ are learned or decision-specified directional weights, an adaptive asymmetric weighted $L_1$ score is
\begin{equation}
\label{eq:adaptive_asymmetric_score}
    s_{(\hat f,\hat r)}(\bm z,\bm d)
    = \sum_{j=1}^J \left[
        \omega_j^+ \frac{(d_j-\hat f_j(\bm z))_+}{\hat r_j(\bm z)}
        + \omega_j^- \frac{(\hat f_j(\bm z)-d_j)_+}{\hat r_j(\bm z)}
    \right].
\end{equation}
When $\bm\omega^+\ne\bm\omega^-$, \eqref{eq:adaptive_asymmetric_score} is asymmetric and therefore need not be metric.

\ConfRO{} enjoys an interpretability benefit by using the coverage level $\alpha$ as the primary tuning parameter. This avoids asking the decision-maker to specify a radius, budget, or ambiguity size whose operational meaning is less direct. For example, $\alpha=0.95$ simply implies that the \ConfRO{} decision is obtained when the constraints are satisfied with marginal probability at least 95\%, as shown in Section~\ref{sec:cp_preliminary}. Consequently, larger values of $\alpha$ produce weakly larger sets $\mathcal{U}_\alpha$ and more conservative decisions, while smaller values produce tighter sets and emphasize efficiency.

\subsubsection{Bounding the Price of Robustness.}\label{subsubsec:cro_value_good_prediction}
Conformal calibration provides statistical validity for \ConfRO{}, but coverage alone does not quantify the efficiency cost of robustness. Calibrated uncertainty sets with the same nominal coverage may induce different decisions because their geometries interact differently with the downstream objective and constraints. A decision-efficiency bound is therefore needed to reveal how the calibrated forecast-error scale and downstream optimization sensitivity jointly determine the price of robustness.

For the decision-efficiency bound below, fix a test context $\bm z_{\mathrm{test}}$, and write $\mathcal U:=\mathcal U_\alpha(\bm z_{\mathrm{test}})$ and $\hat{\bm d}_{\mathrm{test}}:=\hat f(\bm z_{\mathrm{test}})$. Suppose that the score has the metric form $s_{\hat f}(\bm z_{\mathrm{test}},\bm d)=\mathfrak s(\bm d,\hat{\bm d}_{\mathrm{test}})$ for $\bm d\in\mathfrak D$, where $\mathfrak s:\mathfrak D\times\mathfrak D\to\mathbb R_+$ is a metric. Thus, $\mathcal U=\{\bm d\in\mathfrak D:\mathfrak s(\bm d,\hat{\bm d}_{\mathrm{test}})\le\eta_\alpha\}$.

For each $i\in\{0\}\cup[I]$, define its score sensitivity on $\mathcal U$ by
\begin{equation}
\label{eq:score_sensitivity}
    \mathcal L_{\mathfrak s,i}(\bm x)
    :=\sup_{\substack{\bm u,\bm v\in\mathcal U\\ \bm u\ne\bm v}}
    \frac{|a_i(\bm x,\bm u)-a_i(\bm x,\bm v)|}
         {\mathfrak s(\bm u,\bm v)}.
\end{equation}
We set $\mathcal L_{\mathfrak s,i}(\bm x):=0$ when $\mathcal U$ is a singleton. For a test realization $\bm d_{\mathrm{test}}$, let $V_{\mathrm{test}}:=V(\bm d_{\mathrm{test}})$, and let $V_{\ConfRO}$ denote the value of \eqref{eq:ConfRO_primal} at $\bm z_{\mathrm{test}}$. Define the decision-efficiency loss as
\begin{equation}
    \label{eq:de_gap_convex}
    \delta_{\mathrm{Conv}}(\bm z_{\mathrm{test}},\bm d_{\mathrm{test}})
    :=V_{\ConfRO}-V_{\mathrm{test}}.
\end{equation}

\begin{proposition}
    \label{prop:convex_bound}
    Fix $\bm d_{\mathrm{test}}\in\mathcal U$. Suppose that $V_{\mathrm{test}}\in\mathbb R$ is attained at a realized optimizer $\bm x^\star_{\mathrm{test}}$, that $\mathcal X$ is convex, and that $a_i(\cdot,\bm d)$ is convex on $\mathcal X$ for every $\bm d\in\mathcal U$ and $i\in\{0\}\cup[I]$. Assume that $\eta_\alpha<\infty$, $V_{\ConfRO}\in\mathbb R$, and $\mathcal L_{\mathfrak s,i}(\bm x^\star_{\mathrm{test}})<\infty$ for all $i\in\{0\}\cup[I]$. Suppose further that the Lagrangian dual of the outer robust program has zero duality gap and attains its optimum at a multiplier vector $\bm\lambda^R=(\lambda_1^R,\ldots,\lambda_I^R)\in\mathbb R_+^I$ associated with the robust constraints $\sup_{\bm d\in\mathcal U}a_i(\bm x,\bm d)\le b_i$, $i\in[I]$. Then,
    \[
        0\le
        \delta_{\mathrm{Conv}}(\bm z_{\mathrm{test}},\bm d_{\mathrm{test}})
        \le
        2\eta_\alpha
        \left[
            \mathcal L_{\mathfrak s,0}(\bm x^\star_{\mathrm{test}})
            +\sum_{i=1}^I\lambda_i^R
            \mathcal L_{\mathfrak s,i}(\bm x^\star_{\mathrm{test}})
        \right].
    \]
\end{proposition}

Proposition~\ref{prop:convex_bound} provides a pathwise bound for every realization contained in the calibrated uncertainty set. The bound separates the calibrated forecast-error scale $\eta_\alpha$ from downstream optimization sensitivity. The term $\mathcal L_{\mathfrak s,0}(\bm x^\star_{\mathrm{test}})$ measures the objective's exposure to score-measured perturbations, while each robust dual multiplier $\lambda_i^R$ converts the corresponding constraint sensitivity $\mathcal L_{\mathfrak s,i}(\bm x^\star_{\mathrm{test}})$ into objective units. Under a fixed score construction, prediction improvements that concentrate calibration scores near zero shrink $\eta_\alpha$, while decision-aligned score geometry can reduce the objective and constraint sensitivities that determine the price of robustness.

This pathwise bound combines directly with conformal coverage to yield a finite-sample probabilistic guarantee. If the metric-score representation and regularity conditions in Proposition~\ref{prop:convex_bound} hold almost surely for the random test instance, then the bound applies on the coverage event $\{\bm D_{\mathrm{test}}\in\mathcal U_\alpha(\bm Z_{\mathrm{test}})\}$. By the marginal coverage guarantee in \eqref{eq:marginal_coverage}, this event has probability at least $\alpha$. Consequently, the decision-efficiency bound holds with probability at least $\alpha$ under the joint distribution of $(\bm Z_{\mathrm{test}},\bm D_{\mathrm{test}})$.

\citet{patel2024conformal} bound decision-efficiency loss in objective-only problems using a global Lipschitz constant and the diameter of the calibrated uncertainty set. Proposition~\ref{prop:convex_bound} instead analyzes convex programs allowing objective and constraint uncertainty and decomposes downstream sensitivity into objective and dual-weighted constraint components.

\begin{remark}[Sharper Bounds for the Linear Case]
\label{rem:linear_case}
Consider the linear special case with objective $a_0(\bm x,\bm d)=\bm c^\top\bm x$, constraint function $a_1(\bm x,\bm d)=\bm d^\top\bm x$, and right-hand side $b_1=b$.
Write $\delta_{\mathrm{Lin}}(\bm z,\bm d):=\delta_{\mathrm{Conv}}(\bm z,\bm d)$, let $\lambda^R$ denote the multiplier of the single robust constraint, and define
\[
    \sigma_{\mathcal U}(\bm x):=\sup_{\bm d\in\mathcal U}\bm d^\top\bm x,
    \qquad
    \mathcal L_{\mathfrak s}(\bm x):=\mathcal L_{\mathfrak s,1}(\bm x).
\]
Because $a_0$ is independent of $\bm d$, $\mathcal L_{\mathfrak s,0}(\bm x)=0$ for every $\bm x$. Proposition~\ref{prop:convex_bound} therefore specializes to the final inequality in the sharper chain
\begin{equation}
\begin{aligned}
    0\le\delta_{\mathrm{Lin}}(\bm z_{\mathrm{test}},\bm d_{\mathrm{test}})
    &\le\lambda^R\bigl[\sigma_{\mathcal U}(\bm x^\star_{\mathrm{test}})-b\bigr]\\
    &\le\lambda^R\bigl[\sigma_{\mathcal U}(\bm x^\star_{\mathrm{test}})
        -\bm d_{\mathrm{test}}^\top\bm x^\star_{\mathrm{test}}\bigr]\\
    &\le2\eta_\alpha\lambda^R
        \mathcal L_{\mathfrak s}(\bm x^\star_{\mathrm{test}}).
\end{aligned}
\label{eq:de_bound_linear}
\end{equation}
The first intermediate bound in \eqref{eq:de_bound_linear} is tighter than the second by $\lambda^R[b-\bm d_{\mathrm{test}}^\top\bm x^\star_{\mathrm{test}}]$, the multiplier-weighted realized constraint slack. Both intermediate bounds retain the exact support function and can therefore be strictly tighter than the final score-sensitivity bound. The final inequality coincides exactly with the linear specialization of Proposition~\ref{prop:convex_bound}, whereas the preceding inequalities exploit the linear structure to retain instance-specific information.
\end{remark}

\subsection{Target-Oriented View: \ConfRS{}}
\label{sec:crs_framework}

\ConfRO{} begins with a reliability level $\alpha$ and optimizes worst-case performance over the corresponding calibrated set $\mathcal{U}_\alpha(\bm z)$. In many applications, however, robustness requirements are naturally expressed through operational targets: a decision maker may care that cost stays below a budget, utility exceeds a benchmark, or service performance remains acceptable, and then ask for the decision that is least fragile to deviations from the forecast around that target. This target-first view motivates \emph{Conformal Robust Satisficing} (\ConfRS{}), which takes an acceptable target $\tau$ as the primitive robustness input.

Existing robust satisficing frameworks primarily define fragility under distributional ambiguity, including residual-based formulations that incorporate side-information predictions \citep{long2023robust,sim2024analytics}. We study a complementary post-prediction setting in which uncertainty concerns the realized parameter around an arbitrary fixed point forecast. Accordingly, \ConfRS{} uses the same conformal score as the deviation scale for target violations. Section~\ref{subsubsec:crs_vs_drs} gives the exact relationship between this formulation and residual-based distributionally robust satisficing (DRS).

For a fixed context $\bm z$, define the point forecast and its score-induced deviation as
\[
    \hat{\bm d}:=\hat f(\bm z),
    \qquad
    \mathfrak s_{\bm z}(\bm d,\hat{\bm d}):=s_{\hat f}(\bm z,\bm d).
\]
Suppressing the context subscript, throughout this subsection we assume $\mathfrak s(\bm d,\hat{\bm d})\ge0$, $\mathfrak s(\hat{\bm d},\hat{\bm d})=0$, and $\mathfrak s(\bm d,\hat{\bm d})>0$ for $\bm d\ne\hat{\bm d}$; metric properties are imposed only when explicitly stated for certain choices of $s$ and $\mathfrak{s}$. We introduce the \ConfRS{} formulation as
\begin{equation*}
\begin{aligned}
    \kappa^*(\tau)=\inf_{\bm x,k} &\; k\\
    \text{s.t.} & \;  a(\bm{x},\bm{d})-\tau \le k\, \mathfrak{s}(\bm{d},\hat{\bm{d}}),\quad \forall \bm{d}\in \mathfrak{D},\\
     &\; \bm{x}\in \mathcal{X},\quad k\ge 0.
\end{aligned}\tag{\ConfRS{}} \label{eq:ConfRS}
\end{equation*}
We use the extended-value convention $\inf\varnothing=+\infty$.
Here $\mathfrak D\subseteq\mathbb R^J$ denotes the support of the uncertain parameter, rather than an uncertainty set to be calibrated. Equivalently, the key constraint in \eqref{eq:ConfRS} can be written as
\[
\sup_{\bm d\in\mathfrak D}\{a(\bm x,\bm d)-k\mathfrak s(\bm d,\hat{\bm d})\}\le \tau.
\]
\begin{remark}
\label{rem:multi_constraints}
For multiple objectives or constraints, one may introduce target levels $\tau_i$ and fragility variables $k_i$ through constraints of the form $a_i(\bm x,\bm d)-\tau_i\le k_i\mathfrak s(\bm d,\hat{\bm d})$, and minimize a weighted aggregate $\bm w^\top\bm k$.
\end{remark}

The optimal value $\kappa^*(\tau)$ in \eqref{eq:ConfRS} characterizes \emph{decision fragility}: small $\kappa^*(\tau)$ means that the target violation $a(\bm{x},\bm{d})-\tau$ scales slowly with respect to the score deviation $\mathfrak{s}(\bm{d}, \hat{\bm{d}})$. Unlike distributionally robust satisficing, which controls worst-case expected performance through Wasserstein distance from a reference distribution, \ConfRS{} controls realized target violations through a post-prediction score deviation $\mathfrak{s}(\bm{d}, \hat{\bm{d}})$.

Denote the boundary target values
\[
    \underline\tau:=\inf_{\bm x\in\mathcal X}a(\bm x,\hat{\bm d}),
    \qquad
    \bar\tau:=\inf_{\bm x\in\mathcal X}\sup_{\bm d\in\mathfrak D}a(\bm x,\bm d).
\]
If $\tau<\underline\tau$, the constraint at $\bm d=\hat{\bm d}$ is infeasible. If the robust benchmark is attained and $\tau\ge\bar\tau$, then there exists a decision whose worst-case cost over $\mathfrak D$ is at most $\tau$, so the optimal fragility is $\kappa^*(\tau)=0$. Therefore, in order for $0<\kappa^*(\tau)<\infty$ to hold, $\tau$ should lie in $[\underline \tau, \bar\tau)$.

More specifically, for a fixed decision $\bm x$, finite fragility is equivalent to
\begin{equation}
\label{eq:finite_fragility_condition}
    a(\bm x,\hat{\bm d})\le \tau
    \quad\text{and}\quad
    \sup_{\bm d\in\mathfrak D\setminus\{\hat{\bm d}\}}
    \frac{(a(\bm x,\bm d)-\tau)_+}{\mathfrak s(\bm d,\hat{\bm d})}<\infty.
\end{equation}
Bounded support, continuity, and local Lipschitz continuity of $a(\bm x,\cdot)$ with respect to $\mathfrak s$ are sufficient for \eqref{eq:finite_fragility_condition}. For unbounded $\mathfrak D$, a simple sufficient growth condition is
\begin{equation}
    \limsup_{\|\bm{d}\|\to\infty}\frac{a(\bm{x},\bm{d})-\tau}{\mathfrak{s}(\bm{d},\hat{\bm{d}})} < \infty,
\end{equation}
so the score-induced deviation grows at least as fast as the objective along the tails. This condition is satisfied, for example, by affine objectives on $\mathbb R^J_+$ under norm-based scores, including the fractional knapsack model studied in Section~\ref{sec:experiments}.

\subsubsection{Conformal Score Induces a Fragility Measure.}
We next make explicit the fragility notion embedded in \ConfRS{}. The optimal value $\kappa^*(\tau)$ measures the minimum fragility at target $\tau$. More generally, any conformal score induces a \emph{fragility measure} satisfying the required axiomatic conditions. Such induction is important because it directly links the deviation from the nominal forecast $\hat{\bm{d}}$, captured by $\mathfrak{s}(\bm{d}, \hat{\bm{d}})$,  with a robustness criterion of the downstream decision, captured by the fragility measure that is formally defined next.

For notational simplicity, we denote the target violation $v_{\bm x}(\bm d):=a(\bm x,\bm d)-\tau$ for a fixed decision $\bm x$ and target $\tau$. When there is no ambiguity, let $v$ denote a generic violation function.

\begin{definition}[Conformal Fragility]
    \label{def:fragility_measure}
    Given a violation function $v$ and score deviation function $\mathfrak{s}(\bm{d}, \hat{\bm d})$,  its corresponding conformal fragility is defined as
    \begin{equation}
        \rho(v):=\inf\left\{k\ge0: v(\bm d)\le k\mathfrak s(\bm d, \hat{\bm d}),\ \forall \bm d\in\mathfrak D\right\},
    \end{equation}
    with the convention that $\inf\varnothing=+\infty$.
\end{definition}

Under Definition~\ref{def:fragility_measure}, \eqref{eq:ConfRS} is equivalently viewed as $\inf_{\bm x\in\mathcal X}\ \rho\big(a(\bm x,\bm d)-\tau\big)$, so any optimizer selects a decision whose target-violation function has the smallest conformal fragility. We prove that $\rho$ satisfies the standard axiomatic requirements of a fragility measure in the satisficing literature \citep{brown2009satisficing} and thus establish it as a measure of fragility.

\begin{theorem}[Conformal Fragility Measure]
    \label{thm:fragility_measure}
    Suppose the score-induced deviation satisfies $\mathfrak s(\bm d,\hat{\bm d})\ge0$, $\mathfrak s(\hat{\bm d},\hat{\bm d})=0$, and $\mathfrak s(\bm d,\hat{\bm d})>0$ for $\bm d\ne\hat{\bm d}$. The conformal fragility measure $\rho$ is lower semicontinuous and satisfies the following properties.
    \begin{enumerate}[label={\textup{(\roman*)}}]
        \item \textbf{Monotonicity:} If $v_{1}(\bm{d})\ge v_{2}(\bm{d})$ for all $\bm{d}\in \mathfrak{D}$, then $\rho(v_1)\ge \rho(v_2)$. \label{thm:fm_1}
        \item \textbf{Positive homogeneity:} For any $\lambda\ge0$, $\rho(\lambda v)=\lambda \rho(v)$. \label{thm:fm_2}
        \item \textbf{Subadditivity:} $\rho(v_1+v_2)\le \rho(v_1)+\rho(v_2)$. \label{thm:fm_3}
        \item \textbf{Pro-robustness:} If $v(\bm{d})\le 0$ for all $\bm{d}\in \mathfrak{D}$, then $\rho(v)=0$. \label{thm:fm_4}
        \item \textbf{Anti-fragility:} If $v(\hat{\bm{d}})>0$, then $\rho(v)=+\infty$. \label{thm:fm_5}
    \end{enumerate}
\end{theorem}

Properties~\ref{thm:fm_1}--\ref{thm:fm_3} of Theorem~\ref{thm:fragility_measure} show that conformal fragility is a monotone, sublinear, and therefore convex, functional. Properties~\ref{thm:fm_4}--\ref{thm:fm_5} of Theorem~\ref{thm:fragility_measure} simply encode the boundary behavior required for satisficing: fragility vanishes, i.e., $\rho(v)=0$,  when the target is uniformly met and becomes infinite, i.e., $\rho(v)=+\infty$, when the target is missed at the baseline forecast. More specifically, when $v(\hat{\bm d})\le0$, this functional admits the representation
\begin{equation}
    \rho(v)=\sup_{\bm d\in\mathfrak D\setminus\{\hat{\bm d}\}}
    \frac{v(\bm d)_+}{\mathfrak s(\bm d,\hat{\bm d})},
    \qquad v(\bm d)_+:=\max\{v(\bm d),0\}.
\end{equation}
If $v(\hat{\bm d})>0$, then $\rho(v)=+\infty$ because the target is violated at the forecast $\hat{\bm{d}}$ with zero score deviation.  Theorem~\ref{thm:fragility_measure} thus formally establishes a parameter-uncertainty counterpart to robust satisficing fragility, centered around score-measured deviations from prediction.

\subsubsection{Data-Driven Target Guarantees.}
The fragility parameter also yields a data-driven certificate for target-violation probabilities. Fix $\tau$, and for each $\bm z\in\mathcal Z$, let $(\bm x^*(\bm z),k^*(\bm z))$ be an optimizer of \eqref{eq:ConfRS}. Given the calibration sample $\mathcal D_{\mathrm{cal}}=\{(\bm Z_i,\bm D_i)\}_{i=1}^n$, define
\[
R_\tau(\bm z,\bm d):=k^*(\bm z)s_{\hat f}(\bm z,\bm d),
\qquad
\widehat F_n^R(t):=\frac{1}{n}\sum_{i=1}^n
\mathbf 1\{R_\tau(\bm Z_i,\bm D_i)\le t\}.
\]

\begin{proposition}
    \label{prop:crs_performance}
    Under Assumption~\ref{asp:iid_stability}, for every violation margin $\Delta>0$ and confidence level $\epsilon\in(0,1)$, with probability at least $1-\epsilon$ over $\mathcal D_{\mathrm{cal}}$,
\begin{equation}
    \mathbb{P}\{a(\bm{x}^*(\bm Z_\text{test}), \bm{D}_{\mathrm{test}}) > \tau + \Delta\}
    \le 1 - \widehat{F}^{R}_n\left(\Delta\right) + \sqrt{\frac{\ln(1/\epsilon)}{2n}}.
\end{equation}
\end{proposition}

Proposition~\ref{prop:crs_performance} clarifies how conformal calibration enters \ConfRS{}. The bound is constructed from held-out fragility-scaled scores obtained by evaluating the \ConfRS{} policy on the calibration contexts, rather than relying on distribution-specific concentration parameters or tail assumptions for the data-generating distribution \citep{long2023robust}. The result also highlights the value of predictive accuracy and score design: smaller prediction-error scores $s(\bm Z_i,\bm D_i)$ and smaller context-dependent fragility $k^*(\bm Z)$ concentrate the distribution of $R_\tau$ near zero, causing its empirical CDF to rise more quickly and tightening the bound on large target violations.
Thus, $k^*(\bm Z)$ converts a unit of score-measured prediction error into a decision-relevant upper envelope on target violation.
Section~\ref{sec:ConfRO_ConfRS_equivalence} formalizes this interpretation by building the connection between \ConfRO{} and \ConfRS{}.

By Proposition~\ref{prop:calibration_transfer}, it holds that
\[\mathbb P\{a(\bm x^*(\bm Z_{\mathrm{test}}),\bm D_{\mathrm{test}})\le \tau+k^*(\bm Z_{\mathrm{test}})\eta_\alpha\}\ge\alpha.
\]
Thus, $k^*\eta_\alpha$ is a calibrated upper bound on the target exceedance at reliability level $\alpha$, complementing the probabilistic guarantee from Proposition~\ref{prop:crs_performance}. Smaller fragility $k^*$ tightens this bound, while better predictors and better-aligned scores reduce $\eta_\alpha$. Together, these two data-driven guarantees show that conformal calibration converts the typically deterministic fragility measure into a finite-sample probabilistic certificate for target satisfaction.

\begin{remark}[Relation to Chance-Constrained Programs]
Chance-constrained programs restrict $\mathbb P\{a(\bm x^*,\bm D)>\tau+\Delta\}\le\gamma$ for a risk level $\gamma\in(0,1)$ and violation margin $\Delta>0$, but are often nonconvex or require conservative safe approximations \citep{jiang2022also}. Proposition~\ref{prop:crs_performance} gives a nonparametric data-driven approximation with finite-sample guarantees. Specifically, it suffices to impose
$1-\hat F^{R}_n(\Delta)+\sqrt{\ln(1/\epsilon)/(2n)}\le\gamma.$
Let $\zeta:=1-\gamma+\sqrt{\ln(1/\epsilon)/(2n)}$. When $\zeta<1$, this condition is equivalent to $\Delta\ge \hat q^R_\zeta$, where $\hat q^R_\zeta:=\inf\{r\ge0:\hat F^R_n(r)\ge\zeta\}$ is the empirical $\zeta$-quantile of the calibration scores. Since \ConfRS{} minimizes fragility, it also tightens the resulting family of chance-type certificates among feasible decisions under the same target and score specification.
\end{remark}

\subsubsection{\ConfRS{} vs. Distributionally Robust Satisficing.}
\label{subsubsec:crs_vs_drs}

\ConfRS{} shares the target-oriented philosophy of Distributionally Robust Satisficing (DRS), but differs in the role of uncertainty. Let $\mathcal P(\mathfrak D)$ denote the class of probability distributions supported on $\mathfrak D$, and let $\Delta$ be a distributional discrepancy on this class. Standard DRS evaluates satisficing at the level of expected performance relative to a reference distribution $\widehat P\in\mathcal P(\mathfrak D)$:
\begin{equation*}
\begin{aligned}
    \kappa_{\mathrm{DRS}}(\tau)
    = \min_{\bm x,k}\quad & k \\
    \text{s.t.}\quad
    & \mathbb{E}_{\tilde{\bm d}\sim P}
    \bigl[a(\bm x,\tilde{\bm d})\bigr]-\tau
    \leq k\,\Delta(P,\widehat P),
    \quad \forall P\in\mathcal P(\mathfrak D), \\
    & \bm x\in\mathcal X,\quad k\geq 0.
\end{aligned}
\tag{DRS}\label{eq:DRS}
\end{equation*}

For Wasserstein-based DRS with empirical reference distribution
$\widehat P=S^{-1}\sum_{s\in[S]}\delta_{\tilde{\bm d}_s}$, an equivalent
representation of \eqref{eq:DRS} given by \citet{long2023robust} is
\begin{equation*}
\begin{aligned}
    \kappa_{\mathrm{DRS}}(\tau)
    = \min_{\bm x,k}\quad & k \\
    \text{s.t.}\quad
    & \frac{1}{S}\sum_{s\in[S]}
    \sup_{\bm d\in\mathfrak D}
    \left\{
        a(\bm x,\bm d)
        -k\,\mathfrak c(\bm d,\tilde{\bm d}_s)
    \right\}
    \leq \tau, \\
    & \bm x\in\mathcal X,\quad k\geq 0,
\end{aligned}
\tag{DRS-E}\label{eq:DRS-E}
\end{equation*}
where $\{\tilde{\bm d}_s\}_{s\in[S]}$ are the support points of the empirical
reference distribution and $\mathfrak c$ is a transport cost, typically a norm
or a power of a metric.

When the score-induced deviation $\mathfrak s$ is an admissible transport cost, a
\emph{fixed-context} \ConfRS{} problem admits a one-support-point DRS
representation. Specifically, fix a context $\bm z$ and define the
prediction-generated reference distribution
$\widehat P_{\bm z}=\delta_{\hat f(\bm z)}$. Under the choices
$S=1$, $\tilde{\bm d}_1=\hat f(\bm z)$, $\mathfrak c(\bm d,\tilde{\bm d}_1)=s_{\hat f}(\bm z,\bm d)$,
\eqref{eq:DRS-E} reduces to
$
    \sup_{\bm d\in\mathfrak D}
    \left\{
        a(\bm x,\bm d)
        -k\,s_{\hat f}(\bm z,\bm d)
    \right\}
    \leq \tau,
$
which coincides with the optimization constraint in \eqref{eq:ConfRS}.

Despite this algebraic coincidence, the complete frameworks differ in how
prediction uncertainty is represented and calibrated. First, the conformal
score $s_{\hat f}(\bm z,\bm d)$ need not be a transport cost and may encode
context-dependent notions of forecast error beyond Wasserstein geometry.
Second, \ConfRS{} treats an arbitrary fitted point predictor as fixed and
centers each deployed problem at the current forecast
$\hat{\bm d}=\hat f(\bm z)$, producing one anchor per decision instance and a
context-indexed family of anchors and scores across instances. In contrast, the residual-based framework of \citet{sim2024analytics} uses a structured parametric prediction model to construct a multi-support predicted empirical distribution and explicitly fortifies the decision against estimation error in the prediction coefficients. Third, held-out calibration scores in \ConfRS{} provide finite-sample target-violation certificates and connect the resulting fragility to a reliability scale (Propositions~\ref{prop:crs_performance} and~\ref{prop:calibration_transfer}).

\begin{remark}
Later, Section~\ref{subsec:reliability_target_translation} includes a concrete Example~\ref{example:distinction} in which the set of decisions attainable under \ConfRS{} is strictly larger than the corresponding set under the DRS model considered there. This comparison highlights the distinction while also suggesting opportunities to enrich distributionally robust models through context-dependent prediction and conformal calibration. Indeed, extending conformal calibration of arbitrary black-box predictors to the distributional setting represents a promising direction for future research. A direct distribution-level extension may require richer predictive outputs, such as conditional distributions or context-dependent empirical reference measures, together with greater data and modeling requirements that may be impractical in some applications. We discuss these opportunities and challenges in Appendix~\ref{app:cdro_note}.
\end{remark}

\section{Score-Calibrated Robustness Frontiers: Equivalence and Sensitivity}
\label{sec:ConfRO_ConfRS_equivalence}

The two decision-making strategies, \ConfRO{} and \ConfRS{} presented in Section~\ref{sec:cro_crs_framework},  use the same post-prediction score-calibrated robustness scale in different manners, and Table~\ref{tab:cro_crs_comparison} contrasts their key components.
\begin{table}[htb]
 \vspace{-1em}
    \centering
    \caption{Key conceptual components of \ConfRO{} and \ConfRS{}.}
    \label{tab:cro_crs_comparison}
    \small
    \renewcommand{\arraystretch}{1.2}
    \setlength{\tabcolsep}{3pt}
    \begin{tabularx}{\textwidth}{@{}l
        >{\hsize=0.95\hsize\linewidth=\hsize\raggedright\arraybackslash}X
        >{\hsize=1.15\hsize\linewidth=\hsize\raggedright\arraybackslash}X
        >{\hsize=1.00\hsize\linewidth=\hsize\raggedright\arraybackslash}X
        >{\hsize=0.90\hsize\linewidth=\hsize\raggedright\arraybackslash}X@{}}
        \toprule
        {Model} & {Decision-maker input} & {Objective} & {Uncertainty anchor} & {Decision philosophy} \\
        \midrule
        {\ConfRO{}} & Coverage level $\alpha$ & Minimize cost $V_{\ConfRO}$ & Calibrated set $\mathcal U_\alpha(\bm z)$ & Reliability-based \\
        {\ConfRS{}} & Acceptable target $\tau$ & Minimize fragility $\kappa^*(\tau)$ & Score deviation $\mathfrak s(\bm d,\hat{\bm d})$ & Target-oriented \\
        \bottomrule
    \end{tabularx}
    \vspace{-1em}
\end{table}

Despite these differences, a natural question is whether the two formulations represent merely distinct robustness modeling choices, or whether a deeper connection exists. Addressing this question clarifies how reliability levels and acceptable targets can be mapped to one another, and how the cost of robustness evolves along the frontier.

\subsection{Dual Representation and Target Interpretation}
\label{subsec:dual_target_interpretation}

To build the connection, we consider a family of  \ConfRO{} problems indexed by $\theta$. Here $\theta$ is the score radius linked to the reliability level $\alpha$, defined as $\theta(\alpha):=F^{-1}(\alpha)$, where $F^{-1}$ is the inverse CDF of the score distribution. Let \ConfRO{}$(\theta)$ denote the problem with optimal value
\begin{equation}
\label{eq:ConfRO_theta}
    V(\theta):=
    \inf_{\bm{x}\in\mathcal{X}}\sup_{\bm{d}\in\mathfrak{D}(\theta)} a(\bm{x},\bm{d}),
    \qquad
    \text{ where } \mathfrak{D}(\theta):=\{\bm d\in\mathfrak D:\mathfrak s(\bm d,\hat{\bm d})\le\theta\}.
\end{equation}

Under strong duality for the inner worst-case maximization problem, \eqref{eq:ConfRO_theta} has the equivalent value representation
\begin{equation*}
\begin{aligned}
    \inf_{\bm x,k,\tau}&\;  \tau + k\theta\\
    \text{s.t.} &\;  \sup_{\bm{d}\in\mathfrak{D}}\{a(\bm{x},\bm{d})-k\mathfrak{s}(\bm{d},\hat{\bm{d}})\}\le\tau, \\
   &\;  \bm{x}\in\mathcal{X},\quad k\ge 0.
\end{aligned}\tag{\ConfRO{}-D} \label{eq:ConfRO-D}
\end{equation*}
A key observation is that \eqref{eq:ConfRO-D} and  \eqref{eq:ConfRS} share the same constraint structure.
Consequently, we can view \ConfRO{} as choosing a target $\tau$ and a fragility level $k$ and then minimizing the objective $\tau+k\theta$. This observation leads to the correspondence results in this section. We formally state the conditions needed in Assumption~\ref{ass:cro_crs_equiv}.

\begin{assumption}
    \label{ass:cro_crs_equiv}
    For the \ConfRO{} and \ConfRS{} problems considered, the following conditions hold.
    \begin{enumerate}[label={\textup{(\roman*)}}]
    \item The decision set $\mathcal X$ is convex, and $a(\bm x,\bm d)$ is proper, closed, and convex in $\bm x$ for every $\bm d\in\mathfrak D$. \label{ass:equiv_convex}

    \item   \label{ass:equiv_dual} For every $\bm x\in\mathcal X$ and every score radius $\theta$ of interest,
    \[
        \sup_{\bm d\in\mathfrak D(\theta)}a(\bm x,\bm d)
        =
        \inf_{k\ge0}
        \left\{
        k\theta+
        \sup_{\bm d\in\mathfrak D}
        \bigl(a(\bm x,\bm d)-k\mathfrak s(\bm d,\hat{\bm d})\bigr)
        \right\}.
    \]
    \item The uncertainty enters through one objective (or one constraint through reformulation). Multiple independent uncertain constraints require an additional aggregation rule and need not yield the same decision-level correspondence.  \label{ass:equiv_single}
    \end{enumerate}
\end{assumption}

A set of verifiable sufficient conditions for Assumption~\ref{ass:cro_crs_equiv}\ref{ass:equiv_dual} is given by Proposition~\ref{prop:sufficient_duality} below.
Thus, the correspondence results apply to a broad but explicitly delimited class of convex models as specified in condition \ref{ass:equiv_convex} with uncertain objective functions, including the fractional knapsack problem in Section~\ref{subsec:knapsack}; see a summary of other score geometries in Table~\ref{tab:score_duality_conditions}. Similar to the distributional case \citep[Remark~\ref{rem:multi_constraints}]{wang2025equivalence}, this correspondence need not hold when there are multiple uncertain constraints without an additional aggregation rule; we therefore impose condition \ref{ass:equiv_single} of Assumption~\ref{ass:cro_crs_equiv}. These conditions, however, do not impose computational restrictions; both \ConfRO{} and \ConfRS{} can be implemented beyond Assumption~\ref{ass:cro_crs_equiv}.

\begin{proposition}
\label{prop:sufficient_duality}
Given any $\bm x\in\mathcal X$ and $\theta>0$, suppose $\mathfrak D$ is closed and convex, $\mathfrak s(\cdot,\hat{\bm d})$ is proper, closed, and convex on $\mathfrak D$, $a(\bm x,\cdot)$ is concave and upper semicontinuous on $\mathfrak D$, the worst-case value over $\mathfrak D(\theta)$ is finite, and there exists $\bar{\bm d}\in\operatorname{ri}(\mathfrak D)$ such that $\mathfrak s(\bar{\bm d},\hat{\bm d})<\theta$ and $a(\bm x,\bar{\bm d})$ is finite. Then the strong duality condition in Assumption~\ref{ass:cro_crs_equiv}\ref{ass:equiv_dual} holds for $\bm x$ and $\theta$, and the infimum over $k$ is attained.
\end{proposition}

\begin{proposition}[Target-Oriented Interpretation of \ConfRO{}]
    \label{prop:cro_to_crs}
    Suppose Assumption~\ref{ass:cro_crs_equiv}  \ref{ass:equiv_dual}-\ref{ass:equiv_single} holds. Fix $\theta>0$. For any optimal solution $(\bm x^*,k^*,\tau^*)$ of \eqref{eq:ConfRO-D}, the decision-fragility pair $(\bm x^*,k^*)$ is optimal for \ConfRS{}$(\tau^*)$.
\end{proposition}

Proposition~\ref{prop:cro_to_crs} shows that every attained solution of the dual representation has a target-oriented interpretation: the optimizer selects a target $\tau^*$ and then chooses a decision with minimum fragility at that target. In particular, every decision represented by an attained optimizer of \eqref{eq:ConfRO-D} belongs to a \ConfRS{} optimal decision set at its endogenously selected target. This one-sided mapping follows directly from the shared constraint structure exposed by \eqref{eq:ConfRO-D}.

\subsection{Reliability--Target Translation}
\label{subsec:reliability_target_translation}

The reverse direction is more challenging: can a target specified exogenously in \ConfRS{} be represented by an appropriate robust set size (i.e., the score radius $\theta$)? We provide an affirmative answer in Theorem~\ref{thm:crs_to_cro}.

Recall that $\kappa^*(\tau)$ denotes the extended optimal value of \ConfRS{}$(\tau)$. Under Assumption~\ref{ass:cro_crs_equiv}, $\kappa^*$ is convex and nonincreasing on its effective domain. For a fixed score radius $\theta$, \eqref{eq:ConfRO-D} reduces to the following scalar value problem:
\begin{equation*}
    V(\theta)=\inf_{\tau\ge\underline\tau} g(\tau,\theta),
    \qquad g(\tau,\theta):=\tau+\theta\kappa^*(\tau).
    \tag{\ConfRO{}-E}\label{eq:cro_e}
\end{equation*}

\begin{theorem}[Reliability--Target Translation]
    \label{thm:crs_to_cro}
    Suppose Assumption~\ref{ass:cro_crs_equiv} holds and, for the coverage-level interpretation, Assumption~\ref{asp:exchangeability} holds. Fix an interior target $\tau_{\mathrm{rs}}\in(\underline\tau,\bar\tau)$ and choose any subgradient $\beta(\tau_{\mathrm{rs}})\in\partial\kappa^*(\tau_{\mathrm{rs}})$ with $\beta(\tau_{\mathrm{rs}})<0$. Set $
        \theta^*=-{1}/{\beta(\tau_{\mathrm{rs}})}$.
    \begin{enumerate}[label={\textup{(\roman*)}}, itemsep=0.25em, topsep=0.25em]
        \item\label{thm:crs_to_cro_correspondence_item} Every optimal decision of \ConfRS{}$(\tau_{\mathrm{rs}})$ is optimal for \ConfRO{}$(\theta^*)$.
        \item\label{thm:crs_to_cro_equivalence_item} If, in addition, $\tau_{\mathrm{rs}}$ is the unique minimizer of $g(\cdot,\theta^*)$ in \eqref{eq:cro_e} and the score-dual infimum in Assumption~\ref{ass:cro_crs_equiv}\ref{ass:equiv_dual} is attained for every \ConfRO{}$(\theta^*)$ optimal decision, then \ConfRS{}$(\tau_{\mathrm{rs}})$ and \ConfRO{}$(\theta^*)$ have the same optimal decision set.
        \item\label{thm:crs_to_cro_coverage_item} Furthermore, if the score distribution has continuous CDF $F$, then the corresponding coverage level is $\alpha^*=F(\theta^*)=F(-1/\beta(\tau_{\mathrm{rs}}))$. With atoms, the same statement holds using the generalized quantile convention, possibly yielding an interval of coverage levels.
    \end{enumerate}
\end{theorem}

Theorem~\ref{thm:crs_to_cro} makes the reliability--target translation precise. The selected subgradient $\beta(\tau)$ is the marginal change in minimum fragility as the acceptable target is relaxed. Because the theorem selects $\beta(\tau)<0$, the matched radius $\theta=-1/\beta(\tau)$ is positive. The condition $0\in 1+\theta\partial\kappa^*(\tau)$ is precisely the first-order condition for $\tau$ to solve \eqref{eq:cro_e}. When $g(\cdot,\theta)$ has a flat set of minimizers, the correspondence is set-valued rather than one-to-one; this is the mathematical source of the step-like parameter mappings observed in Figure~\ref{fig:parameter_mapping} of the numerical study.

The actual score distribution is not typically provided in practice, but its empirical distribution is observed during calibration. Therefore, the target-to-coverage map can also be approximated. For a fixed deterministic selection $\theta(\tau)=-1/\beta(\tau)$, define $\alpha(\tau)=F(\theta(\tau))$ and $\hat\alpha_n(\tau)=\hat F_n(\theta(\tau))$, where $\hat F_n$ is the empirical CDF of the calibration scores. Corollary~\ref{cor:mapping_compute} captures the accuracy guarantee of the empirical estimation.
\begin{corollary}
    \label{cor:mapping_compute}
    Suppose Assumptions~\ref{asp:iid_stability} and~\ref{ass:cro_crs_equiv} hold. For any $\delta \in (0, 1)$,
    \[
    \mathbb{P}\left( \sup_{\tau \in [\underline{\tau},\bar{\tau}]} |\hat{\alpha}_n(\tau) - \alpha(\tau)| \le \sqrt{\frac{\log(2/\delta)}{2n}} \right) \ge 1 - \delta.
    \]
\end{corollary}

Figure~\ref{fig:proof_flow} illustrates the geometric intuition behind Theorem~\ref{thm:crs_to_cro}. Strong duality for the inner score-constrained problem yields the value representation \eqref{eq:ConfRO-D}. For a fixed target $\tau$, the infimum of feasible $k$ in \eqref{eq:ConfRO-D} is exactly $\kappa^*(\tau)$, the optimal value of \ConfRS{}$(\tau)$. Substitution gives the scalar trade-off \eqref{eq:cro_e}. The key analytical step is convexity of $\kappa^*(\tau)$ (Lemma~\ref{lemma:kappa_convex} in Appendix~\ref{app:proof_equivalence}). Convexity gives the subgradient optimality condition
\[
    0\in \partial_\tau g(\tau,\theta)=1+\theta\partial\kappa^*(\tau),
\]
which yields the radius $\theta=-1/\beta(\tau)$ for any negative subgradient $\beta(\tau)\in\partial\kappa^*(\tau)$. This condition states that, at the matched parameters, the marginal cost of relaxing the target is exactly balanced by the marginal reduction in fragility.

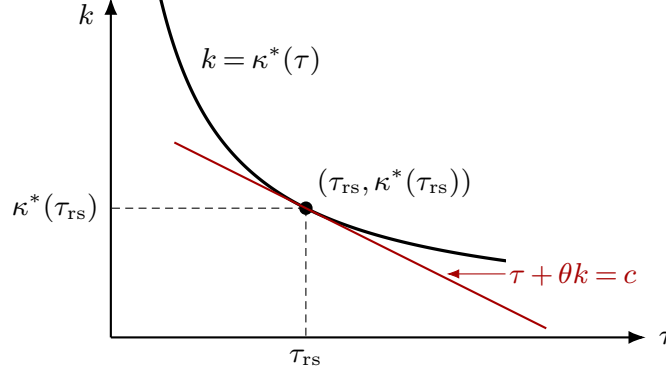
\begin{figure}[htb]
    \centering
    \caption{Geometric intuition of the \ConfRO{}--\ConfRS{} correspondence.}
    \label{fig:proof_flow}
    \begin{tikzpicture}[>=Latex, font=\footnotesize, scale=1.2, transform shape]
\pgfmathsetmacro{\taurs}{2.15}
\pgfmathsetmacro{\krs}{3/(\taurs+0.30)+0.20}
\pgfmathsetmacro{\betaval}{-3/((\taurs+0.30)^2)}

\draw[->, thick] (0,0) -- (5.9,0) node[right] {$\tau$};
\draw[->, thick] (0,0) -- (0,3.75);
\node[anchor=east] at (-0.05,3.57) {$k$};

\draw[very thick]
  plot[domain=0.55:5.25, samples=140, smooth] (\x,{3/(\x+0.30)+0.20});
\node[anchor=west, fill=white, inner sep=1pt] at (0.95,3.05)
  {$k=\kappa^*(\tau)$};

\coordinate (P) at (\taurs,\krs);
\fill (P) circle (2.1pt);
\node[anchor=south west, fill=white, inner sep=1pt] at ($(P)+(0.08,0.08)$)
  {$(\tau_{\mathrm{rs}},\kappa^*(\tau_{\mathrm{rs}}))$};

\draw[densely dashed] (P) -- (\taurs,0) node[below] {$\tau_{\mathrm{rs}}$};
\draw[densely dashed] (P) -- (0,\krs) node[left] {$\kappa^*(\tau_{\mathrm{rs}})$};

\draw[thick, draw=red!65!black]
  ({\taurs-1.45},{\krs+\betaval*(-1.45)}) --
  ({\taurs+2.65},{\krs+\betaval*(2.65)});
\node[anchor=west, fill=white, inner sep=1pt, text=red!65!black] (isocostlabel) at (4.35,0.70)
  {$\tau+\theta k=c$};
\draw[->, draw=red!65!black] (isocostlabel.west) -- ({\taurs+1.55},{\krs+\betaval*(1.45)});
\end{tikzpicture}
\vspace{-1em}
\end{figure}

Theorem~\ref{thm:crs_to_cro} is conceptually related to the equivalence in the Wasserstein setting \citep{wang2025equivalence}. However, the two correspondence results rely on different model primitives, and neither implies the other. First, \ConfRO{} and \ConfRS{} admit calibrated scores that need not be symmetric or metric; Wasserstein DRO and DRS are built from a metric transport cost. Second, \ConfRO{} and \ConfRS{} are anchored at a context-specific forecast and calibrated by forecast residuals; Wasserstein DRO and DRS are anchored at an empirical distribution and calibrated by distributional ambiguity.

The following Example~\ref{example:distinction} further demonstrates the distinction.

\begin{example}[Distinction from DRO-DRS Equivalence]
\label{example:distinction}
Consider one-dimensional $\bm x$, $\bm d$, written as $x$,$d$ for clarity. Let $a(x,d)=x^2/2-dx$, let $\mathfrak s(d,\hat d)=|d-\hat d|$, and take the point prediction $\hat d=1$. For $\theta\in[0,1]$, \ConfRO{} reduces to
\[
    \min_x\left\{x^2/2-x+\theta|x|\right\},
\]
whose optimizer is $x^*=1-\theta$. The corresponding \ConfRS{} target is $\tau=(\theta^2-1)/2$, which yields the same optimizer. By contrast, a Wasserstein DRO model with empirical distribution $\widehat P=(\delta_1+\delta_{-1})/2$ and transport cost $|d-d'|$ reduces to
\[
    \min_x\left\{x^2/2+\vartheta|x|\right\},
\]
whose optimizer is $x^*=0$. The Wasserstein DRO-DRS correspondence therefore recovers a distributionally conservative decision that is insensitive to the point prediction, whereas the conformal correspondence preserves the prediction-adaptive decision $1-\theta$.
\end{example}

\subsection{Sensitivity Interpretation of the Robust Frontier}
\label{subsec:marginal_cost_robustness}

The reliability--target correspondence also provides a local sensitivity interpretation of the robust frontier. Under Assumption~\ref{ass:cro_crs_equiv}, suppose $\kappa^*$ is finite and continuous on $[\underline\tau,\bar\tau]$, and consider an open interval $\mathcal I$ of score radii on which the scalar representation in \eqref{eq:cro_e} is valid:
\begin{equation}
    V(\theta)
    =
    \min_{\tau\in[\underline\tau,\bar\tau]}
    \{\tau+\theta\kappa^*(\tau)\},
    \qquad \theta\in\mathcal I.
\end{equation}
For each $\theta\in\mathcal I$, define the corresponding set of target minimizers as
\[
    T(\theta):=
    \operatorname*{arg\,min}_{\tau\in[\underline\tau,\bar\tau]}
    \{\tau+\theta\kappa^*(\tau)\}.
\]
Fix $\tau_{\mathrm{rs}}\in(\underline\tau,\bar\tau)$ and a subgradient $\beta_{\mathrm{rs}}\in\partial\kappa^*(\tau_{\mathrm{rs}})$ with $\beta_{\mathrm{rs}}<0$. Set $\theta_{\mathrm{rs}}:=-1/\beta_{\mathrm{rs}}$ and suppose $\theta_{\mathrm{rs}}\in\mathcal I$. The subgradient optimality condition then implies $\tau_{\mathrm{rs}}\in T(\theta_{\mathrm{rs}})$. Finally, let $F$ be the CDF of $S:=\mathfrak s(\bm D_{\mathrm{test}},\hat{\bm d})$ and set $\alpha_{\mathrm{rs}}:=F(\theta_{\mathrm{rs}})$.

\begin{theorem}[Marginal Costs of Robustness and Reliability]
\label{thm:target_value_marginal_cost}
The following sensitivity identities of $V(\theta)$ hold.
\begin{enumerate}[label={\textup{(\roman*)}}, itemsep=0.25em, topsep=0.25em]
    \item
    \label{thm:marginal_cost_robustness_item}
    The one-sided derivatives of $V$ at $\theta_{\mathrm{rs}}$ exist and satisfy
    \begin{equation}
    \label{eq:mc_robustness}
    \begin{aligned}
        \lim_{h\downarrow0} h^{-1}
        \{V(\theta_{\mathrm{rs}}+h)-V(\theta_{\mathrm{rs}})\}
        &=
        \min_{\tau\in T(\theta_{\mathrm{rs}})}\kappa^*(\tau),\\
        \lim_{h\downarrow0} h^{-1}
        \{V(\theta_{\mathrm{rs}})-V(\theta_{\mathrm{rs}}-h)\}
        &=
        \max_{\tau\in T(\theta_{\mathrm{rs}})}\kappa^*(\tau).
    \end{aligned}
    \end{equation}
    \vspace{-2.0em}
    \item
    \label{thm:marginal_cost_reliability_item}
    Suppose, in addition, that $F$ is differentiable and strictly increasing in a neighborhood of $\theta_{\mathrm{rs}}$, with $f_S(\theta_{\mathrm{rs}}):=F'(\theta_{\mathrm{rs}})>0$, and that $V$ is differentiable at $\theta_{\mathrm{rs}}$. Then
    \begin{equation}
    \label{eq:mc_reliability}
        \left.
        \partial_\alpha
        V\bigl(F^{-1}(\alpha)\bigr)
        \right|_{\alpha=\alpha_{\mathrm{rs}}}
        =
        \kappa^*(\tau_{\mathrm{rs}})
        f_S(\theta_{\mathrm{rs}})^{-1}.
    \end{equation}
    \vspace{-2.5em}
\end{enumerate}
\end{theorem}
Theorem~\ref{thm:target_value_marginal_cost} adds a local sensitivity interpretation to the reliability--target correspondence. In \eqref{eq:mc_robustness}, the optimized fragility $\kappa^*(\tau_{\mathrm{rs}})$ is the shadow price of expanding the conformal uncertainty region. At a smooth point, increasing the score radius from $\theta_{\mathrm{rs}}$ to $\theta_{\mathrm{rs}}+\Delta\theta$ raises the robust value $V(\theta)$ by approximately $\kappa^*(\tau_{\mathrm{rs}})\Delta\theta$. Thus, $\kappa^*(\tau_{\mathrm{rs}})$ measures the operational cost, in objective-value units, of requiring protection against one additional unit of conformal-score deviation from the forecast.

In \eqref{eq:mc_reliability} of Theorem~\ref{thm:target_value_marginal_cost}, this cost of increasing the reliability level $\alpha$ becomes $\kappa^*(\tau_{\mathrm{rs}})f_S(\theta_{\mathrm{rs}})^{-1}$, which separates decision fragility from statistical calibration: reliability is costly either because the decision is highly sensitive to score-level perturbations, as captured by a large $\kappa^*(\tau_{\mathrm{rs}})$, or because the score distribution is sparse near the current radius, as captured by a small $f_S(\theta_{\mathrm{rs}})$. Theorem~\ref{thm:target_value_marginal_cost} thus provides a diagnostic for whether additional reliability is limited by the decision model or by the calibration score distribution.

\subsection{Score Selection via Performance Evaluation}
\label{subsec:score_selection}

The preceding analysis treats the conformal score as fixed, but different score geometries can induce different decision frontiers at the same nominal reliability level. Because no score is universally optimal, even for pure prediction tasks \citep{angelopoulos2023conformal}, we develop a data-driven procedure that selects among candidates according to downstream robust performance, aligning score choice with the decision problem. Independent recalibration preserves finite-sample validity, as formalized in Proposition~\ref{prop:post_selection_validity}.

Let
$
    \mathcal S=\{s_\ell:\ell\in\mathcal A_{\mathrm{sc}}\}
$
be a finite collection of candidate score functions indexed by the finite set
$\mathcal A_{\mathrm{sc}}$. Each score
$s_\ell:\mathcal Z\times\mathfrak D\to\mathbb R_+$
may depend on the fitted predictor $\hat f$ and on score-specific nuisance estimates, such as residual scales, weights, or covariance matrices. All such fitted quantities are treated as fixed before the score-selection and final-calibration stages. For $\eta\ge0$, define
$
    \mathcal U_{\ell,\eta}(\bm z)
    :=
    \{\bm d\in\mathfrak D:s_\ell(\bm z,\bm d)\le \eta\}.
$

After training the predictor and fixing the candidate score functions, split the remaining data into three independent parts $\mathcal D_{\mathrm{pil}}$, $\mathcal D_{\mathrm{sel}}$, and $\mathcal D_{\mathrm{cal}}$, with sizes $N_p,N_s,N_c$, respectively. The pilot sample is used to construct provisional radii for comparing scores; the selection sample is used to choose a score; and the final calibration sample is used only after the score has been selected.

For each $\ell\in\mathcal A_{\mathrm{sc}}$, compute the pilot scores
$
    S^{\mathrm{pil}}_{\ell,i}
    =
    s_\ell(\bm z_i,\bm d_i),
    (\bm z_i,\bm d_i)\in\mathcal D_{\mathrm{pil}},
$
sort them as
$S^{\mathrm{pil}}_{\ell,(1)}\le\cdots\le S^{\mathrm{pil}}_{\ell,(N_p)}$,
set $S^{\mathrm{pil}}_{\ell,(N_p+1)}:=+\infty$, and define $k_p:=\lceil (N_p+1)\alpha\rceil$ and $\tilde\eta_{\alpha,\ell}:=S^{\mathrm{pil}}_{\ell,(k_p)}$.
The provisional set for score $\ell$ is
$
    \widetilde{\mathcal U}_{\alpha,\ell}(\bm z)
    =
    \mathcal U_{\ell,\tilde\eta_{\alpha,\ell}}(\bm z).
$ We compare candidate scores through the induced \ConfRO{} frontier. For a score $\ell$, radius $\eta$, and context $\bm z$, define
\begin{equation}
\label{eq:score_selection_frontier_value}
\begin{aligned}
    \phi_\ell(\bm z;\eta)
    :=
    \min_{\bm x}
        &\sup_{\bm d\in\mathcal U_{\ell,\eta}(\bm z)}
        a_0(\bm x,\bm d) \\
    \mathrm{s.t.}\quad
        & a_i(\bm x,\bm d)\le b_i,\quad
        i\in[I],\ \forall \bm d\in\mathcal U_{\ell,\eta}(\bm z),\\
        & \bm x\in\mathcal X,
\end{aligned}
\end{equation}
with the convention $\phi_\ell(\bm z;\eta)=+\infty$ if the robust problem is infeasible or has infinite worst-case objective value. Thus $\phi_\ell(\bm z;\eta)$ is the certified robust objective value obtained by using score $\ell$ at radius $\eta$ for context $\bm z$. We select the score by minimizing the empirical frontier value on the selection sample:
\begin{equation}
\label{eq:score_selection_rule}
    \widehat \Psi_\ell
    :=
    \frac{1}{N_s}
    \sum_{(\bm z_i,\bm d_i)\in\mathcal D_{\mathrm{sel}}}
    \phi_\ell(\bm z_i;\tilde\eta_{\alpha,\ell}),
    \qquad
    \widehat{\ell}\in\arg\min_{\ell\in\mathcal A_{\mathrm{sc}}}\widehat \Psi_\ell,
\end{equation}
with ties broken by a fixed deterministic rule.

After selecting $\widehat{\ell}$, discard the pilot radius and recalibrate the selected score on the independent final calibration sample. Specifically, compute $S^{\mathrm{cal}}_{\widehat{\ell},i}=s_{\widehat{\ell}}(\bm z_i,\bm d_i)$ for $(\bm z_i,\bm d_i)\in\mathcal D_{\mathrm{cal}}$,
sort these scores with
$S^{\mathrm{cal}}_{\widehat{\ell},(N_c+1)}:=+\infty$, and set $k_c=\lceil (N_c+1)\alpha\rceil$ and $\widehat\eta_\alpha=S^{\mathrm{cal}}_{\widehat{\ell},(k_c)}$.
The final post-selection conformal set is $
    \widehat{\mathcal U}_\alpha(\bm z)
    =
    \{\bm d\in\mathfrak D:
    s_{\widehat{\ell}}(\bm z,\bm d)\le \widehat\eta_\alpha\}.$ This independent recalibration yields the following post-selection guarantee.

\begin{proposition}
\label{prop:post_selection_validity}
Suppose that, after these pre-calibration quantities are fixed, the final calibration observations in $\mathcal D_{\mathrm{cal}}$ and the test observation
$(\bm Z_{\mathrm{test}},\bm D_{\mathrm{test}})$ are exchangeable. Then the post-selection conformal set $\widehat{\mathcal U}_\alpha(\bm Z_{\mathrm{test}})$ satisfies
$
    \mathbb P\left\{
    \bm D_{\mathrm{test}}
    \in
    \widehat{\mathcal U}_\alpha(\bm Z_{\mathrm{test}})
    \right\}
    \ge
    \alpha.
$
\end{proposition}

We now turn from the validity of the selection procedure to its connection with the \ConfRO{}--\ConfRS{} frontier. This interpretation invokes Assumption~\ref{ass:cro_crs_equiv}, whereas the procedure itself and Proposition~\ref{prop:post_selection_validity} do not. Specifically, when the robust frontier in \eqref{eq:score_selection_frontier_value} reduces to the setting of Assumption~\ref{ass:cro_crs_equiv}\ref{ass:equiv_single}, with no additional independent uncertain robust constraints and with any deterministic constraints absorbed into $\mathcal X$, write the context-fixed deviation scale as
$
    \mathfrak s_{\ell,\bm z}(\bm d)
    :=
    s_\ell(\bm z,\bm d).
$
Let $\kappa_\ell^*(\tau;\bm z)$ denote the optimal \ConfRS{} fragility value under score $\ell$ and context $\bm z$. Then the dual representation in \eqref{eq:ConfRO-D} gives
\begin{equation}
\label{eq:score_selection_crs_frontier}
    \phi_\ell(\bm z;\eta)
    =
    \inf_{\tau\in[\underline\tau_\ell(\bm z),\bar\tau_\ell(\bm z)]}
    \left\{
        \tau+\eta\,\kappa_\ell^*(\tau;\bm z)
    \right\}.
\end{equation}
Therefore, in this setting, score selection by \eqref{eq:score_selection_rule} selects the candidate with the best empirical \ConfRO{}--\ConfRS{} frontier at reliability level $\alpha$. This avoids comparing raw set volumes across geometries, which can be misleading because different scores use different units and level-set shapes.

\section{Applications}
\label{sec:experiments}

We evaluate the framework through synthetic benchmarks and a real-data case study. We begin by benchmarking on a classical fractional knapsack problem to validate empirical coverage, realized utility, and the \ConfRO{}--\ConfRS{} correspondence. We then study a multiperiod inventory problem using real data from a large online grocery platform, where deep learning-based demand forecasts must be translated into operationally implementable robust inventory decisions. Further details of the numerical results and an additional synthetic study on facility location are deferred to Appendix~\ref{app:numerical_details}.

\subsection{Simulation Study: Robust Fractional Knapsack Problem}
\label{subsec:knapsack}

We consider robustifying a data-driven \emph{fractional knapsack problem} in which item utilities $\bm c$ are uncertain but predictable from contextual features $\bm z$, following the setup of \cite{ho-nguyen2022risk}. We use the cost convention from Section~\ref{sec:cro_crs_framework} such that maximizing realized utility $\bm c^\top\bm x$ is equivalent to minimizing $a(\bm{x},\bm{c}):=-\bm{c}^\top \bm{x}$. The decision variable $\bm{x}$ lies in the feasible region $\mathcal{X}=\{\bm{x}:\bm{x}\in [0,1]^n,\ \bm{p}^\top \bm{x}\le B\}$, where $\bm{p}$ is the price of the items, and $B$ is the budget.

\subsubsection{Numerical Results on Fractional Knapsack.}
To evaluate forecast-centered calibration, we compare \ConfRO{} with methods that use progressively less contextual information: a context-agnostic ellipsoidal robust optimization baseline (Ellipsoid-RO), and the local $k$-means robust optimization (KMeans-RO) and $k$-nearest-neighbor robust optimization (KNN-RO) baselines \citep{ohmori2021predictive}. Joint estimation and robustness optimization (JERO; \citealp{zhu2022joint}) is not included because it optimizes the uncertainty-set radius rather than targeting a specified coverage level.

\begin{table}[htb]
    \centering
    \caption{Out-of-sample performance comparison of \ConfRO{} and benchmarks.}
    \label{tab:performance_comparison}
    \begin{subtable}[b]{0.52\textwidth}
        \centering
        \caption{Mean objective under realized parameters.}
        \label{tab:cro_objective}
        \small
        \resizebox{\linewidth}{!}{\begin{tabular}{@{}lccccccc@{}}
            \toprule
            $\alpha$ & \makecell{\ConfRO{}-\\Box} & \makecell{\ConfRO{}-\\Ellipsoid} & \makecell{\ConfRO{}-\\Budget} & \makecell{KMeans-\\RO} & \makecell{KNN-\\RO} & \makecell{Ellipsoid-\\RO} & PTO \\
            \midrule
            0.6  & -1308 & -1297 & -1309 & -899 & -1031 & 0 &-1311\\
            0.7  & -1309 & -1296 & -1310 & -846 & -1021 & 0 &-1311\\
            0.8  & -1309 & -1295 & -1310 & -905 & -1064 & 0 &-1311\\
            0.85 & -1309 & -1294 & -1311 & -843 & -1044 & 0 &-1311\\
            0.9  & -1309 & -1293 & -1311 & -918 & -1051 & 0 &-1311\\
            0.95 & -1310 & -1291 & -1310 & -663 & -1002 & 0 &-1311\\
            \bottomrule
        \end{tabular}
        }
    \end{subtable}
    \hfill
    \begin{subtable}[b]{0.46\textwidth}
        \centering
        \caption{Empirical coverage level.}
        \small
        \resizebox{\linewidth}{!}{\begin{tabular}{@{}lcccccc@{}}
            \toprule
            $\alpha$ & \makecell{\ConfRO{}-\\Box} & \makecell{\ConfRO{}-\\Ellipsoid} & \makecell{\ConfRO{}-\\Budget} & \makecell{KMeans-\\RO} & \makecell{KNN-\\RO} & \makecell{Ellipsoid-\\RO} \\
            \midrule
            0.6  & 0.56 & 0.58 & 0.56 & 0.46 & 0.60 & 0.62 \\
            0.7  & 0.68 & 0.69 & 0.66 & 0.54 & 0.68 & 0.72 \\
            0.8  & 0.79 & 0.80 & 0.77 & 0.63 & 0.76 & 0.80 \\
            0.85 & 0.84 & 0.84 & 0.83 & 0.70 & 0.80 & 0.85 \\
            0.9  & 0.89 & 0.89 & 0.89 & 0.74 & 0.85 & 0.89 \\
            0.95 & 0.95 & 0.94 & 0.94 & 0.81 & 0.89 & 0.94 \\
            \bottomrule
        \end{tabular}
        }
    \end{subtable}
    \par\smallskip
    \parbox{\textwidth}{\scriptsize \textit{Note.} Panel (a) reports $a(\bm x,\bm c)=-\bm c^\top\bm x$; more negative values indicate higher realized utility. Ellipsoid-RO's overly conservative set yields the zero solution at every level. Panel (b) reports test-set empirical coverage, which may differ slightly from the target.}
    \vspace{-1em}
\end{table}

Table~\ref{tab:performance_comparison} shows that \ConfRO{} closely tracks the nominal coverage level $\alpha$, whereas KMeans-RO substantially undercovers at high levels. \ConfRO{} improves realized utility over KNN-RO and KMeans-RO by 23.2\%--97.7\% and nearly matches PTO while retaining calibrated protection. Ellipsoid-RO is overly conservative and returns the zero solution at every level. These findings support our central premise that the useful robustness scale lies in calibrated residual variation around a context-specific forecast rather than unconditional variability in $\bm c$.

\subsubsection{Validating the Correspondence of \ConfRO{} and \ConfRS{}.}

We next examine the reliability--target correspondence from Section~\ref{sec:ConfRO_ConfRS_equivalence}.
Let $\mathfrak{s}(\bm{c},\hat{\bm{c}})=\|(\bm{c}-\hat{\bm{c}})/\hat{\bm{r}}\|_1$, where $\hat{\bm{r}}$ denotes the estimated absolute residual. The conditions supporting the \ConfRO{}--\ConfRS{} correspondence are satisfied: (i) $a(\bm{x},\bm{c})=-\bm{c}^\top \bm{x}$ is affine in $\bm x$, (ii) the feasible decision region $\mathcal{X}$ is convex, (iii) uncertainty enters through a single objective, and (iv) strong duality and dual attainment hold for the inner score-constrained maximization in Assumption~\ref{ass:cro_crs_equiv}\ref{ass:equiv_dual}. Proposition~\ref{prop:cro_to_crs} therefore maps any attained solution of the \ConfRO{} dual representation to an optimal \ConfRS{} pair at its selected target, while Theorem~\ref{thm:crs_to_cro} maps any eligible interior target back to a matched \ConfRO{} radius. Thus, the models admit a generally set-valued parameter correspondence. Equality of their full optimal decision sets additionally requires the unique-target condition in Theorem~\ref{thm:crs_to_cro}\ref{thm:crs_to_cro_equivalence_item}.

\begin{figure}[htb]
    \centering
    \begin{minipage}[b]{0.43\textwidth}
        \centering
        \captionof{figure}{Step-like pattern in the parameter mapping
        between \ConfRS{} and \ConfRO{}.}
        \label{fig:parameter_mapping}
        \includegraphics[width=\textwidth]
            {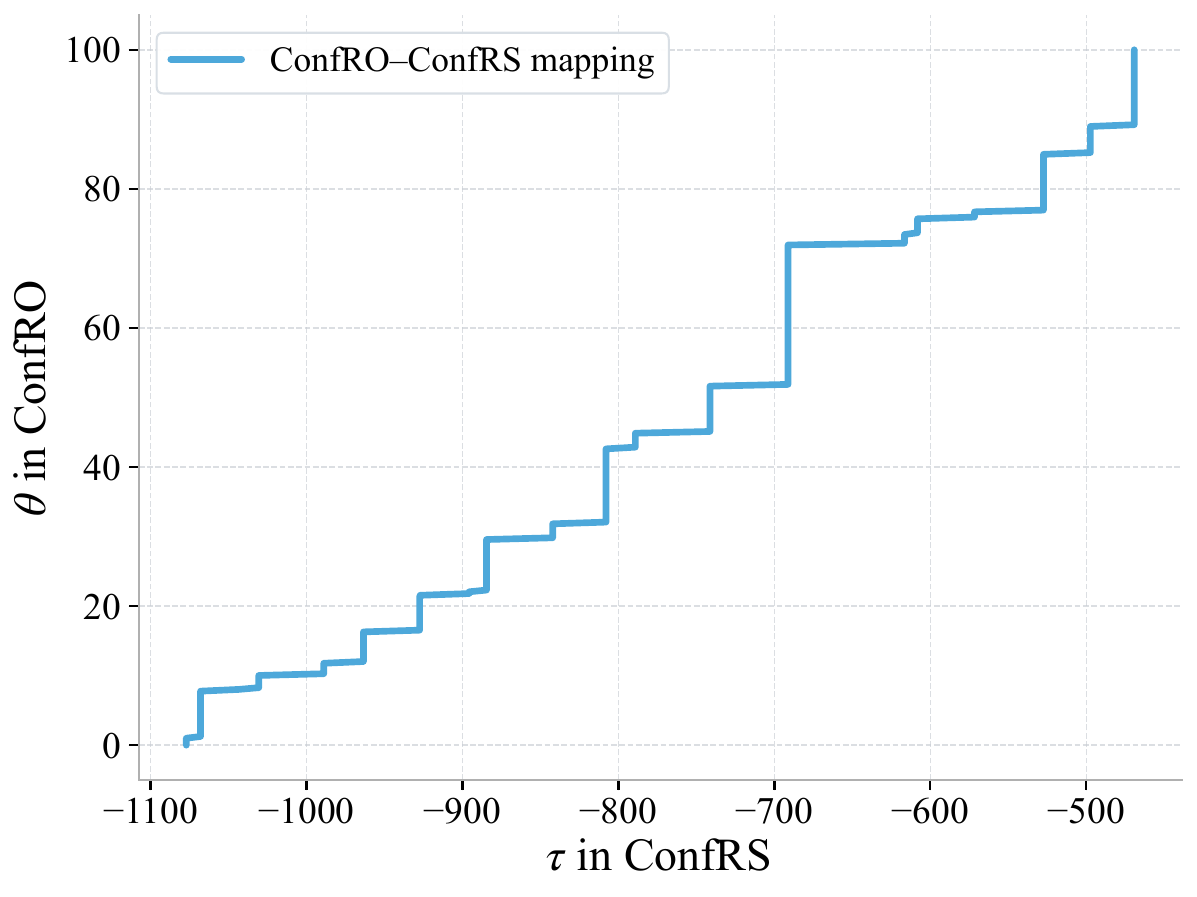}
    \end{minipage}
    \hfill
    \begin{minipage}[b]{0.5\textwidth}
        \centering
        \captionof{figure}{Performance comparison between the
        \ConfRS{} and DRS models.}
        \label{fig:crs_drs_comparison}
        \vspace{1em}
        \includegraphics[width=\textwidth]
            {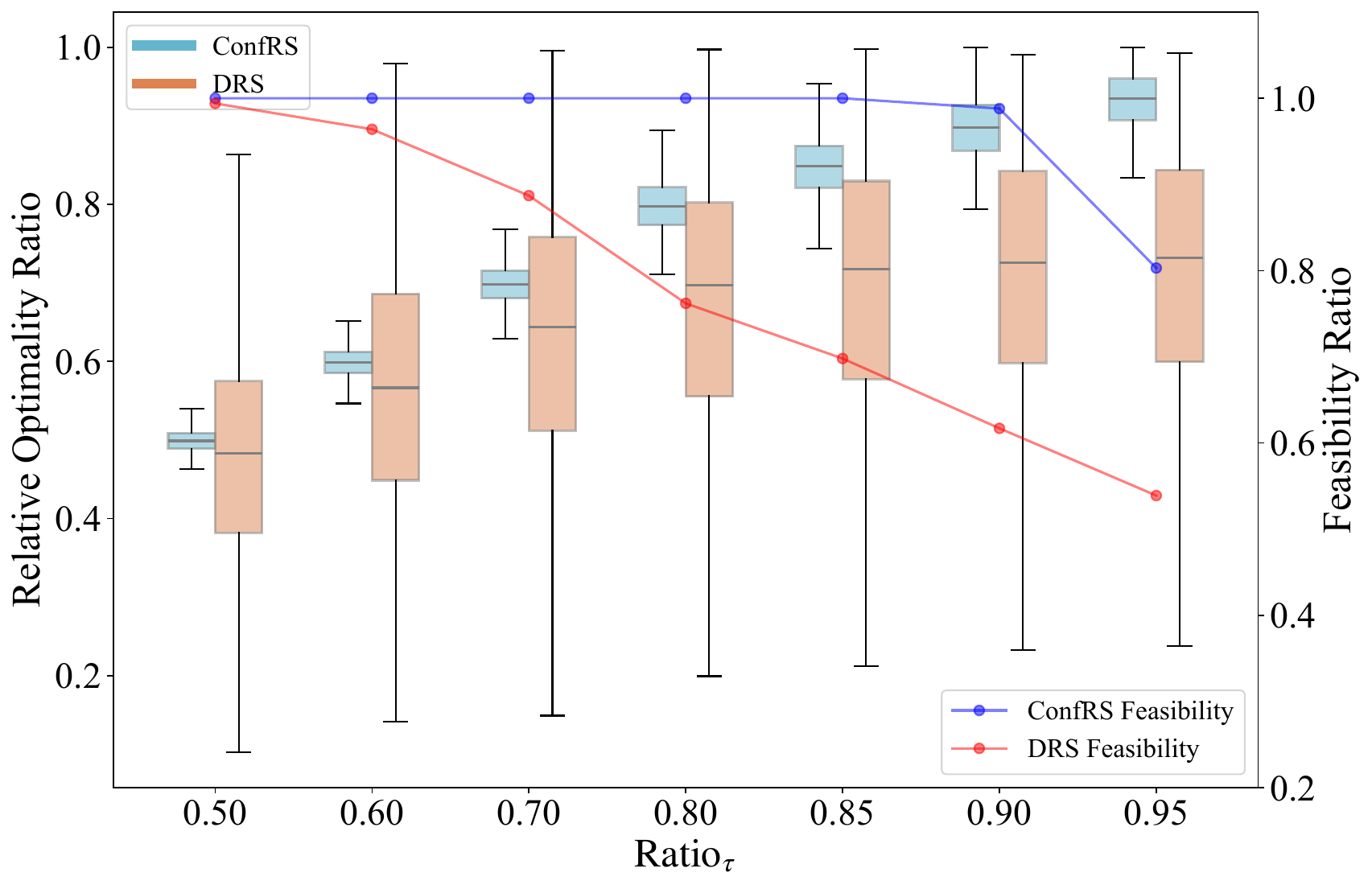}
    \end{minipage}
\end{figure}

Figure~\ref{fig:parameter_mapping} visualizes the parameter mapping between $\tau$ and $\theta$ for a representative instance, numerically corroborating our theoretical conclusion in Section~\ref{sec:ConfRO_ConfRS_equivalence}. Notably, the mapping curves exhibit an interesting step-like pattern, and this is because a single decision often remains optimal for a continuous range of target values rather than for a unique point.

\subsubsection{The Benefit of Conditioning Fragility on Prediction.}

\label{sec:sim_drs}

We use $S=40$ historical samples and the unscaled score $\mathfrak s(\bm c,\hat{\bm c})=\|\bm c-\hat{\bm c}\|_1$, which makes the comparison directly aligned with the DRS formulation. For each test instance, let $V_D$ denote the perfect-information objective value under the cost convention $a=-\bm c^\top\bm x$. We set the target $\tau=\rho_\tau V_D$, where $\rho_\tau\in\{0.5,0.6,0.7,0.8,0.85,0.9,0.95\}$; larger $\rho_\tau$ values make the target closer to the perfect-information benchmark and hence more stringent. Across 1{,}000 test instances, we evaluate feasibility and the relative optimality ratio $V_{RS}/V_D$, where $V_{RS}$ is the realized objective value; values closer to one indicate performance closer to the perfect-information benchmark.

Figure~\ref{fig:crs_drs_comparison} reveals two patterns. First, DRS becomes increasingly infeasible as the target tightens, i.e., $\rho_\tau$ increases, reflecting the coarseness of finite empirical samples when they are not conditioned on the current context. \ConfRS{} remains feasible for most instances and deteriorates noticeably only at the most stringent targets. Second, among feasible instances, \ConfRS{} achieves a higher mean optimality ratio $V_{RS}/V_D$ and lower dispersion. These findings support the mechanism developed in Section~\ref{sec:crs_framework}: using a context-specific forecast as the anchor for fragility yields decisions that are both more feasible and more stable than those based on an unconditional empirical distribution.

\subsection{Case Study: Multiperiod Inventory Management in Online Grocery}
\label{sec:case_study}

We next apply the score-calibrated robustness framework to a multiperiod inventory problem at \textit{Dingdong Maicai}, a large online grocery platform. Demand forecasting is particularly consequential in this setting because demand varies substantially across products, stores, and time. Our framework provides a robustness interface for translating predictions from industrial deep-learning demand models into reliable operational decisions.

To quantify the resulting operational benefits, we conduct a real-data study using \textsf{FreshRetailNet-50K}, a public dataset from \textit{Dingdong Maicai} \citep{2025freshretailnet-50k}. It contains granular hourly records for 50{,}000 store--SKU (stock-keeping-unit) pairs over 97 days across 898 stores in 18 major cities. The records combine hierarchical city, store, and product identifiers and hourly sales and inventory-status measures with discount and promotional-activity indicators, holiday and day-of-week information, and local weather measures such as temperature, humidity, wind, and precipitation.

Following the dataset's technical report, we use \textsf{TimesNet} \citep{Wu2023TimesNet} to impute demand observations missing because of stockouts. We partition the data chronologically, using the first 90 days for training and reserving the final seven days as a common holdout window for calibration and testing.

The robust inventory management problem accounts for holding and backlogging costs and uses a multiperiod robust optimization formulation with a replenishment schedule fixed before the planning horizon \citep{bertsimas2006robust, Erick2015robust}. For \ConfRO{}, the uncertainty set $\mathcal U_\alpha$ is calibrated to provide a prescribed coverage level $\alpha$ for the demand vector $\bm w\in\mathbb R^T$:
\begin{subequations}\label{prob:inv}
\begin{align}
\min_{\bm u,\bm y}\quad
    &\sum_{k=0}^{T-1}\left(cu_k+y_k\right)
    \label{prob:inv-obj}\\
\mathrm{s.t.}\quad
    &y_k\ge h\left(x_0+\sum_{i=0}^k(u_i-w_i)\right),
    &&\forall \bm w\in\mathcal{U}_\alpha,\quad
      k=0,\ldots,T-1,
    \label{prob:inv-holding}\\
    &y_k\ge -p\left(x_0+\sum_{i=0}^k(u_i-w_i)\right),
    &&\forall \bm w\in\mathcal{U}_\alpha,\quad
      k=0,\ldots,T-1,
    \label{prob:inv-shortage}\\
    &u_k\ge0,\quad y_k\ge0,
    &&k=0,\ldots,T-1.
    \label{prob:inv-nonnegative}
\end{align}
\end{subequations}
The planning horizon consists of $T=7$ periods. Here, $x_0$ denotes the initial inventory level, while $u_k$ and $w_k$ denote the replenishment decision and uncertain demand in period $k$, respectively. The unit procurement, holding, and shortage costs are denoted by $c$, $h$, and $p$, respectively. The objective~\eqref{prob:inv-obj} minimizes total procurement cost plus the periodwise inventory-cost bounds $y_k$. The cumulative term $x_0+\sum_{i=0}^k(u_i-w_i)$ is the end-of-period inventory position; for every demand trajectory $\bm w\in\mathcal U_\alpha$, constraints~\eqref{prob:inv-holding} and~\eqref{prob:inv-shortage} make $y_k$ upper-bound, respectively, the holding cost when this position is positive and the backlogging cost when it is negative. Thus, at optimality, $y_k$ is the worst-case period-$k$ inventory-imbalance cost over $\mathcal U_\alpha$. Constraint~\eqref{prob:inv-nonnegative} enforces nonnegativity.

\subsubsection{Predictor Choices and Calibration Pipeline.}

To exploit the heterogeneous demand signals over the seven-day forecast horizon, we compare three complementary predictors. (i) Temporal Fusion Transformer (\textsf{TFT}; \citealp{Lim2021temporal}) is particularly well suited to this setting: its variable-selection and gating mechanisms integrate static and time-varying covariates, while recurrent processing and attention capture local and longer-range temporal dependencies. (ii) \textsf{DLinear} \citep{zeng2023dlinear} provides a parsimonious benchmark that separates trend and seasonal components through linear layers. (iii) Similar Sample Average (\textsf{SSA}; \citealp{2025freshretailnet-50k}) provides an interpretable benchmark that weights historical demand by recency and similarity in holiday status, day of week, precipitation, and discount conditions.

Table~\ref{tab:prediction_performance} reports out-of-sample weighted absolute percentage error (WAPE), weighted percentage error (WPE), and mean absolute error (MAE). \textsf{TFT} ranks first or ties for first in seven of the nine metric-by-group comparisons and second in the other two, attaining the lowest WAPE and the lowest or tied-lowest MAE overall and in both SKU-volume groups. This pattern accords with the \textsf{FreshRetailNet-50K} technical report, which identifies \textsf{TFT} as the strongest overall forecaster across dataset segments \citep{2025freshretailnet-50k}. Its consistent accuracy across demand scales provides a strong predictive center for conformal calibration; we therefore use it as the default predictor for all downstream \ConfRO{} experiments.

\begin{table}[htbp]
\centering
\caption{Predictive performance comparison of different predictors.}
\label{tab:prediction_performance}
\small
\setlength{\tabcolsep}{5pt}

\begin{tabular}{l ccc ccc ccc}
\toprule
\multirow{2}{*}{{Method}} & \multicolumn{3}{c}{{Overall}} & \multicolumn{3}{c}{{High-Volume SKUs}} & \multicolumn{3}{c}{{Low-Volume SKUs}} \\
\cmidrule(lr){2-4} \cmidrule(lr){5-7} \cmidrule(lr){8-10}
 & WAPE & WPE & MAE & WAPE & WPE & MAE & WAPE & WPE & MAE \\
\midrule
\textsf{TFT} & \bestnum{29.2\%} & \secondnum{5.4\%} & \bestnum{0.363} & \bestnum{22.8\%} & \bestnum{2.1\%} & \bestnum{0.644} & \bestnum{38.1\%} & \secondnum{10.0\%} & \bestnum{0.267} \\
\textsf{DLinear} & \secondnum{30.6\%} & 6.6\% & \secondnum{0.381} & \secondnum{23.5\%} & \secondnum{2.7\%} & \secondnum{0.666} & 40.6\% & 12.0\% & 0.283 \\
\textsf{SSA} & 31.1\% & \bestnum{-1.2\%} & 0.387 & 26.0\% & -4.0\% & 0.739 & \secondnum{38.5\%} & \bestnum{2.70\%} & \bestnum{0.267} \\
\bottomrule
\end{tabular}

\par\smallskip
\parbox{\textwidth}{\scriptsize \textit{Note.} Shaded cells denote the best (blue) and second-best (orange) values.}
\vspace{-2em}
\end{table}

We next specify the conformal score used to calibrate forecast uncertainty. Drawing on the designs discussed in Section~\ref{subsubsec:cro_score_geometry}, we vary the score along two dimensions: the geometry of the induced uncertainty set (Box, Ellipsoid, or Budget) and whether a secondary residual model scales the score to reflect the expected magnitude of forecast errors.

Within the seven-day holdout window, we retain approximately 7{,}000 relatively stable store--SKU trajectories to limit extreme noise, randomly assigning 500 to testing and the remainder to calibration. The plausibility of cross-sectional exchangeability for this split is assessed in Appendix~\ref{app:cs_setting}. For each test instance, we construct $\mathcal U_\alpha$ and solve the robust inventory problem.

\subsubsection{Empirical Performance and Implementation Insights.}
\label{cs_results}

We evaluate the robust inventory policies across target coverage levels $\alpha\in[0.5,0.9]$. For comparison, we include Ellipsoid-RO, KNN-RO, and the PTO baseline, which directly uses the point forecast without robustness. To assess robustness to demand perturbations, we multiply each realized out-of-sample demand by an independent factor drawn uniformly from $[0.8,1.4]$.

\begin{figure}[htb]
    \centering
    \caption{Operational cost and empirical coverage: \ConfRO{} framework versus baselines.}
    \includegraphics[width=0.75\linewidth]{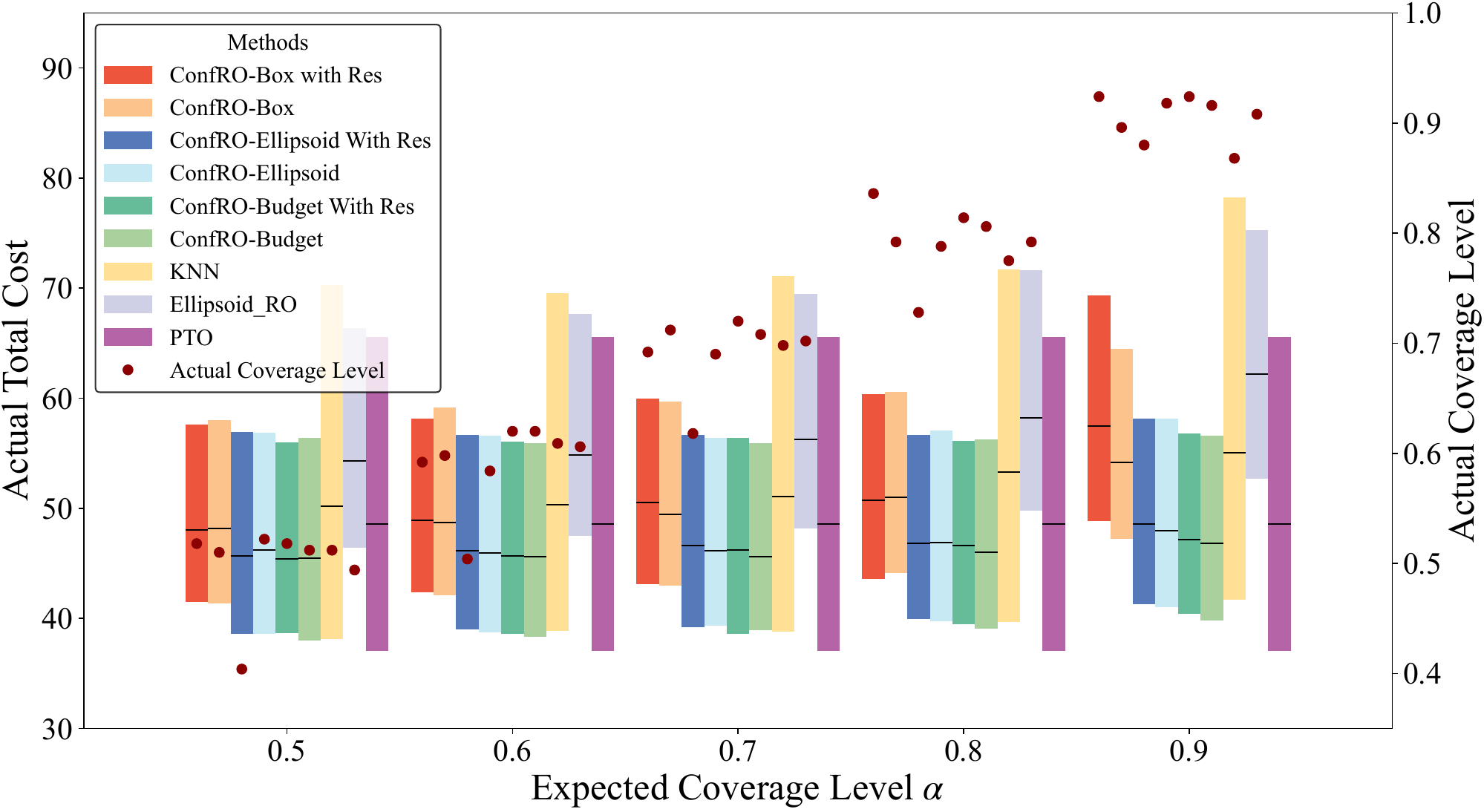}
    \label{fig:case_study_compare}
\end{figure}

Figure~\ref{fig:case_study_compare} summarizes the resulting trade-off between empirical coverage (reliability) and operational cost under representative cost parameters. Relative to deterministic PTO, \ConfRO{} is more resilient to demand perturbations. The Budget and Ellipsoid variants achieve both lower mean operational cost and lower cost variability, showing that explicit hedging against calibrated forecast residuals can dominate reliance on point forecasts alone. \ConfRO{} also outperforms the data-driven robust baselines. By anchoring the uncertainty set $\mathcal{U}_{\alpha}$ at a high-quality \textsf{TFT} forecast and calibrating the residual scale $\eta_\alpha$, \ConfRO{} avoids the static conservatism of Ellipsoid-RO while providing more reliable coverage control than KNN-RO.

The case study shows that the cost--reliability trade-off depends on three implementation choices: the forecast used to center the uncertainty set, the geometry induced by the conformal score, and the use of residual scaling.
\begin{enumerate}[label=(\roman*)]
    \item \emph{Prioritize forecast quality before tuning robustness:} Table~\ref{tab:prediction_performance} shows that \textsf{TFT} achieves the lowest overall WAPE and MAE (29.2\% and 0.363) and performs best on high-volume SKUs. It also leads to the best performance: at $\alpha=0.9$, the best \ConfRO{} variant reduces cost by 20.37\%, 26.95\%, and 7.74\% compared with KNN-RO, Ellipsoid-RO, and PTO, respectively (Table~\ref{tab:case_study_relative_performance}). Thus, reducing conservatism mainly relies on calibrating deviations around a strong context-specific predictor, rather than guarding against unconditional demand variation.
    \item \emph{Match uncertainty-set geometry to the structure of forecast errors:} The results favor geometries that capture joint deviations without imposing uniformly conservative bounds. The Budget score constrains total normalized deviation through an $L_1$ budget, allowing flexible allocation across periods while retaining a linear robust counterpart. It yields the lowest cost among the \ConfRO{} geometries at every tested coverage level, consistent with the fractional-knapsack results in Table~\ref{tab:performance_comparison}. The Ellipsoid score captures correlated errors through covariance-aware geometry and performs comparably in Figure~\ref{fig:case_study_compare}. The Box score instead bounds each coordinate separately and becomes increasingly conservative at higher coverage levels. These findings favor Budget sets for interpretability and tractability, and Ellipsoid sets when dependence can be estimated reliably.
    \item \emph{Residual scaling is not uniformly beneficial:} The adaptive scores use an additional residual-scale model to account for heteroscedastic forecast errors; we implement separate multilayer perceptrons (MLPs) for componentwise and global residual magnitudes. As shown in Figure~\ref{fig:case_study_compare}, the scaled variants do not consistently improve either cost or empirical coverage relative to their unscaled counterparts. One possible explanation is that estimation error in the secondary residual model offsets part of the benefit from adapting the set size. The value of residual scaling therefore depends on whether the additional scale model improves out-of-sample decision performance.
\end{enumerate}

\subsubsection{Conformal Robust Satisficing: Implementation and Comparison.}
\label{sec:cs_rs_exp}

The inventory formulation~\eqref{prob:inv} contains $2T$ uncertainty-dependent inequalities, comprising a holding-cost and a shortage-cost inequality for each period. A direct application of the multi-constraint \ConfRS{} formulation would assign a separate target to every inequality. However, the two inequalities in each period are not distinct performance criteria; they are the two branches of a single piecewise-linear inventory cost. We therefore define the realized inventory cost in period $k$ as
\[
    R_k(\bm u,\bm w):=\max\left\{
    h I_{k+1}(\bm u,\bm w),-p I_{k+1}(\bm u,\bm w)\right\}, \qquad k = 0, \dots, T-1
\]
where $I_{k+1}(\bm u,\bm w)=x_0+\sum_{i=0}^{k}(u_i-w_i)$ is the end-of-period inventory position. The acceptable target $\tau$ constrains total cost over the planning horizon rather than prescribing separate targets for individual periods. We therefore avoid imposing an exogenous target allocation and instead allow the model to determine period-specific inventory-cost allowances $t_k$ and corresponding fragilities $q_k$. Given the demand support $\mathfrak D$, the resulting \ConfRS{} formulation is
\begin{equation*}
\begin{aligned}
    \min_{\bm u,\bm t,\bm q}\quad
        & \sum_{k=0}^{T-1} q_k\\
    \mathrm{s.t.}\quad
        & \sum_{k=0}^{T-1}(cu_k+t_k)\leq \tau,\\
        & R_k(\bm u,\bm w)-t_k
        \leq q_k s(\bm w,\hat{\bm w}),
        \qquad
        \forall \bm w\in\mathfrak D,\
        k=0,\ldots,T-1,\\
        & \bm u,\bm t,\bm q\geq \bm 0,
\end{aligned}
\tag{\ConfRS{}-Inv}\label{eq:ConfRS_Inv}
\end{equation*}

To assess the value of prediction-centered fragility and prediction accuracy, we evaluate \ConfRS{} under different predictors and compare it against a distributionally robust satisficing (DRS) benchmark that retains the same target, parameterization, and demand support but replaces the point anchor $\hat{\bm w}$ with an empirical distribution of historical demand.

We use the same forecasts and $500$ test instances as in the \ConfRO{} experiment. Consistent with the simulation study in Section~\ref{sec:sim_drs}, we use the $L_1$ distance, $s(\bm w,\hat{\bm w})=\|\bm w-\hat{\bm w}\|_1$, as both the conformal score in \ConfRS{} and the transport cost in DRS. To parameterize the target for instance $i$, we first solve the deterministic inventory problem under the realized demand $\bm w_i$ and denote the resulting perfect-information cost by $V_i^0$. We then set
$\tau_i=\rho V_i^0$, $\rho\in\{1.05,1.10,1.15,1.20,1.30,1.40\}$.
At each target level, we solve all methods for the $500$ test instances and evaluate realized total cost under the observed demand, normalized by $V_i^0$. We also compare feasibility, defined as whether an instance admits a finite-fragility satisficing certificate at the specified target. For each target ratio $\rho$, let $N(\rho)$ denote the number of test instances for which all compared methods are feasible; realized relative costs are evaluated on this common subset. Because the fragilities of \ConfRS{} and DRS are governed by the asymptotic shortage-cost slope, both \ConfRS{} and DRS would be degenerate under an unbounded demand support. We therefore impose a data-driven support upper bound using the maximum historical demand observed for each store--SKU (details in Appendix~\ref{app:case_study_rs}).

\begin{figure}[htb]
    \centering
    \begin{minipage}[b]{0.48\textwidth}
        \centering
        \captionof{figure}{Realized relative cost comparison.}
        \label{fig:case_study_rs_cost}
        \includegraphics[width=\textwidth]
            {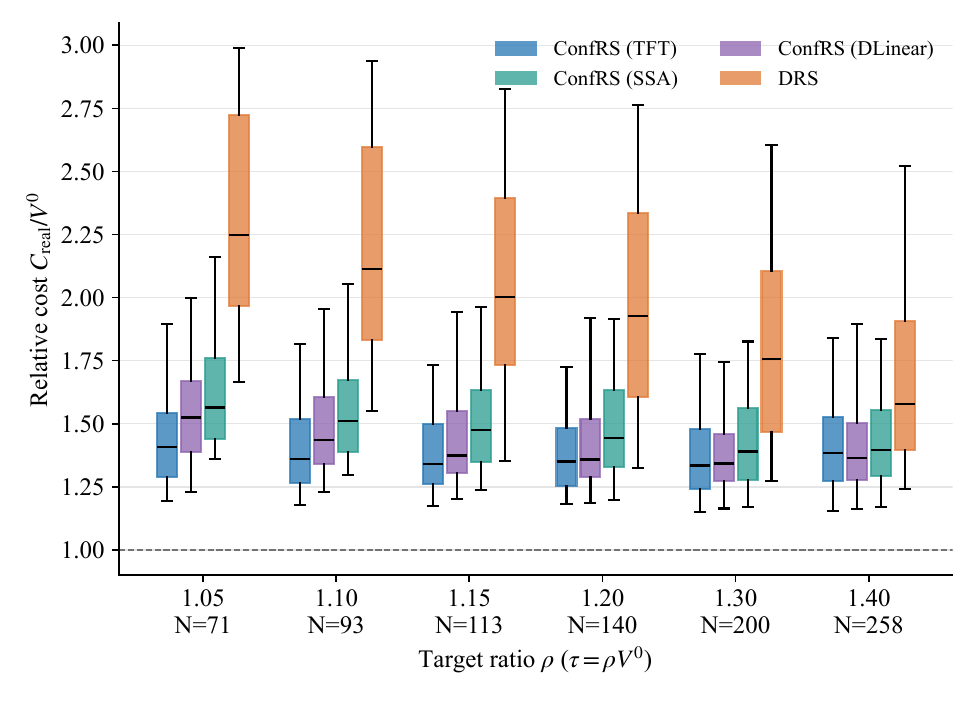}
    \end{minipage}
    \hfill
    \begin{minipage}[b]{0.48\textwidth}
        \centering
        \captionof{figure}{Feasibility ratio comparison.}
        \label{fig:case_study_rs_feasibility}
        \includegraphics[width=\textwidth]
            {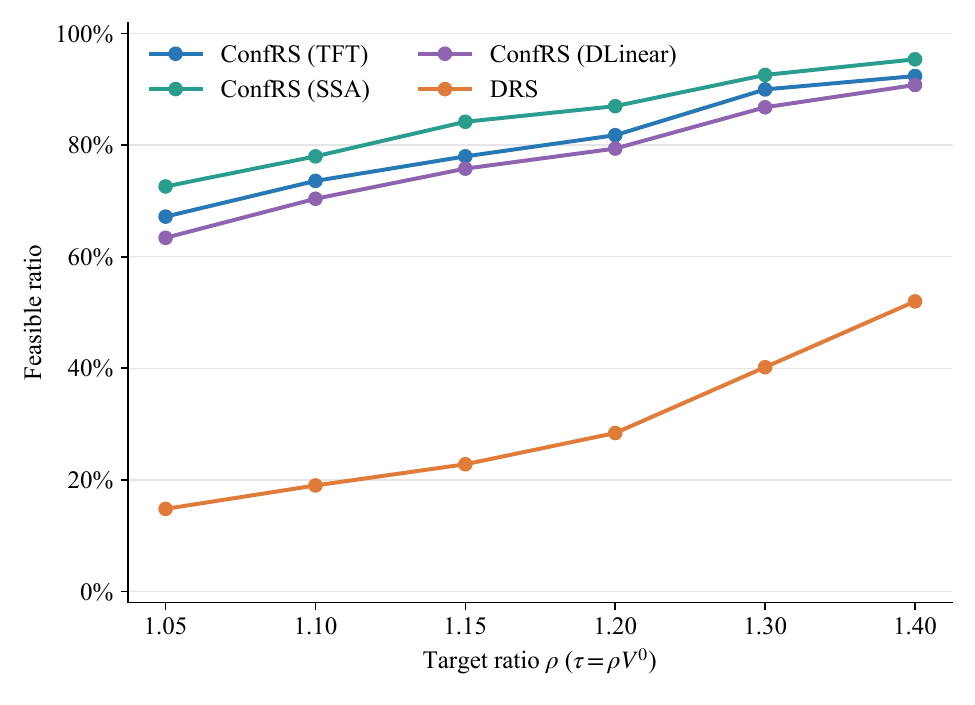}
    \end{minipage}
\end{figure}

As shown in Figures~\ref{fig:case_study_rs_cost} and~\ref{fig:case_study_rs_feasibility}, at each tested target ratio, all three \ConfRS{} variants have higher observed certificate-feasibility rates than DRS and lower mean realized relative costs on the common feasible subset. Detailed means are reported in Appendix Table~\ref{tab:case_study_rs_detailed}. The predictor rankings differ across the two measures. The \textsf{TFT}-based variant has the lowest mean cost for $\rho=1.05$ through $1.30$, while \textsf{DLinear} is marginally lowest at $\rho=1.40$. The \textsf{SSA}-based variant has the highest mean cost among the three \ConfRS{} variants on the common feasible subset, but its mean cost remains below that of DRS at each target ratio, while it attains the highest feasibility rate. These findings suggest that, in this case study, a context-specific forecast may provide a useful reference for fragility relative to the historical empirical distribution. They also indicate that forecast accuracy need not translate directly into certificate feasibility.

\section{Concluding Remarks}
\label{sec:conclusion}

This paper develops a score-calibrated robustness interface for prediction-driven decision-making. Its central message is that a conformal score can provide the missing link between black-box predictions and downstream robustness. We introduce \ConfRS{}, a target-oriented formulation that complements recent \ConfRO{}-type methods and shows that conformal prediction can calibrate not only uncertainty sets, but also decision-relevant robustness requirements around predictions. When uncertainty enters only through the objective and suitable convexity and duality conditions hold, \ConfRO{} and \ConfRS{} provide alternative parameterizations of the same score-calibrated robustness frontier, with a direct correspondence between reliability levels and acceptable performance targets.

The numerical studies support this unified perspective. In the fractional-knapsack experiments, \ConfRO{} reduces conservatism relative to the benchmark methods while maintaining empirical coverage near the nominal levels, and the observed mappings between \ConfRO{} and \ConfRS{} agree with the theoretical characterization. The online-grocery case study for both robustness views further demonstrates that the framework can be combined with deep-learning-based demand predictors in a complex operational environment. From a managerial perspective, anchoring uncertainty at a strong forecast and using score geometries that capture joint deviations can reduce both the level and variability of operating costs, whereas overly complex residual scaling may be counterproductive.

This work represents a step toward combining tractable decision models with powerful black-box predictors. Several directions merit further study. First, the decision-level correspondence between \ConfRO{} and \ConfRS{} may enable efficient algorithms for tracing the robustness frontier, such as warm starts across reliability or target levels. It may also support decision interfaces that translate reliability requirements into implied targets, and vice versa. Second, combining score calibration with sequential decision-making or other dynamic models could lead to hybrid frameworks that retain predictor-agnostic finite-sample calibration while capturing richer operational environments. Third, extending conditional calibration under more general data conditions remains important. Our main framework relies on finite-sample marginal validity under exchangeability, while Appendix~\ref{sec:conditional_cro} establishes pointwise asymptotic conditional coverage under smooth, low-dimensional context representations. Scalable methods for high-dimensional contexts with practical calibration sample sizes could yield more instance-specific uncertainty sets in \ConfRO{} and sharper target certificates in \ConfRS{}. These directions would further advance the broader goal of converting flexible predictors into reliable and efficient operational decisions.

\bibliographystyle{preprint}
\bibliography{references}

\newpage

\def\AppendixFontSize{\normalsize}
\def\AppendixSectionHeadSize{\fs.13.15.\relax}

\begin{APPENDICES}

\renewcommand{\thesection}{\Alph{section}}
\renewcommand{\theHsection}{appendix.\arabic{section}}

\renewcommand{\theHsubsection}{appendix.\arabic{section}.\arabic{subsection}}
\renewcommand{\theHsubsubsection}{appendix.\arabic{section}.\arabic{subsection}.\arabic{subsubsection}}

\numberwithin{equation}{section}
\numberwithin{theorem}{section}
\numberwithin{proposition}{section}
\numberwithin{lemma}{section}
\numberwithin{definition}{section}
\numberwithin{figure}{section}
\numberwithin{table}{section}
\numberwithin{assumption}{section}
\numberwithin{corollary}{section}
\numberwithin{algorithm}{section}

\section{Proofs and Additional Theoretical Results}
\label{app:proof}

Appendix~\ref{app:proof} collects the proofs and additional theoretical results. We first provide the proofs for Sections~\ref{sec:cus_framework}, \ref{sec:cro_crs_framework}, and~\ref{sec:ConfRO_ConfRS_equivalence}. Appendix~\ref{sec:conditional_cro} then develops the localized calibration extension and its conditional-coverage analysis.

\subsection{Proofs for Section~\ref{sec:cp_preliminary}}

Algorithm~\ref{alg:general_cro} records the split-conformal calibration procedure described in Section~\ref{sec:cp_preliminary}.

\begin{algorithm}[H]
    \caption{Conformal Uncertainty Set Construction}
    \label{alg:general_cro}
    \begin{algorithmic}[1]
    \State \textbf{Input}: Training set $\mathcal D_{\mathrm{train}}$, calibration set $\mathcal D_{\mathrm{cal}}=\{(\bm z_i,\bm d_i)\}_{i=1}^{N_c}$, target coverage level $\alpha\in(0,1)$.
    \State Fit the predictor $\hat f$ on $\mathcal D_{\mathrm{train}}$; for a test context, $\hat f(\bm z)$ will serve as the center of the uncertainty set.
    \State Design a conformal score $s_{\hat f}(\bm z,\bm d)$ whose sublevel sets have the desired geometry.
    \State Compute calibration scores
    \[
        S_i\leftarrow s_{\hat f}(\bm z_i,\bm d_i),\qquad i=1,\ldots,N_c.
    \]
    \State Sort the scores as $S_{(1)}\le\cdots\le S_{(N_c)}$, set $S_{(N_c+1)}:=+\infty$, and let
    \[
        k_\alpha=\lceil (N_c+1)\alpha\rceil,\qquad \eta_\alpha=S_{(k_\alpha)}.
    \]
    \State \textbf{Output}: For any test instance $\bm z_{\mathrm{test}}$, return
    \[
        \mathcal U_\alpha(\bm z_{\mathrm{test}})=\{\bm d\in\mathfrak D:s_{\hat f}(\bm z_{\mathrm{test}},\bm d)\le \eta_\alpha\}.
    \]
    \end{algorithmic}
\end{algorithm}

\begin{proof}{Proof of Proposition~\ref{prop:calibration_transfer}.}
We first recall the standard finite-sample validity result underlying split conformal prediction.

\begin{lemma}
\label{lem:conformal_coverage}
Suppose Assumption~\ref{asp:exchangeability} holds. Let $S_i=s_{\hat f}(\bm Z_i,\bm D_i)$ for $i=1,\ldots,n$, let $S_{(1)}\le\cdots\le S_{(n)}$ denote their order statistics, and set $S_{(n+1)}:=+\infty$. For $\alpha\in(0,1)$, define $k_\alpha=\lceil(n+1)\alpha\rceil$, $\eta_\alpha=S_{(k_\alpha)}$, and
$
    \mathcal U_\alpha(\bm Z)=\{\bm d\in\mathfrak D:s_{\hat f}(\bm Z,\bm d)\le\eta_\alpha\}.
$
Then
\[
    \mathbb P\bigl(\bm D_{\mathrm{test}}\in\mathcal U_\alpha(\bm Z_{\mathrm{test}})\bigr)\ge\alpha.
\]
If, in addition, the $n+1$ calibration and test scores are almost surely pairwise distinct, then
\[
    \mathbb P\bigl(\bm D_{\mathrm{test}}\in\mathcal U_\alpha(\bm Z_{\mathrm{test}})\bigr)
    \le\alpha+\frac{1}{n+1}.
\]
\end{lemma}

\begin{proof}{Proof of Lemma~\ref{lem:conformal_coverage}.}
This is a canonical result in conformal prediction \citep{vovk2012conditional,angelopoulos2023conformal}; we include a brief proof for completeness. Treat the predictor and all score-design choices as fixed before calibration, and let $S_{\mathrm{test}}=s_{\hat f}(\bm Z_{\mathrm{test}},\bm D_{\mathrm{test}})$. Assumption~\ref{asp:exchangeability} makes the $n+1$ calibration and test scores exchangeable.

If $k_\alpha=n+1$, then $\eta_\alpha=+\infty$, so the coverage probability is one; both bounds in the lemma follow because $\alpha>n/(n+1)$. Now suppose $k_\alpha\le n$. To accommodate ties, attach i.i.d. auxiliary variables $U_1,\ldots,U_n,U_{\mathrm{test}}\sim\mathrm{Unif}(0,1)$ and rank the pairs $(S_i,U_i)$ lexicographically. The rank $R_{\mathrm{test}}$ of $(S_{\mathrm{test}},U_{\mathrm{test}})$ among the $n+1$ exchangeable pairs is uniform on $\{1,\ldots,n+1\}$. Moreover, $R_{\mathrm{test}}\le k_\alpha$ implies $S_{\mathrm{test}}\le S_{(k_\alpha)}=\eta_\alpha$. Hence
\[
\mathbb P(S_{\mathrm{test}}\le\eta_\alpha)
\ge\mathbb P(R_{\mathrm{test}}\le k_\alpha)
=\frac{k_\alpha}{n+1}
\ge\alpha.
\]
When the scores are almost surely pairwise distinct, the implication is an equivalence, and therefore
\[
\mathbb P(S_{\mathrm{test}}\le\eta_\alpha)
=\frac{k_\alpha}{n+1}
\le\alpha+\frac{1}{n+1}.
\]
Finally, $S_{\mathrm{test}}\le\eta_\alpha$ is precisely the event $\bm D_{\mathrm{test}}\in\mathcal U_\alpha(\bm Z_{\mathrm{test}})$, which proves the lemma.
\hfill\halmos\end{proof}

We now apply Lemma~\ref{lem:conformal_coverage}. Define the conformal coverage event
$
    E_\alpha:=\{\bm D_{\mathrm{test}}\in\mathcal U_\alpha(\bm Z_{\mathrm{test}})\}.
$
The lemma gives $\mathbb P(E_\alpha)\ge\alpha$. By the assumed pathwise certificate, almost surely on $E_\alpha$,
\[
    h_m(\tilde{\bm x},\bm D_{\mathrm{test}})\le0,
    \qquad m=1,\ldots,M.
\]
Consequently,
\begin{align*}
\mathbb P\left\{h_m(\tilde{\bm x},\bm D_{\mathrm{test}})\le0,
    \ m=1,\ldots,M\right\}
&\ge\mathbb P(E_\alpha) \ge\alpha,
\end{align*}
which proves the proposition.
\hfill\halmos
\end{proof}

\begin{proposition}
    \label{prop:sample_complexity}
    Suppose Assumption~\ref{asp:iid_stability} holds and the fixed score $S=s_{\hat f}(\bm Z,\bm D)$ has a continuous cumulative distribution function (CDF). Then the calibration-conditional coverage satisfies a one-sided concentration bound at the canonical nonparametric rate. In particular, for any $\delta\in(0,1)$,
    $
    \mathbb{P}\!\left(p(\mathcal{D}_n)\ge \alpha-\sqrt{\frac{\log(1/\delta)}{2n}}\right)\ge 1-\delta,
    $
    and hence $p(\mathcal{D}_n)\ge \alpha-\mathcal{O}(n^{-1/2})$ with high probability.
\end{proposition}

\begin{proof}{Proof of Proposition~\ref{prop:sample_complexity}.}
We first recall the order-statistic result that controls the lower tail of the calibration-conditional coverage.

\begin{lemma}
\label{lem:sample_complexity}
Under Assumption~\ref{asp:iid_stability}, suppose the score distribution is continuous, and let $k_\alpha=\lceil(n+1)\alpha\rceil$. Define
$
p(\mathcal D_n)
:=\mathbb P\!\left(\bm D_{\mathrm{test}}\in
\mathcal U_\alpha(\bm Z_{\mathrm{test}})\mid\mathcal D_n\right).
$
If $k_\alpha\le n$, then
$
p(\mathcal D_n)\sim\mathrm{Beta}(k_\alpha,n+1-k_\alpha).
$
If $k_\alpha=n+1$, then $p(\mathcal D_n)=1$ almost surely. In either case, for every $\Delta>0$,
\begin{equation}
\mathbb P\!\left(p(\mathcal D_n)<\alpha-\Delta\right)
\le e^{-2n\Delta^2}.
\end{equation}
\end{lemma}

\begin{proof}{Proof of Lemma~\ref{lem:sample_complexity}.}
Because the fitted score function is fixed before calibration, let $F$ denote the CDF of $S=s_{\hat f}(\bm Z,\bm D)$. Under Assumption~\ref{asp:iid_stability} and continuity of $F$, the probability integral transforms $U_i=F(S_i)$ are i.i.d. $\mathrm{Unif}(0,1)$. Write $k=k_\alpha$. If $k\le n$, then, conditional on $\mathcal D_n$, independence of the test pair gives
\begin{align*}
p(\mathcal D_n)
&=\mathbb P\!\left(S_{\mathrm{test}}\le S_{(k)}\mid\mathcal D_n\right) =F(S_{(k)})
=U_{(k)}.
\end{align*}
The $k$th order statistic of $n$ independent uniform random variables has the $\mathrm{Beta}(k,n+1-k)$ distribution. If $k=n+1$, the conformal convention $S_{(n+1)}=+\infty$ instead gives $p(\mathcal D_n)=1$ almost surely.

It remains to establish the lower-tail bound. Fix $\Delta>0$ and set $t=\alpha-\Delta$. If $t\le0$ or $k=n+1$, the result is immediate. Otherwise, the order-statistic identity gives
\[
\mathbb P\!\left(p(\mathcal D_n)<t\right)
=\mathbb P\!\left(U_{(k)}<t\right)
=\mathbb P\!\left(B_{n,t}\ge k\right),
\qquad B_{n,t}\sim\mathrm{Binomial}(n,t).
\]
Because $k=\lceil(n+1)\alpha\rceil>n\alpha$ and $t=\alpha-\Delta$, Hoeffding's inequality yields
\[
\mathbb P(B_{n,t}\ge k)
\le\mathbb P\!\left(\frac{B_{n,t}}{n}-t\ge\alpha-t\right)
\le\exp\{-2n(\alpha-t)^2\}
=e^{-2n\Delta^2}.
\]
This proves the lemma.
\hfill\halmos\end{proof}

We now complete the proof of the proposition. Lemma~\ref{lem:sample_complexity} implies that, for every $\Delta>0$,
\begin{equation}
\label{eq:tail_bound_cor}
\mathbb{P}\bigl(p(\mathcal{D}_n)< \alpha-\Delta\bigr)\le e^{-2n\Delta^2}.
\end{equation}
Fix an arbitrary confidence level $\delta\in(0,1)$ and set
$
\Delta_n(\delta) := \sqrt{\frac{\log(1/\delta)}{2n}}.
$
Substituting $\Delta=\Delta_n(\delta)$ into \eqref{eq:tail_bound_cor} yields
\[
\mathbb{P}\left(p(\mathcal{D}_n)< \alpha-\sqrt{\frac{\log(1/\delta)}{2n}}\right)
\le \exp\left(-2n\cdot \frac{\log(1/\delta)}{2n}\right)=\delta.
\]
Equivalently,
\[
\mathbb{P}\left(p(\mathcal{D}_n)\ge \alpha-\sqrt{\frac{\log(1/\delta)}{2n}}\right)\ge 1-\delta.
\]
Hence, for any fixed $\delta$, with probability at least $1-\delta$ we have
\[
p(\mathcal{D}_n)\ge\alpha - C_\delta\,n^{-1/2},
\qquad\text{where}\qquad
C_\delta:=\sqrt{\frac{\log(1/\delta)}{2}}.
\]
This establishes that the calibration-conditional coverage deviates below $\alpha$ by at most
$\mathcal{O}(n^{-1/2})$ with high probability, proving the claimed convergence rate.
\hfill\halmos
\end{proof}

\subsection{Proofs for Section~\ref{sec:cro_framework}}
\begin{proof}{Proof of Proposition~\ref{prop:convex_bound}.}
Fix $\bm z_{\mathrm{test}}$ and $\bm d_{\mathrm{test}}\in\mathcal U_\alpha(\bm z_{\mathrm{test}})$, and write
\[
    \mathcal U:=\mathcal U_\alpha(\bm z_{\mathrm{test}}),
    \qquad
    \hat{\bm d}:=\hat f(\bm z_{\mathrm{test}}).
\]
Define
$A_i^{\mathcal U}(\bm x):=\sup_{\bm d\in\mathcal U}a_i(\bm x,\bm d)$,
$i\in\{0\}\cup[I]$. Because $\bm d_{\mathrm{test}}\in\mathcal U$, every robust-feasible decision is feasible for the realized problem, and
$A_0^{\mathcal U}(\bm x)\ge a_0(\bm x,\bm d_{\mathrm{test}})$.
Feasible-set containment and objective dominance therefore imply
$
    \delta_{\mathrm{Conv}}(\bm z_{\mathrm{test}},\bm d_{\mathrm{test}}) \ge 0.
$

Let $\bm\lambda^R=(\lambda_1^R,\ldots,\lambda_I^R)$ be the dual-optimal multiplier vector in the proposition, and define the robust Lagrangian
\[
    \mathcal L_R(\bm x,\bm\lambda)
    :=
    A_0^{\mathcal U}(\bm x)
    +\sum_{i=1}^I\lambda_i
      \bigl(A_i^{\mathcal U}(\bm x)-b_i\bigr).
\]
Strong duality and dual attainment for the outer robust program give
\[
    V_{\ConfRO}
    =\inf_{\bm x\in\mathcal X}\mathcal L_R(\bm x,\bm\lambda^R).
\]
Moreover, realized feasibility of $\bm x^\star_{\mathrm{test}}$ gives
$a_i(\bm x^\star_{\mathrm{test}},\bm d_{\mathrm{test}})\le b_i$
for all $i\in[I]$. Hence,
\begin{align*}
\delta_{\mathrm{Conv}}(\bm z_{\mathrm{test}},\bm d_{\mathrm{test}})
&=V_{\ConfRO}
  -a_0(\bm x^\star_{\mathrm{test}},\bm d_{\mathrm{test}})\\
&\le
  A_0^{\mathcal U}(\bm x^\star_{\mathrm{test}})
  -a_0(\bm x^\star_{\mathrm{test}},\bm d_{\mathrm{test}})
  +\sum_{i=1}^I\lambda_i^R
   \bigl(A_i^{\mathcal U}(\bm x^\star_{\mathrm{test}})-b_i\bigr)\\
&\le
  A_0^{\mathcal U}(\bm x^\star_{\mathrm{test}})
  -a_0(\bm x^\star_{\mathrm{test}},\bm d_{\mathrm{test}})\\
&\quad
  +\sum_{i=1}^I\lambda_i^R
   \bigl[
      A_i^{\mathcal U}(\bm x^\star_{\mathrm{test}})
      -a_i(\bm x^\star_{\mathrm{test}},\bm d_{\mathrm{test}})
   \bigr],
\end{align*}
where the first inequality evaluates the Lagrangian infimum at
$\bm x^\star_{\mathrm{test}}$, and the second also uses
$\bm\lambda^R\ge\bm{0}$.

For every $i\in\{0\}\cup[I]$, the score-sensitivity definition yields
\begin{align*}
&A_i^{\mathcal U}(\bm x^\star_{\mathrm{test}})
-a_i(\bm x^\star_{\mathrm{test}},\bm d_{\mathrm{test}})\\
&\quad\le
\sup_{\bm d\in\mathcal U}
\left|
    a_i(\bm x^\star_{\mathrm{test}},\bm d)
    -a_i(\bm x^\star_{\mathrm{test}},\bm d_{\mathrm{test}})
\right|\\
&\quad\le
\mathcal L_{\mathfrak s,i}(\bm x^\star_{\mathrm{test}})
\sup_{\bm d\in\mathcal U}
\mathfrak s(\bm d,\bm d_{\mathrm{test}}).
\end{align*}
When $\mathcal U$ is a singleton, each envelope difference is zero, and the result follows from the convention $\mathcal L_{\mathfrak s,i}(\bm x)=0$. Otherwise, the triangle inequality and metric symmetry give, for every $\bm d\in\mathcal U$,
\[
    \mathfrak s(\bm d,\bm d_{\mathrm{test}})
    \le
    \mathfrak s(\bm d,\hat{\bm d})
    +\mathfrak s(\hat{\bm d},\bm d_{\mathrm{test}})
    \le2\eta_\alpha.
\]
Substituting this bound for each objective and constraint envelope difference proves
\[
    0\le
    \delta_{\mathrm{Conv}}(\bm z_{\mathrm{test}},\bm d_{\mathrm{test}})
    \le
    2\eta_\alpha
    \left[
        \mathcal L_{\mathfrak s,0}(\bm x^\star_{\mathrm{test}})
        +\sum_{i=1}^I\lambda_i^R
        \mathcal L_{\mathfrak s,i}(\bm x^\star_{\mathrm{test}})
    \right].
\]
\hfill\halmos
\end{proof}

\begin{proof}{Derivation for Remark~\ref{rem:linear_case}.}
Under the linear specialization, write
\[
    \mathcal U:=\mathcal U_\alpha(\bm z_{\mathrm{test}}),
    \qquad
    \sigma_{\mathcal U}(\bm x)
    :=\sup_{\bm d\in\mathcal U}\bm d^\top\bm x.
\]
The robust Lagrangian and, by the assumed strong duality of the outer robust program, the robust value satisfy
\[
\begin{aligned}
    \mathcal L_R(\bm x,\lambda)
    &:=\bm c^\top\bm x+\lambda[\sigma_{\mathcal U}(\bm x)-b],\\
    V_{\ConfRO}
    &=\inf_{\bm x\in\mathcal X}\mathcal L_R(\bm x,\lambda^R).
\end{aligned}
\]
Evaluating the infimum at $\bm x^\star_{\mathrm{test}}$ and using
$V_{\mathrm{test}}=\bm c^\top\bm x^\star_{\mathrm{test}}$ gives
\begin{align*}
0\le\delta_{\mathrm{Lin}}(\bm z_{\mathrm{test}},\bm d_{\mathrm{test}})
&\le
\lambda^R\bigl[\sigma_{\mathcal U}(\bm x^\star_{\mathrm{test}})-b\bigr]\\
&\le
\lambda^R\bigl[
    \sigma_{\mathcal U}(\bm x^\star_{\mathrm{test}})
    -\bm d_{\mathrm{test}}^\top\bm x^\star_{\mathrm{test}}
\bigr]\\
&=
\lambda^R\sup_{\bm d\in\mathcal U}
(\bm d-\bm d_{\mathrm{test}})^\top\bm x^\star_{\mathrm{test}}\\
&\le
2\eta_\alpha\lambda^R
\mathcal L_{\mathfrak s}(\bm x^\star_{\mathrm{test}}).
\end{align*}
The second inequality uses realized feasibility,
$\bm d_{\mathrm{test}}^\top\bm x^\star_{\mathrm{test}}\le b$.
For the last inequality, taking absolute values and applying the definition of
$\mathcal L_{\mathfrak s}$ bounds the support-function difference by
$\mathcal L_{\mathfrak s}(\bm x^\star_{\mathrm{test}})
\sup_{\bm d\in\mathcal U}\mathfrak s(\bm d,\bm d_{\mathrm{test}})$. Lastly,
the metric-ball argument in the proposition bounds the supremum by
$2\eta_\alpha$.
\hfill\halmos
\end{proof}

\subsection{Proofs for Section~\ref{sec:crs_framework}}
\begin{proof}{Proof of Theorem~\ref{thm:fragility_measure}.}
For the lower semicontinuity claim, we view $\rho$ as an extended-real-valued functional on $\{v: \mathfrak{D} \rightarrow \mathbb{R}\}$ equipped with the topology of pointwise convergence. Equivalently, the argument below applies to any function space in which pointwise evaluations $v\mapsto v(\bm{d})$ are continuous. Let $K_v=\{k\ge 0:v(\bm{d})\le k\cdot\mathfrak{s}(\bm{d}), \forall \bm{d}\in\mathfrak{D}\}$. By definition $\rho(v)=\inf K_v$ with $\inf \varnothing=+\infty$.

We first prove lower semicontinuity. Consider the epigraph $\text{epi}(\rho)=\{(v,k):k\ge \rho(v)\}$. Because $K_v$ is an intersection of closed half-lines in $k$, whenever it is nonempty, it is closed and upward closed. Thus
\[
(v,k)\in \operatorname{epi}(\rho)
\quad\Longleftrightarrow\quad
k\ge0,\quad v(\bm{d})\le k \mathfrak{s}(\bm{d}),\ \forall \bm{d}\in\mathfrak D.
\]
Equivalently, we have
\[
\text{epi}(\rho) = \bigcap_{\bm{d}\in\mathfrak D} \{(v,k): v(\bm{d})-k \mathfrak{s}(\bm{d})\le0\} \cap \{(v,k):k\ge0\}.
\]
For each fixed $\bm{d}$, the map $(v,k)\mapsto v(\bm{d})-k\mathfrak{s}(\bm{d})$ is continuous, so each constraint set is closed. Hence the epigraph is an intersection of closed sets and is closed. Therefore $\rho$ is lower semicontinuous.

Then we verify the five properties within the axioms sequentially.

\emph{Monotonicity.} Assume $v_1(\bm{d})\ge v_2(\bm{d})$ for all $\bm{d}\in\mathfrak{D}$. If $k$ satisfies $v_1(\bm{d})\le k\cdot\mathfrak{s}(\bm{d}),\ \forall \bm{d}\in\mathfrak{D}$, then automatically $v_2(\bm{d})\le k\cdot\mathfrak{s}(\bm{d}),\ \forall \bm{d}\in\mathfrak{D}$. Thus $K_1\subseteq K_2$, so $\inf K_1\ge \inf K_2$. Equivalently, via the supremum form, $\frac{v_1(\bm{d})}{\mathfrak{s}(\bm{d})}\ge\frac{v_2(\bm{d})}{\mathfrak{s}(\bm{d})}$ pointwise, hence their suprema satisfy the inequality; taking the max with 0 preserves order.

\emph{Positive Homogeneity.} First, $\rho(0)=0$. For $\lambda=0$, both sides are zero. For $\lambda>0$, $\rho(\lambda v)=\inf\{k\ge0:\lambda v(\bm{d})\le k\cdot \mathfrak{s}(\bm{d}), \forall \bm{d}\in\mathfrak{D}\}=\inf\{k\ge0: v(\bm{d})\le (k/\lambda)\mathfrak{s}(\bm{d})\}$. Mapping $k\rightarrow k/\lambda$ yields $\rho(\lambda v)=\lambda\rho(v)$.

\emph{Subadditivity.} If either $\rho(v_1)=+\infty$ or $\rho(v_2)=+\infty$, the inequality is immediate. Otherwise, if two instances are both feasible, use the supremum representation: for any $\bm{d}$ with $\mathfrak{s}(\bm{d})>0$,
\[
    \frac{v_1(\bm{d})+v_2(\bm{d})}{\mathfrak{s}(\bm{d})}\le \frac{v_1(\bm{d})}{\mathfrak{s}(\bm{d})}+\frac{v_2(\bm{d})}{\mathfrak{s}(\bm{d})}.
\]
Take supremum over $\bm{d}$ and then max with $0$ to obtain
\begin{equation*}
    \begin{aligned}
    \rho(v_1+v_2)&=\max\left\{0,\sup_{\bm{d}}\frac{v_1(\bm{d})+v_2(\bm{d})}{\mathfrak{s}(\bm{d})}\right\}\le \max\left\{0,\sup_{\bm{d}}\frac{v_1(\bm{d})}{\mathfrak{s}(\bm{d})}+\sup_{\bm{d}} \frac{v_2(\bm{d})}{\mathfrak{s}(\bm{d})}\right\}\\
    &\le\max\left\{0,\sup_{\bm{d}}\frac{v_1(\bm{d})}{\mathfrak{s}(\bm{d})}\right\}+\max\left\{0,\sup_{\bm{d}}\frac{v_2(\bm{d})}{\mathfrak{s}(\bm{d})}\right\}=\rho(v_1)+\rho(v_2).
    \end{aligned}
\end{equation*}

\emph{Pro-robustness.} If $v_{\bm{x}}(\bm{d})\le 0$ pointwise in $\mathfrak{D}$, then $k=0$ is feasible in $\mathfrak{D}$. Hence $\rho(v)=\inf K_v\le 0$. Since $\rho(v)\ge 0$ by construction, $\rho(v)=0$.

\emph{Anti-fragility.} If $v(\hat{\bm{d}})>0$, then using $\mathfrak{s}(\hat{\bm{d}})=\mathfrak{s}(\hat{\bm{d}},\hat{\bm{d}})=0$, for any finite $k\ge0$ we have $k\mathfrak{s}(\hat{\bm{d}})=0<v(\hat{\bm{d}}).$
Thus no finite $k$ can satisfy the defining inequality at $\bm{d}=\hat{\bm{d}}$, so $K_v=\varnothing$. It follows that $\rho(v)=+\infty$. \hfill\halmos

\end{proof}

\begin{proof}{Proof of Proposition~\ref{prop:crs_performance}.}
Fix the target $\tau$, condition on the data used to construct the predictor and score, and fix a measurable optimizer rule; suppress this conditioning below. By feasibility of the selected solution for \eqref{eq:ConfRS}, for every $\bm z$ and every $\bm d\in\mathfrak D$,
\begin{equation}
    a\bigl(\bm x^*(\bm z),\bm d\bigr)-\tau\le
    k^*(\bm z)s_{\hat f}(\bm z,\bm d)=
    R_\tau(\bm z,\bm d).
    \label{eq:crs_scaled_score_pathwise}
\end{equation}
Apply \eqref{eq:crs_scaled_score_pathwise} to the test pair
$(\bm Z_{\mathrm{test}},\bm D_{\mathrm{test}})$. We obtain, almost surely,
\[
    a\bigl(\bm x^*(\bm Z_{\mathrm{test}}),\bm D_{\mathrm{test}}\bigr)-\tau
    \le
    R_\tau(\bm Z_{\mathrm{test}},\bm D_{\mathrm{test}}).
\]
Therefore, for any violation margin $\Delta>0$,
\[
\left\{
a\bigl(\bm x^*(\bm Z_{\mathrm{test}}),\bm D_{\mathrm{test}}\bigr)>\tau+\Delta\right\}\subseteq
\left\{R_\tau(\bm Z_{\mathrm{test}},\bm D_{\mathrm{test}})>\Delta\right\},
\]
which implies
\begin{equation}
\mathbb P\left\{a\bigl(\bm x^*(\bm Z_{\mathrm{test}}),\bm D_{\mathrm{test}}\bigr)>\tau+\Delta\right\}
\le
\mathbb P\left\{R_\tau(\bm Z_{\mathrm{test}},\bm D_{\mathrm{test}})>\Delta\right\}.
\label{eq:crs_event_transfer}
\end{equation}
Let $F^R$ denote the population CDF of the fragility-scaled score, $F^R(t):=\mathbb P\left\{R_\tau(\bm Z,\bm D)\le t\right\}.$
Because the test pair follows the same distribution as
$(\bm Z,\bm D)$ under Assumption~\ref{asp:iid_stability},
\begin{equation}
\mathbb P\left\{
a\bigl(\bm x^*(\bm Z_{\mathrm{test}}),\bm D_{\mathrm{test}}\bigr)>\tau+\Delta
\right\}
\le 1-F^R(\Delta).
\label{eq:crs_population_tail}
\end{equation}

We next estimate $F^R$ using the calibration sample. Under this conditioning, the mapping
\[
    (\bm z,\bm d)\mapsto R_\tau(\bm z,\bm d) = k^*(\bm z)s_{\hat f}(\bm z,\bm d)
\]
is fixed and measurable. Under Assumption~\ref{asp:iid_stability}, the calibration pairs $\{(\bm Z_i,\bm D_i)\}_{i=1}^n$ remain i.i.d. and have the same distribution as the test pair.
Hence, $R_i=R_\tau(\bm Z_i,\bm D_i),\ i=1,\ldots,n$ are i.i.d. observations with CDF $F^R$, and $\widehat F_n^R(t)=\frac{1}{n}\sum_{i=1}^n\mathbf 1\{R_i\le t\}$ is their empirical CDF. Then set $\varepsilon_n:=\sqrt{\frac{\ln(1/\epsilon)}{2n}}.$ The one-sided Dvoretzky--Kiefer--Wolfowitz inequality gives
\[
    \mathbb P_{\mathcal D_{\mathrm{cal}}}
    \left(\sup_{t\in\mathbb R}\left\{\widehat F_n^R(t)-F^R(t)\right\}>\varepsilon_n\right)
    \le\exp(-2n\varepsilon_n^2)=\epsilon.
\]
Therefore, with probability at least $1-\epsilon$ over the calibration
sample,
\[
    F^R(t)\ge\widehat F_n^R(t)-\varepsilon_n,\qquad \forall t\in\mathbb R.
\]
Evaluating this uniform inequality at the fixed threshold $t=\Delta$
gives
\[
    1-F^R(\Delta)\le1-\widehat F_n^R(\Delta)+\sqrt{\frac{\ln(1/\epsilon)}{2n}}.
\]
Finally, combining this inequality with
\eqref{eq:crs_population_tail}, we conclude that, with probability at
least $1-\epsilon$ over the calibration sample,
\[
    \mathbb P\left\{a\bigl(\bm x^*(\bm Z_{\mathrm{test}}),\bm D_{\mathrm{test}}\bigr)>\tau+\Delta\right\}
    \le
    1-\widehat F_n^R(\Delta)+\sqrt{\frac{\ln(1/\epsilon)}{2n}}.
\]
\hfill\halmos
\end{proof}

\subsection{Proofs for Section~\ref{sec:ConfRO_ConfRS_equivalence}}
\label{app:proof_equivalence}

\begin{proof}{Proof of Proposition~\ref{prop:sufficient_duality}.}
Fix $\bm x$ and $\theta$, and write
$
    v_\theta(\bm x)
    :=
    \sup\{a(\bm x,\bm d):\bm d\in\mathfrak D,
    \mathfrak s(\bm d,\hat{\bm d})\le\theta\}.
$
Equivalently, $-v_\theta(\bm x)$ is the value of the convex minimization problem
\[
    \inf_{\bm d\in\mathfrak D}
    \{-a(\bm x,\bm d):
    \mathfrak s(\bm d,\hat{\bm d})-\theta\le0\}.
\]
The assumed concavity and upper semicontinuity of $a(\bm x,\cdot)$ make $-a(\bm x,\cdot)$ proper, lower semicontinuous, and convex on $\mathfrak D$, while the score constraint is convex. The relative-interior point $\bar{\bm d}$ with $\mathfrak s(\bar{\bm d},\hat{\bm d})<\theta$ is Slater's condition for this convex problem. Hence strong Lagrangian duality holds, and the dual optimum is attained at some $k^*\ge0$. In particular,
\[
    -v_\theta(\bm x)
    =
    \sup_{k\ge0}
    \inf_{\bm d\in\mathfrak D}
    \{-a(\bm x,\bm d)+k(\mathfrak s(\bm d,\hat{\bm d})-\theta)\}.
\]
Multiplying by $-1$ and using
$-\inf_{\bm d}q(\bm d)=\sup_{\bm d}[-q(\bm d)]$ yields
\[
    v_\theta(\bm x)
    =
    \inf_{k\ge0}
    \left\{
    k\theta+
    \sup_{\bm d\in\mathfrak D}
    \bigl(a(\bm x,\bm d)-k\mathfrak s(\bm d,\hat{\bm d})\bigr)
    \right\},
\]
which is the desired score-duality identity, with the infimum attained at $k^*$. \hfill\halmos
\end{proof}

\begin{table}[htb]
\centering
\caption{Common score geometries satisfying strong duality for affine objective uncertainty.}
\label{tab:score_duality_conditions}
\small
\setlength{\tabcolsep}{4pt}
\begin{tabularx}{\textwidth}{p{0.20\textwidth}p{0.26\textwidth}p{0.24\textwidth}X}
\toprule
Score class & Set geometry & Duality mechanism & Used in the paper \\
\midrule
Weighted box / $L_\infty$ & Polyhedral interval set & Linear-programming duality under relative-interior feasibility & Score-design examples \\
Weighted budget / $L_1$ & Polyhedral budget set & Linear-programming duality under relative-interior feasibility & Fractional knapsack and online-grocery case study \\
Mahalanobis / ellipsoidal & Ellipsoid or second-order-cone set & Conic or convex-quadratic strong duality under Slater feasibility & Fractional knapsack and online-grocery case study \\
Asymmetric weighted budget & Polyhedral set with direction-specific slopes & Linear-programming duality under relative-interior feasibility & Score-design examples \\
\bottomrule
\end{tabularx}
\end{table}

To establish the value representation, define
$
    B(\bm x,k)
    :=\sup_{\bm d\in\mathfrak D}
    \{a(\bm x,\bm d)-k\mathfrak s(\bm d,\hat{\bm d})\}.
$
Assumption~\ref{ass:cro_crs_equiv}\ref{ass:equiv_dual} gives, for every $\bm x\in\mathcal X$,
$
    \sup_{\bm d\in\mathfrak D(\theta)}a(\bm x,\bm d)
    =\inf_{k\ge0}\{k\theta+B(\bm x,k)\}.
$
Taking the outer infimum and combining the two nested infima therefore yields
\[
    V(\theta)=\inf_{\bm x\in\mathcal X}\inf_{k\ge0}\{k\theta+B(\bm x,k)\}=\inf_{\substack{\bm x\in\mathcal X,\ k\ge0,\ \tau\in\mathbb R\\ B(\bm x,k)\le\tau}}\{\tau+k\theta\}.
\]
The second equality is the epigraph representation of $B(\bm x,k)$ and does not require the infimum over $k$ to be attained. Written as an optimization problem, it is exactly \eqref{eq:ConfRO-D}.

For a fixed target $\tau$, the infimum of $k$ over the pairs $(\bm x,k)$ satisfying $B(\bm x,k)\le\tau$ is precisely the extended value $\kappa^*(\tau)$ of \ConfRS{}$(\tau)$. Hence, for $\theta>0$, taking the infimum first over $(\bm x,k)$ and then over $\tau$ gives
\[
    V(\theta)
    =\inf_{\tau\in\mathbb R}\{\tau+\theta\kappa^*(\tau)\}
    =\inf_{\tau\ge\underline\tau}\{\tau+\theta\kappa^*(\tau)\}.
\]
The last equality follows because $\tau<\underline\tau$ makes \ConfRS{}$(\tau)$ infeasible at $\bm d=\hat{\bm d}$, so $\kappa^*(\tau)=+\infty$. This establishes the scalar value representation \eqref{eq:cro_e} without assuming that either infimum is attained.

\begin{proof}{Proof of Proposition~\ref{prop:cro_to_crs}.}
    Let $(\bm{x}^*,k^*,\tau^*)$ be an optimal solution of \eqref{eq:ConfRO-D}. Its feasibility in \eqref{eq:ConfRO-D} gives $\bm{x}^*\in\mathcal X$, $k^*\ge0$, and
    \[
        \sup_{\bm d\in\mathfrak D}
        \bigl\{a(\bm{x}^*,\bm d)-k^*\mathfrak s(\bm d,\hat{\bm d})\bigr\}
        \le \tau^*.
    \]
    The last inequality is equivalent to
    $a(\bm{x}^*,\bm d)-\tau^*\le
    k^*\mathfrak s(\bm d,\hat{\bm d})$ for every
    $\bm d\in\mathfrak D$. Hence, by the constraints in \eqref{eq:ConfRS},
    $(\bm{x}^*,k^*)$ is feasible for \ConfRS{}($\tau^*$).

    Now let $(\bm{x},k)$ be an arbitrary feasible pair for
    \ConfRS{}($\tau^*$). By \eqref{eq:ConfRS},
    $\bm{x}\in\mathcal X$, $k\ge0$, and
    \[
        \sup_{\bm d\in\mathfrak D}
        \bigl\{a(\bm{x},\bm d)-k\mathfrak s(\bm d,\hat{\bm d})\bigr\}
        \le \tau^*.
    \]
    Therefore, $(\bm{x},k,\tau^*)$ is feasible for \eqref{eq:ConfRO-D}.
    The optimality of $(\bm{x}^*,k^*,\tau^*)$ in \eqref{eq:ConfRO-D}
    then implies
    \[
        \tau^*+\theta k^*\le \tau^*+\theta k.
    \]
    Since $\theta>0$ by the hypothesis of Proposition~\ref{prop:cro_to_crs},
    it follows that $k^*\le k$. Because $(\bm{x},k)$ was an arbitrary
    feasible pair for \ConfRS{}($\tau^*$), we have
    $k^*=\kappa^*(\tau^*)$, and $(\bm{x}^*,k^*)$ is optimal for
    \ConfRS{}($\tau^*$).
\hfill\halmos\end{proof}

\begin{proof}{Proof of Theorem~\ref{thm:crs_to_cro}.}
    We first establish the convexity property used by the scalar representation in \eqref{eq:cro_e}.

    \begin{lemma}
    If $\mathcal X$ is convex and $a(\cdot,\bm d)$ is convex for every $\bm d\in\mathfrak D$, then $\kappa^*$ is convex.
    \label{lemma:kappa_convex}
    \end{lemma}
    \begin{proof}{Proof of Lemma~\ref{lemma:kappa_convex}.}
        Fix $\tau_1,\tau_2$ at which $\kappa^*$ is finite and let $\varepsilon>0$. Choose feasible pairs $(\bm{x}_1,k_1)$ and $(\bm{x}_2,k_2)$ for $\ConfRS{}(\tau_1)$ and $\ConfRS{}(\tau_2)$, respectively, such that $k_i\le\kappa^*(\tau_i)+\varepsilon$ for $i\in\{1,2\}$. For any $\lambda\in[0,1]$, the convex combination $(\lambda \bm{x}_1+(1-\lambda)\bm{x}_2,\lambda k_1+(1-\lambda)k_2)$ is feasible for \ConfRS{}($\lambda\tau_1+(1-\lambda)\tau_2$). Indeed,
        \begin{align*}
            &\sup_{\bm{d}\in\mathfrak{D}}\{a(\lambda \bm{x}_1+(1-\lambda)\bm{x}_2,\bm{d})-(\lambda k_1+(1-\lambda)k_2)\cdot \mathfrak{s}(\bm{d},\hat{\bm{d}})\}\\
            \le&\sup_{\bm{d}\in\mathfrak{D}}\{\lambda[a(\bm{x}_1,\bm{d})-k_1\cdot \mathfrak{s}(\bm{d},\hat{\bm{d}})]+(1-\lambda)[a(\bm{x}_2,\bm{d})-k_2\cdot \mathfrak{s}(\bm{d},\hat{\bm{d}})]\}\\
            \le&\lambda\sup_{\bm{d}\in\mathfrak{D}}\{[a(\bm{x}_1,\bm{d})-k_1\cdot \mathfrak{s}(\bm{d},\hat{\bm{d}})]\}+(1-\lambda) \sup_{\bm{d}\in\mathfrak{D}}\{[a(\bm{x}_2,\bm{d})-k_2\cdot \mathfrak{s}(\bm{d},\hat{\bm{d}})]\}\\
            \le&\lambda \tau_1+(1-\lambda)\tau_2.
        \end{align*}
        The first inequality follows from the convexity of $a(\bm{x}, \bm{d})$ in $\bm{x}$, the second from subadditivity of the supremum, and the third from the feasibility of $(\bm{x}_1,k_1)$ and $(\bm{x}_2,k_2)$. Since $\lambda \bm{x}_1+(1-\lambda)\bm{x}_2\in\mathcal{X}$ and $\lambda k_1+(1-\lambda)k_2\ge 0$, the convex combination is feasible. Therefore,
        \[
        \kappa^*(\lambda\tau_1+(1-\lambda)\tau_2)
        \le \lambda\kappa^*(\tau_1)+(1-\lambda)\kappa^*(\tau_2)+\varepsilon.
        \]
        Letting $\varepsilon\downarrow0$ proves that $\kappa^*$ is convex.
    \hfill\halmos\end{proof}

    Now fix the interior target $\tau_{\mathrm{rs}}$ and negative subgradient $\beta(\tau_{\mathrm{rs}})$ from the theorem, and set $\theta^*=-1/\beta(\tau_{\mathrm{rs}})>0$. For $g(\tau,\theta):=\tau+\theta\kappa^*(\tau)$, the subdifferential sum rule gives
    \[
        \partial_\tau g(\tau_{\mathrm{rs}},\theta^*)
        =1+\theta^*\partial\kappa^*(\tau_{\mathrm{rs}}).
    \]
    By construction, $0\in\partial_\tau g(\tau_{\mathrm{rs}},\theta^*)$, so $\tau_{\mathrm{rs}}$ minimizes $g(\cdot,\theta^*)$ in \eqref{eq:cro_e}. Let $(\bm x^*,k^*)$ be any optimal pair for \ConfRS{}$(\tau_{\mathrm{rs}})$. Then $(\bm x^*,k^*,\tau_{\mathrm{rs}})$ is feasible for \eqref{eq:ConfRO-D} and has objective
    \[
        \tau_{\mathrm{rs}}+\theta^*k^*
        =g(\tau_{\mathrm{rs}},\theta^*)
        =V(\theta^*).
    \]
    Thus it is optimal for \eqref{eq:ConfRO-D}. Its feasibility and the score-duality identity imply
    \[
        \sup_{\bm d\in\mathfrak D(\theta^*)}a(\bm x^*,\bm d)
        \le
        \theta^*k^*+
        \sup_{\bm d\in\mathfrak D}
        \{a(\bm x^*,\bm d)-k^*\mathfrak s(\bm d,\hat{\bm d})\}
        \le V(\theta^*).
    \]
    Since $V(\theta^*)$ is the infimum of the left-hand side over $\bm x\in\mathcal X$, equality holds and $\bm x^*$ is optimal for \ConfRO{}$(\theta^*)$. This proves part~\ref{thm:crs_to_cro_correspondence_item}.

    For part~\ref{thm:crs_to_cro_equivalence_item}, suppose additionally that $\tau_{\mathrm{rs}}$ is the unique minimizer of $g(\cdot,\theta^*)$ and that the score-dual infimum is attained for every \ConfRO{}$(\theta^*)$ optimizer. Take any such optimizer $\bar{\bm x}$ and an attaining multiplier $\bar k$. Set
    \[
        \tau'
        :=\sup_{\bm d\in\mathfrak D}
        \{a(\bar{\bm x},\bm d)-\bar k\mathfrak s(\bm d,\hat{\bm d})\}.
    \]
    Dual attainment gives $V(\theta^*)=\tau'+\theta^*\bar k$, so $(\bar{\bm x},\bar k,\tau')$ is optimal for \eqref{eq:ConfRO-D}. Moreover, $\tau'\ge a(\bar{\bm x},\hat{\bm d})\ge\underline\tau$, while $\bar k\ge0$ gives $\tau'\le V(\theta^*)$. Because $\mathfrak D(\theta^*)\subseteq\mathfrak D$, we also have $V(\theta^*)\le\bar\tau$. Hence $\tau'\in[\underline\tau,\bar\tau]$. Proposition~\ref{prop:cro_to_crs} implies that $(\bar{\bm x},\bar k)$ is optimal for \ConfRS{}$(\tau')$, so $g(\tau',\theta^*)=V(\theta^*)$. Thus $\tau'$ minimizes $g(\cdot,\theta^*)$, and uniqueness yields $\tau'=\tau_{\mathrm{rs}}$. Hence $\bar{\bm x}$ is also optimal for \ConfRS{}$(\tau_{\mathrm{rs}})$, proving the reverse inclusion.

    For part~\ref{thm:crs_to_cro_coverage_item}, if the score CDF $F$ is continuous, the calibrated coverage level is
    $\alpha^*=F(\theta^*)=F(-1/\beta(\tau_{\mathrm{rs}}))$.
    With atoms, the same conclusion uses the generalized-quantile convention and may correspond to an interval of coverage levels.
\hfill\halmos\end{proof}

\begin{proof}{Proof of Theorem~\ref{thm:target_value_marginal_cost}.}
We first record the scalar representation and the matched-target relation,
then prove the two marginal-cost claims.

First, we recall the scalar representation induced by
Assumption~\ref{ass:cro_crs_equiv}. Under the exact score-duality condition
in Assumption~\ref{ass:cro_crs_equiv}\ref{ass:equiv_dual}, the robust
problem \ConfRO{}$(\theta)$ is equivalent to
\[
\begin{aligned}
    \inf_{\bm x,k,\tau}\quad
    & \tau+\theta k \\
    \text{s.t.}\quad
    & \sup_{\bm d\in\mathfrak D}
    \left\{
        a(\bm x,\bm d)
        -
        k\mathfrak s(\bm d,\hat{\bm d})
    \right\}
    \le \tau,\\
    & \bm x\in\mathcal X,\quad k\ge0.
\end{aligned}
\tag{\ConfRO{}-D}
\]
For a fixed target $\tau$, the constraint in \ConfRO{}-D is exactly the
feasibility condition in \ConfRS{}$(\tau)$. Therefore, the infimum of $k$
over all pairs $(\bm x,k)$ feasible for this fixed $\tau$ is
$\kappa^*(\tau)$. Taking the nested infima gives
\[
    V(\theta)
    =
    \min_{\tau\in[\underline\tau,\bar\tau]}
    \left\{
        \tau+\theta\kappa^*(\tau)
    \right\}.
\]
The minimum over $\tau$ is attained because $\kappa^*$ is finite and
continuous on the compact interval $[\underline\tau,\bar\tau]$ by the setup
of Theorem~\ref{thm:target_value_marginal_cost}.

The matched subgradient condition identifies
$\tau_{\mathrm{rs}}$ as a selected target at radius
$\theta_{\mathrm{rs}}$. Since
$\beta_{\mathrm{rs}}\in\partial\kappa^*(\tau_{\mathrm{rs}})$, for every
$\tau\in[\underline\tau,\bar\tau]$,
\[
    \kappa^*(\tau)
    \ge
    \kappa^*(\tau_{\mathrm{rs}})
    +
    \beta_{\mathrm{rs}}(\tau-\tau_{\mathrm{rs}}).
\]
Multiplying by $\theta_{\mathrm{rs}}>0$ and adding $\tau$ to both sides
yields
\[
    \tau+\theta_{\mathrm{rs}}\kappa^*(\tau)
    \ge
    \tau
    +
    \theta_{\mathrm{rs}}\kappa^*(\tau_{\mathrm{rs}})
    +
    \theta_{\mathrm{rs}}\beta_{\mathrm{rs}}
    (\tau-\tau_{\mathrm{rs}})
    =
    \tau_{\mathrm{rs}}
    +
    \theta_{\mathrm{rs}}\kappa^*(\tau_{\mathrm{rs}})
    +
    (1+\theta_{\mathrm{rs}}\beta_{\mathrm{rs}})
    (\tau-\tau_{\mathrm{rs}}).
\]
By construction,
$\theta_{\mathrm{rs}}=-1/\beta_{\mathrm{rs}}$, so
$1+\theta_{\mathrm{rs}}\beta_{\mathrm{rs}}=0$. Hence
\[
    \tau+\theta_{\mathrm{rs}}\kappa^*(\tau)
    \ge
    \tau_{\mathrm{rs}}
    +
    \theta_{\mathrm{rs}}\kappa^*(\tau_{\mathrm{rs}})
    \qquad
    \forall \tau\in[\underline\tau,\bar\tau].
\]
Thus $\tau_{\mathrm{rs}}$ minimizes
$g(\cdot,\theta_{\mathrm{rs}})$ over $[\underline\tau,\bar\tau]$, i.e.,
$\tau_{\mathrm{rs}}\in T(\theta_{\mathrm{rs}})$, and
$
    V(\theta_{\mathrm{rs}})
    =
    \tau_{\mathrm{rs}}
    +
    \theta_{\mathrm{rs}}\kappa^*(\tau_{\mathrm{rs}}).
$

We now prove part \textup{(i)}. For notational simplicity, write
\[
    g(\tau,\theta):=\tau+\theta\kappa^*(\tau).
\]
By the setup of Theorem~\ref{thm:target_value_marginal_cost}, $\kappa^*$ is
finite and continuous on $[\underline\tau,\bar\tau]$. Hence
$g(\cdot,\theta)$ is continuous on $[\underline\tau,\bar\tau]$, and
$T(\theta)$ is nonempty and compact for every radius $\theta$ under
consideration.

Fix $\theta_{\mathrm{rs}}$ and define
$
    T_{\mathrm{rs}}
    :=
    T(\theta_{\mathrm{rs}}).
$
Let
$
    m_+
    :=
    \min_{\tau\in T_{\mathrm{rs}}}\kappa^*(\tau),
    \;
    m_-
    :=
    \max_{\tau\in T_{\mathrm{rs}}}\kappa^*(\tau).
$
Both extrema exist because $T_{\mathrm{rs}}$ is compact and $\kappa^*$ is
continuous.

We first compute the right derivative. For any $h>0$, choosing a target
$\tau^+\in T_{\mathrm{rs}}$ with
$\kappa^*(\tau^+)=m_+$ gives
\[
    V(\theta_{\mathrm{rs}}+h)
    \le
    g(\tau^+,\theta_{\mathrm{rs}}+h)
    =
    g(\tau^+,\theta_{\mathrm{rs}})
    +
    h\kappa^*(\tau^+)
    =
    V(\theta_{\mathrm{rs}})
    +
    h m_+.
\]
Therefore
\[
    \frac{
        V(\theta_{\mathrm{rs}}+h)
        -
        V(\theta_{\mathrm{rs}})
    }{h}
    \le
    m_+.
\]

For the reverse inequality, let
$\tau_h\in T(\theta_{\mathrm{rs}}+h)$. Since
$V(\theta_{\mathrm{rs}})\le g(\tau_h,\theta_{\mathrm{rs}})$, we have
\[
    V(\theta_{\mathrm{rs}}+h)-V(\theta_{\mathrm{rs}})
    =
    g(\tau_h,\theta_{\mathrm{rs}}+h)
    -
    V(\theta_{\mathrm{rs}})
    \ge
    g(\tau_h,\theta_{\mathrm{rs}}+h)
    -
    g(\tau_h,\theta_{\mathrm{rs}})
    =
    h\kappa^*(\tau_h).
\]
Thus
\[
    \frac{
        V(\theta_{\mathrm{rs}}+h)
        -
        V(\theta_{\mathrm{rs}})
    }{h}
    \ge
    \kappa^*(\tau_h).
\]
Take any sequence $h_j\downarrow0$. Because
$[\underline\tau,\bar\tau]$ is compact, the corresponding sequence
$\{\tau_{h_j}\}$ has a convergent subsequence; denote its limit by
$\bar\tau$. Along this subsequence,
\[
    g(\tau_{h_j},\theta_{\mathrm{rs}}+h_j)
    =
    V(\theta_{\mathrm{rs}}+h_j)
    \le
    V(\theta_{\mathrm{rs}})+h_j m_+,
\]
where the last inequality follows from the upper bound already proved.
Letting $j\to\infty$ and using continuity of $g$ gives
$
    g(\bar\tau,\theta_{\mathrm{rs}})
    \le
    V(\theta_{\mathrm{rs}}).
$
Since $V(\theta_{\mathrm{rs}})$ is the minimum value of
$g(\cdot,\theta_{\mathrm{rs}})$, equality must hold. Thus
$\bar\tau\in T_{\mathrm{rs}}$. By continuity of $\kappa^*$,
\[
    \lim_{j\to\infty}\kappa^*(\tau_{h_j})
    =
    \kappa^*(\bar\tau)
    \ge
    m_+.
\]
Hence the lower bound converges to at least $m_+$ along every vanishing
sequence of positive $h$. Combining this with the upper bound yields
\[
    \lim_{h\downarrow0}
    \frac{
        V(\theta_{\mathrm{rs}}+h)
        -
        V(\theta_{\mathrm{rs}})
    }{h}
    =
    m_+
    =
    \min_{\tau\in T(\theta_{\mathrm{rs}})}
    \kappa^*(\tau).
\]

The left derivative is analogous. For $h>0$, choose
$\tau^-\in T_{\mathrm{rs}}$ with $\kappa^*(\tau^-)=m_-$. Then
\[
    V(\theta_{\mathrm{rs}}-h)
    \le
    g(\tau^-,\theta_{\mathrm{rs}}-h)
    =
    g(\tau^-,\theta_{\mathrm{rs}})
    -
    h\kappa^*(\tau^-)
    =
    V(\theta_{\mathrm{rs}})
    -
    h m_-.
\]
Therefore
\[
    \frac{
        V(\theta_{\mathrm{rs}})
        -
        V(\theta_{\mathrm{rs}}-h)
    }{h}
    \ge
    m_-.
\]
Conversely, let $\tilde\tau_h\in T(\theta_{\mathrm{rs}}-h)$. Since
$V(\theta_{\mathrm{rs}})\le g(\tilde\tau_h,\theta_{\mathrm{rs}})$, we obtain
\[
    V(\theta_{\mathrm{rs}})
    \le
    g(\tilde\tau_h,\theta_{\mathrm{rs}})
    =
    g(\tilde\tau_h,\theta_{\mathrm{rs}}-h)
    +
    h\kappa^*(\tilde\tau_h)
    =
    V(\theta_{\mathrm{rs}}-h)
    +
    h\kappa^*(\tilde\tau_h).
\]
Thus
\[
    \frac{
        V(\theta_{\mathrm{rs}})
        -
        V(\theta_{\mathrm{rs}}-h)
    }{h}
    \le
    \kappa^*(\tilde\tau_h).
\]
Repeating the compactness argument above, every cluster point of
$\tilde\tau_h$ as $h\downarrow0$ belongs to $T_{\mathrm{rs}}$. Hence every
subsequential limit of $\kappa^*(\tilde\tau_h)$ is at most $m_-$. Therefore
\[
    \lim_{h\downarrow0}
    \frac{
        V(\theta_{\mathrm{rs}})
        -
        V(\theta_{\mathrm{rs}}-h)
    }{h}
    =
    m_-
    =
    \max_{\tau\in T(\theta_{\mathrm{rs}})}
    \kappa^*(\tau).
\]

If $V$ is differentiable at $\theta_{\mathrm{rs}}$, then the two one-sided
derivatives coincide, so $m_+=m_-$. Since the matched-target relation above
gives $\tau_{\mathrm{rs}}\in T_{\mathrm{rs}}$, this common value must equal
$\kappa^*(\tau_{\mathrm{rs}})$, and hence
$V'(\theta_{\mathrm{rs}})=\kappa^*(\tau_{\mathrm{rs}})$. In particular,
uniqueness of $\tau_{\mathrm{rs}}$ as the minimizer of
$g(\cdot,\theta_{\mathrm{rs}})$ implies this conclusion. This proves
part \textup{(i)}.

Finally, we prove part \textup{(ii)}. Define
$
    H(\alpha)
    :=
    V\bigl(F^{-1}(\alpha)\bigr).
$
At $\alpha_{\mathrm{rs}}=F(\theta_{\mathrm{rs}})$, we have
$F^{-1}(\alpha_{\mathrm{rs}})=\theta_{\mathrm{rs}}$ under the stated local
differentiability and positive-density assumptions. Since
$F'(\theta_{\mathrm{rs}})=f_S(\theta_{\mathrm{rs}})>0$, the inverse function
is differentiable at $\alpha_{\mathrm{rs}}$ and
\[
    \left.
    \frac{d}{d\alpha}F^{-1}(\alpha)
    \right|_{\alpha=\alpha_{\mathrm{rs}}}
    =
    \frac{1}{f_S(\theta_{\mathrm{rs}})}.
\]
Applying the chain rule and using part \textup{(i)} gives
\[
\begin{aligned}
    \left.
    \frac{d}{d\alpha}
    V\bigl(F^{-1}(\alpha)\bigr)
    \right|_{\alpha=\alpha_{\mathrm{rs}}}
    &=
    V'(\theta_{\mathrm{rs}})
    \left.
    \frac{d}{d\alpha}F^{-1}(\alpha)
    \right|_{\alpha=\alpha_{\mathrm{rs}}} \\
    &=
    \kappa^*(\tau_{\mathrm{rs}})
    \cdot
    \frac{1}{f_S(\theta_{\mathrm{rs}})} \\
    &=
    \frac{
        \kappa^*(\tau_{\mathrm{rs}})
    }{
        f_S(\theta_{\mathrm{rs}})
    },
\end{aligned}
\]
which completes the proof of part (ii).
\hfill\halmos
\end{proof}

\begin{proof}{Proof of Corollary~\ref{cor:mapping_compute}.}
    Let $\hat{F}_n$ and $F$ denote the empirical and true CDFs of the calibration scores, respectively. By the Dvoretzky-Kiefer-Wolfowitz (DKW) inequality, the uniform deviation between $\hat{F}_n$ and $F$ is bounded for any $\epsilon > 0$ as follows:
    \[
    \mathbb{P}\left(\sup_{s} |\hat{F}_n(s) - F(s)| > \epsilon\right) \le 2\exp(-2n\epsilon^2).
    \]
    To establish the bound for a specific confidence level $\delta \in (0,1)$, we equate the upper bound of the error probability to $\delta$:
    \[
    2\exp(-2n\epsilon^2) = \delta \implies \epsilon = \sqrt{\frac{\log(2/\delta)}{2n}}.
    \]
    Taking the complement of the probability event yields the high-probability uniform bound for the CDFs:
    \[
    \mathbb{P}\left(\sup_{s} |\hat{F}_n(s) - F(s)| \le \sqrt{\frac{\log(2/\delta)}{2n}}\right) \ge 1 - \delta.
    \]
    Because the empirical mapping $\hat{\alpha}_n(\tau)$ and the true mapping $\alpha(\tau)$ are functionally determined by $\hat{F}_n$ and $F$ respectively, the mapping deviation over the target range $\tau \in [\underline{\tau}, \bar{\tau}]$ is fundamentally bounded by the maximum deviation of the CDFs. That is, $\sup_{\tau\in[\underline{\tau},\bar{\tau}]} |\hat{\alpha}_n(\tau) - \alpha(\tau)| \le \sup_{s} |\hat{F}_n(s) - F(s)|$. Substituting this relationship into the inequality directly yields the final result
    \[
    \mathbb{P}\left(\sup_{\tau\in[\underline{\tau},\bar{\tau}]} |\hat{\alpha}_n(\tau) - \alpha(\tau)| \le \sqrt{\frac{\log(2/\delta)}{2n}}\right) \ge 1 - \delta.
    \]
    \hfill\halmos
\end{proof}

\begin{proof}{Proof of Proposition~\ref{prop:post_selection_validity}.}
Once the data and fitted quantities used before final calibration are fixed, the selected score $s_{\widehat{\ell}}$ is fixed. Hence the final calibration scores
$
    s_{\widehat{\ell}}(\bm z_i,\bm d_i),
    \qquad
    (\bm z_i,\bm d_i)\in\mathcal D_{\mathrm{cal}},
$
and the test score
$
    s_{\widehat{\ell}}(\bm Z_{\mathrm{test}},\bm D_{\mathrm{test}})
$
are exchangeable.

The split-conformal rank argument gives, for the remaining randomness,
\[
    \mathbb P\left\{
    s_{\widehat{\ell}}(\bm Z_{\mathrm{test}},\bm D_{\mathrm{test}})
    \le
    \widehat\eta_\alpha
    \right\}
    \ge
    \frac{\lceil (N_c+1) \alpha \rceil}{N_c +1}
    \ge
    \alpha.
\]
Averaging over the randomness in the pre-calibration quantities gives the same coverage bound unconditionally.
This event is precisely
$\{\bm D_{\mathrm{test}}\in
\widehat{\mathcal U}_\alpha(\bm Z_{\mathrm{test}})\}$.
\hfill\halmos
\end{proof}

\paragraph{Decision-transfer implications.}
The coverage event in Proposition~\ref{prop:post_selection_validity} implies that every robust constraint enforced over
$\widehat{\mathcal U}_\alpha(\bm Z_{\mathrm{test}})$
holds at $\bm D_{\mathrm{test}}$, and the robust objective upper-bounds the realized objective.

For any fixed target $\tau$, let
$(\widehat{\bm x}_\tau(\bm z),\widehat k_\tau(\bm z))$
be a feasible \ConfRS{} pair using $s_{\widehat\ell}$. Then
\[
    a(\widehat{\bm x}_\tau(\bm z),\bm d)-\tau
    \le
    \widehat k_\tau(\bm z)\,
    s_{\widehat{\ell}}(\bm z,\bm d),
    \qquad
    \forall \bm d\in\mathfrak D.
\]
On the same conformal coverage event,
$s_{\widehat{\ell}}(\bm Z_{\mathrm{test}},\bm D_{\mathrm{test}})
\le \widehat\eta_\alpha$, and therefore
\[
    a(\widehat{\bm x}_\tau(\bm Z_{\mathrm{test}}),\bm D_{\mathrm{test}})
    \le
    \tau+
    \widehat k_\tau(\bm Z_{\mathrm{test}})\widehat\eta_\alpha.
\]
Hence this \ConfRS{} bound also holds with probability at least $\alpha$.

\subsection{Localized Calibration with Conditional Coverage}
\label{sec:conditional_cro}

As discussed in Section~\ref{sec:cp_preliminary}, exact distribution-free pointwise conditional coverage is impossible for nontrivial procedures without additional structure \citep{foygel2021limits}. We therefore study kernel localization at a fixed interior context $\bm z_0\in\mathbb R^{d_z}$, where $d_z$ is the dimension of a meaningful metric representation, possibly a fixed pre-trained embedding, and the conditional score law varies smoothly near $\bm z_0$.

Throughout this subsection, we condition on the predictor and all other score components, fitted using data independent of both the calibration and test samples; all $\mathcal O_p$ statements are conditional on these components.

Let $\mathcal D_{\mathrm{cal}}=\{(\bm z_i,\bm d_i)\}_{i=1}^T$ and $S_i=s_{\hat f}(\bm z_i,\bm d_i)$. Define $K_h(\bm z,\bm z_0)=K((\bm z-\bm z_0)/h_T)$ and $W_T(\bm z_0)=\sum_{j=1}^T K_h(\bm z_j,\bm z_0)$. When $W_T(\bm z_0)>0$, set $w_i(\bm z_0)=K_h(\bm z_i,\bm z_0)/W_T(\bm z_0)$; otherwise, set $w_i(\bm z_0)=1/T$.

The localized CDF, quantile, and uncertainty set are, respectively, $\hat F_{S\mid\bm z_0}(t)=\sum_{i=1}^T w_i(\bm z_0)\indc\{S_i\le t\}$, $\hat\eta_\alpha(\bm z_0)=\inf\{t\ge0:\hat F_{S\mid\bm z_0}(t)\ge\alpha\}$, and $\mathcal U_\alpha(\bm z_0)=\{\bm d\in\mathfrak D:s_{\hat f}(\bm z_0,\bm d)\le\hat\eta_\alpha(\bm z_0)\}$. For $S=s_{\hat f}(\bm Z,\bm D)$, let $F_{S\mid\bm z}(t)$ be a regular conditional-CDF version of $\mathbb P(S\le t\mid\bm Z=\bm z)$ satisfying the smoothness condition below; conditioning on the fitted score components is suppressed in the notation. Finally, let $\eta_\alpha^*(\bm z_0)=\inf\{t\ge0:F_{S\mid\bm z_0}(t)\ge\alpha\}$.

Importantly, localization changes only the calibration step and leaves the pre-trained predictor and score unchanged, preserving the modularity of the main framework.

\begin{lemma}
\label{lem:weighted_threshold_process}
Let $S_1,\ldots,S_T$ be independent real-valued random variables with CDFs $F_i$, and let $a_1,\ldots,a_T$ be deterministic weights. Then there is a universal constant $C<\infty$ such that
\[
    \mathbb E\left[
    \sup_{t\in\mathbb R}\left|\sum_{i=1}^T a_i\{\indc\{S_i\le t\}-F_i(t)\}\right|
    \right]
    \le C\sqrt{\log(T+1)\sum_{i=1}^T a_i^2}.
\]
For a sigma-field $\mathcal G$, the same inequality holds almost surely with conditional expectation given $\mathcal G$ on the left if the $a_i$ are $\mathcal G$-measurable and the $S_i$ are conditionally independent given $\mathcal G$, with $F_i(t)=\mathbb P(S_i\le t\mid\mathcal G)$.
\end{lemma}

\begin{proof}{Proof of Lemma~\ref{lem:weighted_threshold_process}.}
By symmetrization, the left-hand side is at most $2\mathbb E\sup_t|\sum_{i=1}^T a_i\varepsilon_i\indc\{S_i\le t\}|$, where the $\varepsilon_i$ are independent Rademacher variables.

Conditional on $S_1,\ldots,S_T$, thresholds induce at most $T+1$ binary vectors, so the finite-class sub-Gaussian maximal inequality gives an upper bound of $2\sqrt{2\log(2(T+1))\sum_{i=1}^T a_i^2}$; see, e.g., \citet{wainwright2019high}. Absorbing constants gives the claim, and the same argument applies conditionally. \hfill\halmos
\end{proof}

\begin{lemma}
\label{lem:conditional_coverage}
Fix an interior context $\bm z_0$. Suppose Assumption~\ref{asp:iid_stability} holds and $F_{S\mid\bm z}(t)$ is H\"older continuous in $\bm z$ uniformly over $t$, i.e., for some $L<\infty$ and $s>0$ and all $\bm z,\bm z'$ in a neighborhood of $\bm z_0$,
$
    \sup_{t\ge 0}\left|F_{S\mid \bm z}(t)-F_{S\mid \bm z'}(t)\right|
    \le L\|\bm z-\bm z'\|_2^s.
$
Assume also that the density of $\bm Z$ is bounded above and away from zero near $\bm z_0$. Let $K:\mathbb R^{d_z}\to[0,\infty)$ be bounded and compactly supported, with $K(\bm u)\ge c_K>0$ whenever $\|\bm u\|_2\le r_K$ for some $r_K>0$. If $h_T\to0$ and $T h_T^{d_z}/\log T\to\infty$, then
\[
    \sup_{t\ge0}\left|\hat F_{S\mid \bm z_0}(t)-F_{S\mid \bm z_0}(t)\right|
    =\mathcal O_p\!\left(h_T^s+\sqrt{\frac{\log T}{T h_T^{d_z}}}\right).
\]
\end{lemma}

\begin{proof}{Proof of Lemma~\ref{lem:conditional_coverage}.}
Write $F_0(t)=F_{S\mid\bm z_0}(t)$, $F_i(t)=F_{S\mid\bm Z_i}(t)$, $K_i=K_h(\bm Z_i,\bm z_0)$, and $W_T=\sum_{i=1}^T K_i$. The local density and kernel assumptions give $\mathbb E[K_i]\asymp h_T^{d_z}$ and $\mathbb E[K_i^2]=\mathcal O(h_T^{d_z})$. Bernstein's and Markov's inequalities therefore yield $W_T\asymp_p T h_T^{d_z}$ and $\sum_{i=1}^T K_i^2=\mathcal O_p(T h_T^{d_z})$. In particular, $\mathbb P(W_T=0)\to0$, and on $\{W_T>0\}$,
\begin{equation}
\label{eq:effective_weights_revised}
    \sum_{i=1}^T w_i^2
    =\frac{\sum_{i=1}^T K_i^2}{W_T^2}
    =\mathcal O_p\!\left(\frac{1}{T h_T^{d_z}}\right).
\end{equation}

On this event, decompose the CDF error as
\[
\sup_{t\ge0}|\hat F_{S\mid \bm z_0}(t)-F_0(t)|
\le
\underbrace{\sup_{t\ge0}\left|\sum_{i=1}^T w_i\{\indc\{S_i\le t\}-F_i(t)\}\right|}_{R_T}
+
\underbrace{\sup_{t\ge0}\left|\sum_{i=1}^T w_i\{F_i(t)-F_0(t)\}\right|}_{B_T}.
\]

If $\operatorname{supp}(K)\subseteq\{\bm u:\|\bm u\|_2\le R\}$, then $K_i>0$ implies $\|\bm Z_i-\bm z_0\|_2\le Rh_T$. Nonnegativity of the weights and the H\"older condition thus give $B_T\le LR^s h_T^s$. Conditional on $\bm Z_1,\ldots,\bm Z_T$, the weights are fixed and the scores are independent with CDFs $F_1,\ldots,F_T$. Lemma~\ref{lem:weighted_threshold_process}, \eqref{eq:effective_weights_revised}, and conditional Markov's inequality give $R_T=\mathcal O_p(\sqrt{\log T/(T h_T^{d_z})})$.

Combining these bounds proves the result on $\{W_T>0\}$. Because the fallback empirical CDF is bounded and $\mathbb P(W_T=0)\to0$, it does not affect the claimed rate. \hfill\halmos
\end{proof}

\begin{proposition}
\label{prop:cond_cov_approx}
Under the conditions of Lemma~\ref{lem:conditional_coverage}, suppose $\eta_\alpha^*(\bm z_0)\in(0,\infty)$ is an interior regular quantile and, for some $0<r<\eta_\alpha^*(\bm z_0)$, $F_{S\mid\bm z_0}$ is absolutely continuous on $(\eta_\alpha^*(\bm z_0)-r,\eta_\alpha^*(\bm z_0)+r)$ with density bounded below by a positive constant. For an independent test point, define the context-and-calibration-conditional coverage
\[
    p_T(\bm z_0)=\mathbb P\!\left(\bm D_{\mathrm{test}}\in\mathcal U_\alpha(\bm z_0)
    \mid \bm Z_{\mathrm{test}}=\bm z_0,\mathcal D_{\mathrm{cal}},s_{\hat f}\right).
\]
Then, with $\epsilon_T=h_T^s+\sqrt{\log T/(T h_T^{d_z})}$, $|\hat\eta_\alpha(\bm z_0)-\eta_\alpha^*(\bm z_0)|=\mathcal O_p(\epsilon_T)$ and $|p_T(\bm z_0)-\alpha|=\mathcal O_p(\epsilon_T)$.
\end{proposition}

\begin{proof}{Proof of Proposition~\ref{prop:cond_cov_approx}.}
Let $F_0=F_{S\mid\bm z_0}$, $\hat F=\hat F_{S\mid\bm z_0}$, $\eta^*=\eta_\alpha^*(\bm z_0)$, $\hat\eta=\hat\eta_\alpha(\bm z_0)$, and $\delta_T=\sup_{t\ge0}|\hat F(t)-F_0(t)|$. Lemma~\ref{lem:conditional_coverage} gives $\delta_T=\mathcal O_p(\epsilon_T)$. By assumption, there are $r>0$ and $\underline f>0$ such that $F_0$ is absolutely continuous with density at least $\underline f$ on $(\eta^*-r,\eta^*+r)$; hence $F_0(\eta^*)=\alpha$.

Choose $A>1/\underline f$. On $\{0<A\delta_T<r\}$, the density lower bound and the definition of $\delta_T$ imply
\[
    \hat F(\eta^*-A\delta_T)<\alpha
    <\hat F(\eta^*+A\delta_T).
\]
Thus $|\hat\eta-\eta^*|\le A\delta_T=\mathcal O_p(\epsilon_T)$; the case $\delta_T=0$ is immediate.

Independence of the test point gives $p_T(\bm z_0)=F_0(\hat\eta)$. With probability tending to one, $\hat\eta$ lies in the neighborhood above, where $F_0$ is continuous. Since the weighted empirical CDF is right-continuous, $\hat F(\hat\eta)\ge\alpha$, so $F_0(\hat\eta)\ge\alpha-\delta_T$. Moreover, $\hat F(t)<\alpha$ for every $t<\hat\eta$; letting $t\uparrow\hat\eta$ gives $F_0(\hat\eta)\le\alpha+\delta_T$. Therefore $|p_T(\bm z_0)-\alpha|\le\delta_T=\mathcal O_p(\epsilon_T)$. \hfill\halmos
\end{proof}

Choosing $h_T\asymp(\log T/T)^{1/(2s+d_z)}$ yields the rate $\mathcal O_p((\log T/T)^{s/(2s+d_z)})$. The two terms in Proposition~\ref{prop:cond_cov_approx} make the price of localization explicit: $h_T^s$ is the bias from averaging across nearby contexts, whereas $\sqrt{\log T/(T h_T^{d_z})}$ is the stochastic error governed by the effective number of nearby calibration observations. A smaller bandwidth therefore adapts more closely to local score behavior but becomes less stable when few observations receive appreciable weight.

\section{Supplementary Discussion}
\label{app:supplementary_discussion}

Appendix~\ref{app:supplementary_discussion} provides a structured comparison with closely related prediction-driven robustness frameworks and discusses opportunities and challenges in extending conformal calibration to the distributional setting.

\subsection{Comparison with Closely Related Prediction-Driven Robustness Frameworks}
\label{subsec:related_work}

In addition to the literature review in Section~\ref{sec:literature} of the main text, we compare in Table~\ref{tab:closest_positioning} several closely related frameworks in terms of predictor choices, uncertainty modeling, and decision guarantees to provide a more straightforward comparison and better position this paper's contributions.

\begin{table}[htbp]
\centering
\caption{Positioning relative to closely related robustness frameworks from a prediction-driven perspective.}
\label{tab:closest_positioning}
\scriptsize
\setlength{\tabcolsep}{1.3pt}
\renewcommand{\arraystretch}{1.0}
\begin{tabular}{@{}>{\raggedright\arraybackslash}m{0.25\textwidth}>{\centering\arraybackslash}m{0.105\textwidth}>{\centering\arraybackslash}m{0.105\textwidth}>{\centering\arraybackslash}m{0.105\textwidth}>{\centering\arraybackslash}m{0.105\textwidth}>{\raggedright\arraybackslash}m{0.14\textwidth}>{\raggedright\arraybackslash}m{0.155\textwidth}@{}}
\toprule
Paper & \makecell{Contextual\\predictor} & \makecell{Parameter-\\uncertainty\\robust\\optimization} & \makecell{Parameter-\\uncertainty\\robust\\satisficing} & \makecell{Decision\\equivalence} & \makecell[l]{Robustness\\object} & \makecell[l]{Robustness \\ interpretation} \\
\midrule
\parbox[c]{\linewidth}{\raggedright Conformal robust optimization and risk-sensitive linear programs\\\citep{sun2023predict,patel2024conformal,chenreddy2024end,cai2025out}} & Arbitrary & Yes & No & No & Calibrated uncertainty set & Score not used as a unifying decision primitive \\
\addlinespace[5pt]

\parbox[c]{\linewidth}{\raggedright Joint estimation and robustness optimization \\\citep{zhu2022joint}} & Structured$^\dagger$ & Yes & No$^\ddagger$ & No & Estimation error in input parameters & Tailored to specific estimation procedures \\
\addlinespace[8pt]

\parbox[c]{\linewidth}{\raggedright Robust satisficing and equivalence\\\citep{long2023robust,wang2025equivalence}} & Not applied & No$^*$ & No$^*$ & Yes$^*$ & Distributional ambiguity & Not prediction-centered \\

\addlinespace[8pt]
\parbox[c]{\linewidth}{\raggedright
Prediction-based robust satisficing and fortification\\
\citep{sim2024analytics}}
&
Structured$^\dagger$
&
No
&
Partial
&
No$^*$
&
{Residual distribution and prediction coefficients}
&
{Prediction-centered distributionally robust satisficing (DRS) with estimation fortification}
\\
\addlinespace[8pt]

\textbf{This paper} & Arbitrary & Yes & Yes & Yes & Calibrated prediction-error score and induced fragility & Decision-level equivalence; interplay of score radius, reliability, and target \\
\bottomrule
\end{tabular}
\par\smallskip
\parbox{\textwidth}{\scriptsize \textit{Note.} $^*$: distributional ambiguity differs from parameter uncertainty, see Sections~\ref{subsubsec:crs_vs_drs} and~\ref{sec:ConfRO_ConfRS_equivalence}.
$^\dagger$: structured estimation method refers to, for example, regression, least absolute shrinkage and selection operator, and maximum likelihood estimation.
$^\ddagger$:  achieving a target is considered though the formulation is relatively more akin to robust optimization.}
\end{table}

\subsection{Distributional Extensions: Opportunities and Challenges}
\label{app:cdro_note}

While our framework focuses on calibrating uncertainty for realized parameter values relative to a point prediction, a natural question is whether this approach can be extended to conformal distributionally robust optimization (DRO). In such a setting, the primitive object of interest shifts from a single realization $\bm D$ to the entire conditional distribution. However, this may introduce significant theoretical and practical hurdles.

Suppose $P_{\bm z}:=\mathcal L(\bm D\mid \bm Z=\bm z)$ is the true conditional distribution, and $\widehat P_{\bm z}\in\mathcal P(\mathfrak D)$ is a distribution-valued predictor. Conceptually, one might want to construct a calibrated ambiguity set $\mathcal P_\alpha(\bm z)$ around $\widehat P_{\bm z}$ to solve
\[
    \min_{\bm x\in\mathcal X} \sup_{Q\in\mathcal P_\alpha(\bm z)} \mathbb E_Q[a(\bm x,\bm D)].
\]
This direction is conceptually attractive because it would combine black-box conditional distribution estimation with the ambiguity-set machinery of DRO.

If we had access to an oracle that could evaluate a distributional discrepancy score $\mathfrak S(P_{\bm Z_i},\widehat P_{\bm Z_i})$ between the true and estimated distributions, conformalizing this process would be straightforward. By computing calibration scores $S_i^{\mathrm{dist}} := \mathfrak S(P_{\bm Z_i},\widehat P_{\bm Z_i})$ for $i=1,\ldots,n$, the standard split-conformal quantile $\eta_\alpha^{\mathrm{dist}}$ would yield the ambiguity set
\[
    \mathcal P_\alpha(\bm z):=
    \left\{
        Q\in\mathcal P(\mathfrak D):
        \mathfrak S(Q,\widehat P_{\bm z})\le \eta_\alpha^{\mathrm{dist}}
    \right\}.
\]
Assuming exchangeability, this provides a rigorous distributional coverage guarantee:
\[
    \mathbb P\!\left\{
        P_{\bm Z_{\mathrm{test}}}\in\mathcal P_\alpha(\bm Z_{\mathrm{test}})
    \right\} \ge \alpha.
\]

However, a main issue is that the usual datasets available for estimation do not give us $P_{\bm Z_i}$. We only see one sample $\bm D_i\sim P_{\bm Z_i}$ per context, meaning the oracle score cannot be evaluated. One workaround is to use a realized score $r(\bm z,\bm d;\widehat P)$, like the negative log-likelihood $r(\bm z,\bm d;\widehat P)=-\log \widehat p_{\bm z}(\bm d)$. But calibrating a realized score fundamentally changes the type of guarantee we get. The resulting prediction region
$C_\alpha(\bm z):=\{\bm d\in\mathfrak D:
r(\bm z,\bm d;\widehat P)\le \eta_\alpha\}$
satisfies $\mathbb P\{\bm D_{\mathrm{test}}\in C_\alpha(\bm Z_{\mathrm{test}})\}\ge\alpha$. This only ensures we cover the realized parameter, not the true distribution $P_{\bm z}$.

We could still try to build DRO-style ambiguity sets from these realized regions, such as probability-mass bounds $\mathcal Q_{\alpha,\gamma}(\bm z):= \{Q\in\mathcal P(\mathfrak D):Q(C_\alpha(\bm z))\ge \gamma\}$ or score-budget sets $\mathcal Q_\eta^r(\bm z):=\{Q\in\mathcal P(\mathfrak D):\mathbb E_Q[r(\bm z,\bm D;\widehat P)]\le \eta\}$.
But unless we have repeated observations per context or make strong structural assumptions, we have no distribution-free guarantee that $P_{\bm z}$ actually belongs to these sets.

There is also a modeling issue. If the ambiguity set is only required to place all of its mass inside the conformal region, e.g., $\mathcal Q_{\alpha,1}(\bm z):= \{Q\in\mathcal P(\mathfrak D):Q(C_\alpha(\bm z))=1\}$, then the formulation becomes much simpler. Without additional moment or shape constraints, the worst-case expected cost reduces to
\[
\sup_{Q\in\mathcal Q_{\alpha,1}(\bm z)}\mathbb E_Q[a(\bm x,\bm D)]=
\sup_{\bm d\in C_\alpha(\bm z)} a(\bm x,\bm d).
\]
In this case, the model is essentially robust optimization over the conformal set $C_\alpha(\bm z)$, rather than a genuinely distributional formulation.

These observations suggest that conformal DRO is an attractive but more demanding extension of the current framework. Achieving direct distributional coverage requires richer information, such as repeated observations per context or structural assumptions that make the conditional law $P_{\bm z}$ estimable. Realized-score methods are more practical, but their distribution-free guarantees apply to realized parameter values rather than the true conditional distribution itself. This work therefore focuses on the parameter-value setting, where one observation per context is sufficient for valid split-conformal calibration.

\section{Details of Numerical Experiments}
\label{app:numerical_details}

Appendix~\ref{app:numerical_details} provides supplementary details for the numerical studies in Section~\ref{sec:experiments}. Appendix~\ref{app:knapsack_details} covers the robust fractional knapsack experiment, including data generation, predictors, benchmark calibration, additional performance comparisons under nominal and shifted evaluations, parameter-mapping computation, and reformulations. We document the online-grocery case study in Appendix~\ref{app:case_study} and Appendix~\ref{app:case_study_rs}, including formulations, the exchangeability justification, forecasting architectures and hyperparameter settings, additional out-of-sample cost comparisons, and the \ConfRS{} experiments.
Finally, Appendix~\ref{app:location} presents an additional robust facility-location study that examines coverage validity, the downstream value of predictive accuracy, and comparative performance.

\subsection{Robust Fractional Knapsack}
\label{app:knapsack_details}

\subsubsection{Experimental Setup.}
\noindent\paragraph{Data Generation.}
Following \cite{ho-nguyen2022risk}, we set $n=20$ and $d=15$. The coefficient matrix $\bm{\Theta}$ has entries drawn from $\mathrm{Binom}(1,0.5)\cdot\mathrm{U}(0.8,1.2)$, with the last two dimensions set to zero to represent irrelevant covariates. For each sample, we draw the contextual feature $\bm{z}\sim \mathrm{U}(0,4)^d\in\mathbb{R}^d$ and define the value of item $i$ as $c_i=((\bm{\Theta}\bm{z})_i)^2\xi_i$, where $\xi_i\sim \mathrm{U}(0.8,1.2)$ is a multiplicative disturbance. Item prices are sampled as $p_i\in\{100,101,\dots,1000\}$, and the budget as $B\in [p_{\max},\sum_i p_i-u p_{\max}]$, where $u\sim \mathrm{U}(0,1)$.

We generate 9,000 samples, split into training, uncertainty quantification, calibration, and test sets in a 3:3:2:1 ratio for \ConfRO{}, and into training, uncertainty quantification, and test sets in a 4:4:1 ratio for \ConfRS{}, as the latter requires no separate calibration set.

\noindent\paragraph{Benchmark Methods.}
For the Ellipsoid-RO baseline, we rely solely on historical observations to estimate the covariance matrix $\hat{\bm{\Sigma}}$ and use the sample mean $\bar{\bm{d}}$ as the centroid. The size parameter $\eta_\alpha$ is calibrated empirically to satisfy the target coverage level $\alpha$ on the calibration set, defining the uncertainty set as $\mathcal{U}_\alpha=\{\bm{d} \in \mathbb{R}^J :(\bm{d}-\bar{\bm{d}})^\top\hat{\bm{\Sigma}}^{-1}(\bm{d}-\bar{\bm{d}}) \le \eta_\alpha\}$.

For the clustering-based $k$-nearest-neighbor robust optimization (KNN-RO) and $k$-means robust optimization (KMeans-RO) baselines, we construct the uncertainty set for a context $\bm{z}$ by first assigning it to a local neighborhood or cluster based on feature similarity, and then forming an ellipsoidal set using the corresponding sample mean and covariance, calibrated to the prescribed coverage level. KNN-RO determines the neighborhood dynamically via nearest neighbors, whereas KMeans-RO uses precomputed clusters.

\noindent\paragraph{Predictor Specification and Residual Modeling.}
We use kernel ridge regression (KRR) with a polynomial kernel to predict the utility vector $\bm{c}$. The regularization parameter $\lambda$ and polynomial degree $d$ are selected by cross-validation based on validation mean squared error (MSE). 

For residual scaling, we further train separate residual-scale predictors using the pinball loss to estimate conditional $\alpha$-quantiles. Both models are trained for 300 epochs with a learning rate of $0.005$. The models use Gaussian error linear unit (GELU) and rectified linear unit (ReLU) activation functions, as detailed in Table~\ref{tab:residual_models}.

\begin{table}[htbp]
    \centering
    \caption{Residual-scale quantile models used for different uncertainty-set geometries in \ConfRO{}. }
    \label{tab:residual_models}
    \setlength{\tabcolsep}{4pt}
    \begin{tabularx}{\textwidth}{@{}lccc>{\raggedright\arraybackslash}X@{}}
        \toprule
        Score geometry & Hidden-layer widths & Activation & Output size & Predicted residual scale \\
        \midrule
        Box/Budget & $64,32$ & GELU & $20$ & $|c_i-\hat c_i|$ for each $i\in[n]$ \\
        Ellipsoid  & $64,16$ & ReLU & $1$ & $\|\bm c-\hat{\bm c}\|_2$ \\
        \bottomrule
    \end{tabularx}
\end{table}

\subsubsection{Further Performance Comparisons and Data Shift Analysis for Knapsack.}
\label{app:knapsack_relative_performance}

Table~\ref{tab:knapsack_relative_performance} compares the out-of-sample utilities of \ConfRO{}, KMeans-RO, KNN-RO, and the predict-then-optimize (PTO) benchmark. We report the relative improvement of the best-performing \ConfRO{} specification over each baseline as
\[
    \mathrm{Imp.}
    =
    \frac{\mathrm{Util}_{\ConfRO{}}-\mathrm{Util}_{\mathrm{baseline}}}
    {\mathrm{Util}_{\mathrm{baseline}}}
    \times 100\%.
\]
Across all coverage levels, \ConfRO{} consistently outperforms both local robust baselines while achieving utility close to PTO. The local robust baselines, however, become increasingly conservative at higher coverage levels: at targets of 90\% or above, fewer than 10\% of instances remain non-degenerate.

\begin{table}[htbp]
\centering
\caption{Out-of-sample performance comparison in the fractional knapsack experiment.}
\small
\label{tab:knapsack_relative_performance}
\setlength{\tabcolsep}{2.5pt}
\resizebox{\textwidth}{!}{
\begin{tabular}{cccccccccc}
\toprule
Coverage & \makecell{\ConfRO{}-\\Box} & \makecell{\ConfRO{}-\\Ellipsoid} & \makecell{\ConfRO{}-\\Budget} & KMeans-RO & KNN-RO & PTO & \makecell{Imp. vs\\KMeans-RO} & \makecell{Imp. vs\\KNN-RO} & \makecell{Imp. vs\\PTO} \\
\midrule
0.60 & 1307.7 & 1296.9 & 1309.4 & 898.9(29.40\%) & 1030.5(37.40\%) & 1311.2 & 45.66\% & 27.06\% & -0.14\% \\
0.70 & 1308.5 & 1296.0 & 1310.2 & 845.9(18.40\%) & 1021.0(25.20\%) & 1311.2 & 54.88\% & 28.32\% & -0.08\% \\
0.80 & 1309.5 & 1294.5 & 1310.4 & 904.9(7.60\%)  & 1063.8(13.60\%) & 1311.2 & 44.81\% & 23.18\% & -0.06\% \\
0.85 & 1309.5 & 1293.7 & 1310.6 & 843.2(7.60\%)  & 1044.0(10.20\%) & 1311.2 & 55.43\% & 25.54\% & -0.05\% \\
0.90 & 1308.7 & 1292.9 & 1310.8 & 918.3(4.80\%)  & 1051.1(5.40\%)  & 1311.2 & 42.74\% & 24.71\% & -0.03\% \\
0.95 & 1309.9 & 1291.2 & 1310.5 & 662.8(4.80\%)  & 1002.4(3.00\%)  & 1311.2 & 97.73\% & 30.74\% & -0.06\% \\
\bottomrule
\end{tabular}
}
\par\smallskip
\parbox{\textwidth}{\scriptsize \textit{Note.} Each KMeans-RO and KNN-RO entry reports realized utility, with the feasibility rate in parentheses.}
\end{table}

For the shifted evaluation, we perturb only the realized test utilities, while keeping predictions, prices, budgets, and decisions fixed; for each instance, items are ranked by predicted utility and only the top quartile is discounted. Their utility adjustment factor is
\[
\max\left\{0.30,\,
1 - 0.70\, s_{ij}\left(0.40 + 0.60 u_{ij}\right)
\right\},
\]
where $s_{ij}\in[0,1]$ denotes the normalized top-quartile rank score and $u_{ij}\in[0,1]$ the normalized residual-uncertainty rank, with $s_{ij}=0$ outside the top quartile. All utilities are additionally multiplied by independent noise drawn uniformly from $[0.985,1.015]$.
As shown in Table~\ref{tab:comparison_CRO_PTO_perturbation}, \ConfRO{} achieves higher realized utility than PTO under this perturbation.

\begin{table}[htbp]
\centering
\caption{Performance comparison under shifted evaluation with downside utility perturbations.}
\label{tab:comparison_CRO_PTO_perturbation}
\small
\begin{tabularx}{0.9\textwidth}{ *{6}{>{\centering\arraybackslash}X} }
\toprule
Coverage & \ConfRO{}-Box & \ConfRO{}-Ellipsoid & \ConfRO{}-Budget & PTO & Imp. \\
\midrule
0.60 & 1084.0 & \textbf{1097.0} & 1082.4 & 1082.4 & 1.35\% \\
0.70 & 1083.2 & \textbf{1096.7} & 1084.2 & 1082.4 & 1.32\% \\
0.80 & 1086.4 & \textbf{1096.2} & 1082.5 & 1082.4 & 1.28\% \\
0.85 & 1085.2 & \textbf{1096.2} & 1081.4 & 1082.4 & 1.28\% \\
0.90 & 1085.3 & \textbf{1096.1} & 1082.5 & 1082.4 & 1.27\% \\
0.95 & 1086.7 & \textbf{1095.4} & 1080.7 & 1082.4 & 1.20\% \\
\bottomrule
\end{tabularx}
\end{table}

\subsubsection{Parameter Mapping Computation.}

\paragraph{Reformulation of the \ConfRO{} Model.} Under the score function $s(\bm{c},\hat{\bm{c}})=\|(\bm{c}-\hat{\bm{c}})/\hat{\bm{r}}\|_1$, and explicitly considering the non-negative support set of item utilities $\mathcal{C}=\mathbb{R}_+^n$, the conformal uncertainty set is defined as $\mathcal{C}(\theta)=\{\bm{c}\in\mathbb{R}_+^n:\sum_{i=1}^n\frac{|c_i-\hat{c}_i|}{\hat{r}_i}\le\theta\}$. The reformulation of the \ConfRO{} model for the knapsack problem can be written as follows.
\begin{equation*}
\begin{aligned}
\min_{\bm{x}, k, \bm{\mu}} \quad & -\hat{\bm{c}}^\top \bm{x} + \theta k + \sum_{j=1}^n \frac{\hat{c}_j}{\hat{r}_j} \mu_j \\
\text{s.t.} \quad &\bm{p}^\top \bm{x} \le B, \\
&k + \mu_j \ge \hat{r}_j x_j, \quad \forall j \in [n], \\
&\bm{x}\in [0,1]^n, k \ge 0, \bm{\mu} \ge 0.
\end{aligned}
\tag{\ConfRO{}-KP} \label{eq:cro_kp_reformulation}
\end{equation*}

\paragraph{Computing the Parameter Pairs.}
The fractional knapsack problem can be verified to satisfy the conditions of Theorem~\ref{thm:crs_to_cro}. Since the parameter correspondence need not be unique, we construct it from the \ConfRO{} radius $\theta$ to the corresponding \ConfRS{} target $\tau$. The \ConfRO{} problem admits the equivalent representation
\begin{align*}
    \min_{\bm{x}\in\mathcal{X},k\ge 0,\tau} \big\{\tau+k\theta\ | \sup_{\bm{c}\in\mathcal{C}}\{f(\bm{x},\bm{c})-k\cdot s(\bm{c},\hat{\bm{c}})\}\le\tau\big\}.
\end{align*}

By Proposition~\ref{prop:cro_to_crs}, if $(\bm{x}^*,k^*,\tau^*)$ is optimal for this \ConfRO{} problem, then $\bm{x}^*$ also solves the \ConfRS{} problem with target $\tau^*$. For fixed $\bm{x}$, define
\[
\Phi_{\bm{x}}(k):
=\sup_{\bm{c}\in\mathbb R_+^n}\left\{f(\bm{x},\bm{c})-k\,s(\bm{c},\hat{\bm{c}})\right\}
=\sup_{\bm{c}\in\mathbb R_+^n}\left\{-\bm{c}^\top \bm{x}-k\left\|\frac{\bm{c}-\hat{\bm{c}}}{\hat{\bm{r}}}\right\|_1\right\}.
\]
Exploiting separability across items yields
\[
\Phi_{\bm{x}}(k)=-\bm{x}^\top\hat{\bm{c}}+\sum_{j=1}^n\frac{\hat c_j}{\hat r_j}\left(\hat r_jx_j-k\right)_+.
\]
Accordingly, in the reformulation \eqref{eq:cro_kp_reformulation}, $\mu_j^*=(\hat r_jx_j^*-k^*)_+$. Thus, after solving \eqref{eq:cro_kp_reformulation} for a given $\theta$, the corresponding \ConfRS{} target is

\[
    \tau^*=-(\bm{x}^*)^\top\hat{\bm{c}}+\sum_{j=1}^n\frac{\hat c_j}{\hat r_j}\mu_j^*.
\]
The mapping $\theta\mapsto\tau^*(\theta)$ may not be unique when the subgradient $\partial\kappa^*(\tau)$ is a set rather than a singleton. Consequently, the approach described above recovers only one possible mapping curve from the set of feasible solutions.

\subsubsection{Reformulations of \ConfRS{} and DRS.}

Under the score function
$s(\bm{c},\hat{\bm{c}})=\|\bm{c}-\hat{\bm{c}}\|_1$,
the two models admit the following reformulations:\par\noindent
\begin{minipage}[t]{0.47\textwidth}
\vspace{0pt}
\centering
\setlength{\abovedisplayskip}{0pt}
\setlength{\abovedisplayshortskip}{0pt}
\begin{equation*}\tag{\ConfRS{}-KP}
\begin{aligned}
    \min_{\bm{x},k}\quad & k\\
    \text{s.t.}\quad
    & \bm{p}^\top \bm{x} \le B,\\
    & x_i \le k,\quad \forall i\in[n],\\
    & \bm{x}^\top \hat{\bm{c}}+\tau \ge 0,\\
    & \bm{x}\in[0,1]^n,\quad k\ge0.
\end{aligned}
\end{equation*}
\end{minipage}\hfill\begin{minipage}[t]{0.47\textwidth}
\vspace{0pt}
\centering
\setlength{\abovedisplayskip}{0pt}
\setlength{\abovedisplayshortskip}{0pt}
\begin{equation*}\tag{DRS-E-KP}
\begin{aligned}
    \min_{\bm{x},k}\quad & k\\
    \text{s.t.}\quad
    & \bm{p}^\top \bm{x} \le B,\\
    & x_i \le k,\quad \forall i\in[n],\\
    & \frac{1}{S}\sum_{s=1}^S
      \bm{x}^\top\tilde{\bm{c}}_s+\tau \ge 0,\\
    & \bm{x}\in[0,1]^n,\quad k\ge0.
\end{aligned}
\end{equation*}
\end{minipage}
\par

The only structural difference is that \ConfRS{} uses the context-conditioned prediction
$\hat{\bm{c}}$, whereas DRS-E uses the empirical average over historical samples
$\{\tilde{\bm{c}}_s\}_{s=1}^S$.

\subsection{Real-Data Case Study on Inventory Management: \ConfRO{}}
\label{app:case_study}

We implement \eqref{prob:inv} using a joint seven-day demand trajectory and a static robust replenishment schedule committed at the initial planning epoch. Joint modeling captures dependence across periods, while the static schedule reflects the operational requirement that replenishment decisions be fixed in advance.
We consider Box, Ellipsoid, and Budget score structures, each in an adaptive variant using learned residual scaling and a static variant without residual scaling. The corresponding robust counterparts follow from standard robust optimization reformulations.

\subsubsection{Experimental Settings.}
\label{app:cs_setting}

\paragraph{Rich Contextual Information.}
Contextual features in the dataset \textsf{FreshRetailNet-50K} include hierarchical product/location identifiers
(\texttt{city\_id}, \texttt{store\_id}, \texttt{third\_level}--\texttt{product\_id});
sales and inventory metrics
(\texttt{hours\_sale}, \texttt{sale\_amount}, \texttt{stock\_hour6\_22}, \texttt{hours\_stock\_status});
and business, calendar, and weather covariates, including
\texttt{discount}, \texttt{activity\_flag}, \texttt{holiday\_flag},
\texttt{day-of-week}, \texttt{avg\_temperature}, \texttt{avg\_humidity},
\texttt{avg\_wind\_level}, and \texttt{precip}.

\paragraph{Assessment of Exchangeability.} Our conformal guarantees require that the calibration instances and the test instances be exchangeable (Assumption~\ref{asp:exchangeability}). In this case study, we support the plausibility of this requirement by restricting both calibration and evaluation to the same 7-day holdout window and treating each store-SKU pair's demand trajectory over the planning horizon, together with its associated covariates, as a single cross-sectional instance. Because the calibration/test split is randomized across a large pool of store-SKU pairs observed over the same calendar days and processed through the same forecasting and preprocessing pipeline, the resulting collection of instances has no intrinsic ordering and is plausibly permutation-invariant. Shared exogenous factors, such as city-level weather, holidays, or platform-wide promotions, can introduce dependence across pairs; however, such dependence is largely symmetric within the holdout window and therefore is consistent with exchangeability, which is weaker than i.i.d. sampling.

\paragraph{Prediction Details.}
Table~\ref{tab:forecasting_implementations} summarizes the implementation settings of the Temporal Fusion Transformer (\textsf{TFT}) and \textsf{DLinear} models used in the case study. The Similar Sample Average (\textsf{SSA}) baseline matches historical observations using recency, holiday status, day of week, precipitation, and discount information. Let $s_{ij}$ denote the resulting similarity score between target day $i$ and historical day $j$. The forecast is the softmax-weighted average $\hat d_i=(\sum_j\exp(s_{ij})d_j)/(\sum_\ell\exp(s_{i\ell}))$.

\begin{table}[htbp]
    \centering
    \small
    \caption{Forecasting implementations for the inventory case study.}
    \label{tab:forecasting_implementations}
    \renewcommand{\arraystretch}{1.15}
    \setlength{\tabcolsep}{6pt}
    \begin{tabularx}{\textwidth}{@{}lXX@{}}
        \toprule
        Method & Model specification & Training and forecasting settings \\
        \midrule
        \textsf{TFT}
        & Attention-based multi-horizon model using static and time-varying covariates; hidden size $64$, four attention heads, dropout $0.1$
        & Lookback $63$, horizon $7$, QuantileLoss, batch size $64$, $30$ epochs, Ranger optimizer \\

        \textsf{DLinear}
        & Trend--seasonal decomposition with moving-average window $28$ and seven input channels; dropout $0.05$
        & Lookback $62$, horizon $7$, mean absolute error loss, batch size $1024$, $6$ epochs \\
        \bottomrule
    \end{tabularx}
\end{table}

\subsubsection{Out-of-Sample Cost Reduction Relative to Baselines in \ConfRO{} Experiments.}
\label{app:case_study_relative_performance}
Table~\ref{tab:case_study_relative_performance} reports the out-of-sample costs of the best-performing \ConfRO{} specification and benchmark methods under different coverage levels.
\begin{table}[htbp]
\centering
\caption{Out-of-sample cost reduction of \ConfRO{} in the case study.}
\label{tab:case_study_relative_performance}
\small
\setlength{\tabcolsep}{4pt}
\begin{tabular}{cccccccc}
\toprule
\makecell{Coverage\\level}
& \makecell{Best\\\ConfRO{}}
& KNN-RO
& Ellipsoid-RO
& PTO
& \makecell{Imp. vs\\KNN-RO}
& \makecell{Imp. vs\\Ellipsoid-RO}
& \makecell{Imp. vs\\PTO} \\
\midrule
0.5 & 49.653 & 52.405 & 60.630 & 54.539 & 5.25\%  & 18.10\% & 8.96\% \\
0.6 & 49.672 & 53.328 & 61.697 & 54.539 & 6.86\%  & 19.49\% & 8.92\% \\
0.7 & 49.755 & 55.080 & 63.075 & 54.539 & 9.67\%  & 21.12\% & 8.77\% \\
0.8 & 49.938 & 57.933 & 65.132 & 54.539 & 13.80\% & 23.33\% & 8.44\% \\
0.9 & 50.318 & 63.192 & 68.884 & 54.539 & 20.37\% & 26.95\% & 7.74\% \\
\bottomrule
\end{tabular}
\end{table}

As Table~\ref{tab:case_study_relative_performance} shows, \ConfRO{} consistently achieves lower costs than the KNN-RO, Ellipsoid-RO, and PTO baselines across all coverage levels. The largest gains are relative to Ellipsoid-RO, highlighting the benefit of context-dependent, score-calibrated uncertainty sets. Moreover, the relative cost reduction generally increases with the target coverage level, indicating a larger advantage under more stringent robustness requirements.

\subsection{Real-Data Case Study on Inventory Management: \ConfRS{}}
\label{app:case_study_rs}
This section presents the two satisficing formulations used in Section~\ref{sec:cs_rs_exp}, describes the experimental pipeline, and derives their reformulations. The data, prediction results, and hyperparameter settings for the inventory problem are identical to those used in the \ConfRO{} experiments.

\subsubsection{Satisficing Formulations with Common $\ell_1$ Score.}
The \ConfRS{} model used in the experiment is given by \eqref{eq:ConfRS_Inv} with score $s(\bm w,\bm {\hat{w}^r})=\|\bm w-\bm{\hat{w}^r}\|_1$, where $r$ denotes one of the three demand forecasting methods \textsf{TFT}, \textsf{DLinear}, and \textsf{SSA}.

For DRS, let $\widehat P=S^{-1}\sum_{s=1}^{S}\delta_{\bm w^s}$ be the empirical distribution of the historical demand trajectories, and let $W_1$ denote the 1-Wasserstein distance induced by $d(\bm w,\bm a)=\norm{\bm w-\bm a}_1$. Writing $\mathcal P(\mathfrak{D})=\{P:\operatorname{supp}(P)\subseteq\mathfrak{D}\}$, the matched DRS model is
\begin{equation*}
\begin{aligned}
  \min_{\bm u,\bm t,\bm q}\quad
      &\sum_{k=0}^{T-1}q_k\\
  \text{s.t.}\quad
      &\sum_{k=0}^{T-1}(cu_k+t_k)\le\tau,\\
      &\sup_{P\in\mathcal P(\mathfrak{D})}
        \left\{\mathbb{E}_P[R_k(\bm u,\bm W)]
                    -q_kW_1(P,\widehat P)\right\}\le t_k,
        &&k=0,\ldots,T-1,\\
      &\bm u,\bm t,\bm q\ge0.
\end{aligned}
\tag{DRS-Inv}\label{eq:inventory_drs}
\end{equation*}
Thus, the two models use the same target, stage allocation, and unscaled $\ell_1$ deviation; they differ only in whether fragility is anchored at a context-specific point forecast or at the historical empirical distribution. In both cases, the system-level fragility is $K=\sum_{k=0}^{T-1}q_k$.

\subsubsection{Degeneracy of the Generic Formulation and Support Augmentation.}
\label{app:rs_support}
Consider first the unbounded domain $\mathbb{R}_+^T$. For any reference anchor $\bm a$ and coordinate $j\le k$, let $\bm w(\lambda)=\bm a+\lambda\bm e_j$. As $\lambda\to\infty$,
\begin{equation}
    R_k(\bm u,\bm w(\lambda)) -q_k\|\bm w(\lambda)-\bm a\|_1 =(p-q_k)\lambda+O(1).
    \label{eq:rs_unbounded_tail}
\end{equation}
Hence, a finite stage certificate requires $q_k\ge p$. Since $R_k(\bm u,\cdot)$ is $p$-Lipschitz under the $\ell_1$ metric ($p\ge h$), $q_k=p$ is sufficient whenever the corresponding zero-distance reference cost satisfies the target. Therefore, every certificate-feasible instance on the unbounded domain satisfies $K^*=\sum_{k=0}^{T-1}q_k^*=Tp$, so the generic formulation becomes uninformative. We therefore restrict the certificates to a prespecified operational envelope. For store--SKU $i$, let $w^{\mathrm{hist}}_{id}$ denote the latent demand on historical day $d$, and define
\begin{equation}
    M_i=\max_{d=1,\ldots,90}w^{\mathrm{hist}}_{id},
    \qquad
    \overline w_{ij}=\max\{\hat w_{ij},M_i\},
    \qquad
    \mathfrak{D}_i=\prod_{j=0}^{T-1}[0,\overline w_{ij}].
    \label{eq:item_max_support}
\end{equation}
This envelope uses only pre-test information and contains both the current forecast and the historical DRS anchors by construction. The construction introduces no additional tuned support parameter. It is neither a conformal prediction region nor a claim about the true demand support.

\subsubsection{Reformulations of the \ConfRS{} and DRS Models.}
\label{app:rs_reformulations}

\paragraph{Common Inner Problem.}
Suppress the instance and predictor indices, and define
\begin{equation}
    m_k=x_0+\sum_{j=0}^{k}u_j,\qquad
    A_k^a=\sum_{j=0}^{k}a_j,\qquad
    \overline A_k=\sum_{j=0}^{k}\overline w_j.
    \label{eq:rs_prefix_notation}
\end{equation}
For an anchor $\bm a\in\mathfrak{D}$ and $q\ge0$, separability of the $\ell_1$ distance gives, for every $j\le k$,
\begin{align}
  \sup_{0\le w_j\le\overline w_j}
    \{-hw_j-q|w_j-a_j|\}
    &=\max\{-ha_j,-qa_j\},\label{eq:rs_holding_endpoint}\\
  \sup_{0\le w_j\le\overline w_j}
    \{pw_j-q|w_j-a_j|\}
    &=\max\{pa_j,
       p\overline w_j-q(\overline w_j-a_j)\}.
  \label{eq:rs_backlog_endpoint}
\end{align}
The coordinates $j>k$ can be set to $a_j$ without changing the stage cost or incurring a distance penalty.  The slope comparison is common to every coordinate in the prefix, so summing
\eqref{eq:rs_holding_endpoint}--\eqref{eq:rs_backlog_endpoint} yields
\begin{equation}
\label{eq:rs_inner_identity}
\begin{split}
  &\sup_{\bm w\in\mathfrak{D}}
    \{R_k(\bm u,\bm w)-q\|\bm w-\bm a\|_1\}\\
  &\quad=\max\bigl\{
       hm_k-hA_k^a,
       hm_k-qA_k^a,
       -pm_k+pA_k^a,
       -pm_k+p\overline A_k-q(\overline A_k-A_k^a)
       \bigr\}.
\end{split}
\end{equation}

\paragraph{\ConfRS{} Reformulation.}
For predictor $r$, let $\widehat A_k^{\,r}=\sum_{j=0}^{k}\widehat w_j^{\,r}$.  Substituting $\bm a=\widehat{\bm w}^{\,r}$ into \eqref{eq:rs_inner_identity} shows that \eqref{eq:ConfRS_Inv} is exactly the linear program
\begingroup
\allowdisplaybreaks[4]
\begin{align}
  \min_{\bm u,\bm t,\bm q}\quad
    &\sum_{k=0}^{T-1}q_k
    \label{eq:confrs_exact_lp}\\
  \text{s.t.}\quad
    &\sum_{k=0}^{T-1}(cu_k+t_k)\le\tau,\notag\\
    &t_k\ge hm_k-h\widehat A_k^{\,r},
       &&k=0,\ldots,T-1,\notag\\
    &t_k\ge hm_k-q_k\widehat A_k^{\,r},
       &&k=0,\ldots,T-1,\notag\\
    &t_k\ge-pm_k+p\widehat A_k^{\,r},
       &&k=0,\ldots,T-1,\notag\\
    &t_k\ge-pm_k+p\overline A_k
                  -q_k(\overline A_k-\widehat A_k^{\,r}),
       &&k=0,\ldots,T-1,\notag\\
    &\bm u,\bm t,\bm q\ge0.\notag
\end{align}
\endgroup

\paragraph{DRS Reformulation.}
On the compact support $\mathfrak{D}$, the penalized Wasserstein identity gives
\begin{equation}
\begin{split}
  &\sup_{P\in\mathcal P(\mathfrak{D})}
  \left\{\mathbb{E}_P[R_k(\bm u,\bm W)]
       -q_kW_1(P,\widehat P)\right\}=\frac1S\sum_{s=1}^{S}
      \sup_{\bm w\in\mathfrak{D}}
      \left\{R_k(\bm u,\bm w)
       -q_k\|\bm w-\bm w^s\|_1\right\}.
\end{split}
\label{eq:drs_wasserstein_identity}
\end{equation}
Let $A_{ks}=\sum_{j=0}^{k}w_j^s$. Applying \eqref{eq:rs_inner_identity} to each anchor sample and introducing an epigraph variable $\zeta_{ks}$ for each maximum of samples gives the exact linear program
\begingroup
\allowdisplaybreaks[4]
\begin{align}
  \min_{\bm u,\bm t,\bm q,\bm\zeta}\quad
    &\sum_{k=0}^{T-1}q_k
    \label{eq:drs_exact_lp}\\
  \text{s.t.}\quad
    &\sum_{k=0}^{T-1}(cu_k+t_k)\le\tau,\notag\\
    &\frac1S\sum_{s=1}^{S}\zeta_{ks}\le t_k,
       &&k=0,\ldots,T-1,\notag\\
    &\zeta_{ks}\ge hm_k-hA_{ks},
       &&k=0,\ldots,T-1,\ s=1,\ldots,S,\notag\\
    &\zeta_{ks}\ge hm_k-q_kA_{ks},
       &&k=0,\ldots,T-1,\ s=1,\ldots,S,\notag\\
    &\zeta_{ks}\ge-pm_k+pA_{ks},
       &&k=0,\ldots,T-1,\ s=1,\ldots,S,\notag\\
    &\zeta_{ks}\ge-pm_k+p\overline A_k
                  -q_k(\overline A_k-A_{ks}),
       &&k=0,\ldots,T-1,\ s=1,\ldots,S,\notag\\
    &\bm u,\bm t,\bm q,\bm\zeta\ge0.\notag
\end{align}
\endgroup
Note that for each $(k,s)$, we compute the maximum over the holding and backlog branches before taking the sample average. This specific order is essential, as swapping the average and the maximum would fundamentally alter the objective.

\subsubsection{Numerical Results and Analysis.}
\label{app:rs_results}

\paragraph{Metrics.}
We evaluate performance using realized cost and feasibility rate.
Since replenishment has no capacity upper bound, the inventory problem itself is always physically feasible. The reported feasibility rate instead measures the fraction of the 500 instances for which a method admits an optimal finite-$K$ certificate. For \ConfRS{}, this is equivalent to the optimized inventory cost at the corresponding forecast being no greater than $\tau_i$; for DRS, it is equivalent to the optimized empirical mean cost over the 12 historical trajectories being no greater than $\tau_i$. These equivalences follow by evaluating the certificate at zero distance for necessity and taking $q_k=p$ for sufficiency.
As for realized cost, we compare it using the common feasible set
\begin{equation}
  \mathcal I (\rho)
    =\{i:\;\text{all \ConfRS{} variants and DRS are certificate feasible at }\tau_i(\rho)\},\quad N(\rho)=|\mathcal I(\rho)|.
  \label{eq:rs_common_feasible_set}
\end{equation}
For method $m$, the table reports the mean of $C(\bm u_i^m,\bm w_i^{\mathrm{test}})/V_i^0$ over this common set and the improvement over the DRS model.

\paragraph{Results.}
Table~\ref{tab:case_study_rs_detailed} reports the complete comparison. The feasibility rate is computed separately for each method over all 500 instances; realized cost and improvement use only the common set $\mathcal I(\rho)$.

\begin{table}[htbp]
  \centering
  \caption{The feasibility and realized performance of the three \ConfRS{} variants and DRS.}
  \label{tab:case_study_rs_detailed}
  \small
  \setlength{\tabcolsep}{3.0pt}
  \begin{tabular}{cc cccc cccc ccc}
    \toprule
    & & \multicolumn{4}{c}{Feasible ratio (\%)}
      & \multicolumn{4}{c}{Realized relative cost}
      & \multicolumn{3}{c}{Improvement over DRS (\%)}\\
    \cmidrule(lr){3-6}\cmidrule(lr){7-10}\cmidrule(lr){11-13}
    $\rho$ & $N(\rho)$
      & \textsf{TFT} & \textsf{DLinear} & \textsf{SSA} & DRS
      & \textsf{TFT} & \textsf{DLinear} & \textsf{SSA} & DRS
      & \textsf{TFT} & \textsf{DLinear} & \textsf{SSA}\\
    \midrule
    1.05 &  71 & 67.2 & 63.4 & 72.6 & 14.8
         & \textbf{1.461} & 1.561 & 1.642 & 2.322
         & \textbf{37.1} & 32.8 & 29.3\\
    1.10 &  93 & 73.6 & 70.4 & 78.0 & 19.0
         & \textbf{1.419} & 1.497 & 1.573 & 2.174
         & \textbf{34.7} & 31.1 & 27.6\\
    1.15 & 113 & 78.0 & 75.8 & 84.2 & 22.8
         & \textbf{1.399} & 1.448 & 1.519 & 2.067
         & \textbf{32.3} & 29.9 & 26.5\\
    1.20 & 140 & 81.8 & 79.4 & 87.0 & 28.4
         & \textbf{1.388} & 1.427 & 1.499 & 1.977
         & \textbf{29.8} & 27.8 & 24.2\\
    1.30 & 200 & 90.0 & 86.8 & 92.6 & 40.2
         & \textbf{1.384} & 1.393 & 1.438 & 1.813
         & \textbf{23.7} & 23.2 & 20.7\\
    1.40 & 258 & 92.4 & 90.8 & 95.4 & 52.0
         & 1.429 & \textbf{1.422} & 1.441 & 1.698
         & 15.8 & \textbf{16.2} & 15.1\\
    \bottomrule
  \end{tabular}
  \par\smallskip
  \parbox{\textwidth}{\scriptsize \textit{Note.} \textsf{TFT}, \textsf{DLinear}, and \textsf{SSA} denote \ConfRS{} using the corresponding point forecasts. The feasible rate is computed separately for each method over all 500 instances and records the existence of an optimal finite-$K$ certificate. Realized relative costs are sample means over $\mathcal I(\rho)$ and are normalized by $V_i^0$.}
  \vspace{-1em}
\end{table}

The {feasibility} increases monotonically as the target is relaxed, and all three \ConfRS{} variants substantially outperform DRS across all target ratios. However, prediction accuracy does not directly determine feasibility: the \textsf{SSA}-based variant achieves the highest feasibility rate.
On $\mathcal I(\rho)$, every \ConfRS{} variant also achieves a lower mean {realized relative cost} than DRS. The improvement ranges from $29.3\%$--$37.1\%$ at $\rho=1.05$ and remains $15.1\%$--$16.2\%$ at $\rho=1.40$. The most accurate predictor, \textsf{TFT}, yields the lowest realized cost at most target levels, while the differences between the three variants narrow as the target relaxes. Overall, these results show that prediction-centered satisficing consistently outperforms the historical empirical reference across forecasting methods.

\subsection{Robust Facility Location: An Additional Synthetic Study}\label{app:location}

\subsubsection{Setting and Experimental Design.}
Based on the setup in \cite{baron2011facility}, we further consider a multiperiod facility location problem with candidate facilities $i \in [I]$, demand zones $j \in [J]$, and periods $t\in [T]$.

Contextual features $\bm{z}_t$ inform the uncertain demand $\bm{d}_t$ in each period. Decisions comprise facility openings $\bm{y}\in \{0,1\}^I$, with fixed costs $f_i$ and operational allocations $\bm{x}=(x_{ijt})$, where $x_{ijt}$ is the fraction of demand in zone $j$ served by facility $i$ in period $t$, with unit cost $c_{ij,t}$. Let $\mathcal{C}_t(\alpha)\subseteq \mathbb{R}^J$ denote the \emph{conformal uncertainty set} for the demand forecast $\hat{\bm{d}}_t$ at coverage level $\alpha$ in period $t$. The problem is formulated as
\begin{align*}
    \min_{\tau, \bm{y},\bm{x}} &\; \tau\\
    \text{subject to } &\;
    \tau \ge \max_{\substack{ (\bm{d}_1, \dots, \bm{d}_T) \in \Pi_{t=1}^T \mathcal{C}_t(\alpha)}} \left\{ \sum_{i \in [I]} f_i y_i +  \sum_{t=1}^{T} \sum_{i \in [I]} \sum_{j \in [J]} c_{ij,t} d_{jt} x_{ij,t} \right\}\\
    &\;  \sum_{i \in [I]} x_{ij,t} \ge 1, \quad \forall j \in [J],\; t=1,\dots,T\\
    &\;  \sum_{j \in [J]} x_{ij,t} d_{j,t} \le M_{i,t} \, y_i, \quad \forall \bm{d}_t\in \mathcal{C}_t(\alpha),\  \forall i \in [I],\; t=1,\dots,T\\
    &\;  y_i \in \{0,1\}, \quad \forall i \in [I]\\
    &\;  x_{ij,t} \ge 0, \quad \forall i \in [I],\; j \in [J],\; t=1,\dots,T.
\end{align*}

We simulate $I=10$ candidate facilities and $J=30$ demand zones over $T=5$ periods. Their locations are sampled independently and uniformly from the unit square. Let $\ell_{ij}$ denote the Euclidean distance between facility $i$ and demand zone $j$. For every $t\in[T]$, we set capacity coefficient to $M_{i,t} = 2,000$, the time-invariant service cost to
$c_{ij,t}=10(\ell_{ij}+0.1)\delta_{ij}$, where $\delta_{ij}\overset{\mathrm{i.i.d.}}{\sim}\mathrm{U}(0.8,1.2)$, and the fixed facility cost to $f_i=3000\gamma_i$, where $\gamma_i\overset{\mathrm{i.i.d.}}{\sim}\mathrm{U}(0.6,1.4)$. We generate features $\bm{z}\sim N(\bm{0},\bm{I}_{n_d})\in \mathbb{R}^{n_d}$ with $n_d=20$, and demand according to $d_j(\bm{z})=(\frac{1}{\sqrt{n_d}}(\bm{B}\bm{z})_j+3)^5\xi_j/5$, where $\bm B\in\mathbb{R}^{J\times n_d}$ with $B_{jk}\overset{\mathrm{i.i.d.}}{\sim}\mathrm{U}(-1,1)$ for $k\le n_d-2$ and $B_{jk}=0$ otherwise,  $\xi_j\sim U(1-\epsilon,1+\epsilon)$,  and $\epsilon=0.1$. We generate 20,000 samples and split them into training, uncertainty quantification, and testing sets in a 4:2:2 ratio.

\subsubsection{Numerical Results.}
Table~\ref{tab:coverage_comparison} compares realized out-of-sample coverage across the six methods. The three \ConfRO{} variants closely track the target at every value of $\alpha$, whereas KNN-RO and KMeans-RO generally under-cover, particularly at the higher target levels.

\begin{table}[htb]
    \centering
    \begin{minipage}[t]{0.51\textwidth}
        \vspace{0pt}
        \centering
        \captionof{table}{Out-of-sample coverage by method.}
        \label{tab:coverage_comparison}
        \begingroup
        \scriptsize
        \setlength{\tabcolsep}{1.5pt}
        \renewcommand{\arraystretch}{1.05}
        \begin{tabular}{@{}ccccccc@{}}
            \toprule
            $\alpha$ & \makecell{\ConfRO{}-\\Box} & \makecell{\ConfRO{}-\\Ellipsoid} & \makecell{\ConfRO{}-\\Budget} & \makecell{Ellipsoid-\\RO} & \makecell{KNN-\\RO} & \makecell{KMeans-\\RO} \\
            \midrule
            0.60 & 0.600 & 0.601 & 0.596 & 0.579 & 0.599 & 0.525 \\
            0.70 & 0.719 & 0.704 & 0.702 & 0.690 & 0.692 & 0.631 \\
            0.80 & 0.807 & 0.801 & 0.804 & 0.790 & 0.778 & 0.736 \\
            0.85 & 0.860 & 0.848 & 0.857 & 0.843 & 0.819 & 0.797 \\
            0.90 & 0.909 & 0.906 & 0.909 & 0.900 & 0.863 & 0.850 \\
            0.95 & 0.948 & 0.950 & 0.948 & 0.947 & 0.907 & 0.911 \\
            \bottomrule
        \end{tabular}
        \endgroup
    \end{minipage}
    \hfill
    \begin{minipage}[t]{0.47\textwidth}
        \vspace{0pt}
        \centering
        \captionof{table}{Value of prediction in \ConfRO{}.}
        \label{tab:prediction_effectiveness}
        \begingroup
        \scriptsize
        \setlength{\tabcolsep}{1pt}
        \renewcommand{\arraystretch}{1.05}
        \begin{tabular}{@{}lrrr@{}}
            \toprule
            {Prediction method} & {MSE} & \makecell{Actual\\coverage} & \makecell{Mean\\objective} \\
            \midrule
            KRR-RBF            & 10.17  & 0.907 & 43,319.4 \\
            KRR-Poly           & 21.12  & \textbf{0.901} & \textbf{42,354.6} \\
            MLP               & 40.11  & 0.896 & 42,389.1 \\
            SVR               & 151.90 & 0.895 & 42,511.6 \\
            OLS               & 757.61 & 0.899 & 48,916.0 \\
            LASSO             & 768.82 & 0.915 & 49,514.9 \\
            \bottomrule
        \end{tabular}
        \endgroup
    \end{minipage}
    \vspace{-1em}
\end{table}

At $90\%$ target coverage, Table~\ref{tab:prediction_effectiveness} compares six upstream predictors within \ConfRO{}: ordinary least squares (OLS), the least absolute shrinkage and selection operator (LASSO), support vector regression (SVR), a multilayer perceptron (MLP), and kernel ridge regression (KRR) with radial basis function (RBF) and polynomial kernels (KRR-RBF and KRR-Poly, respectively). Hyperparameters for the tunable predictors are selected by five-fold cross-validation. KRR-RBF has the lowest prediction MSE, whereas KRR-Poly yields the lowest mean robust objective while attaining $0.901$ actual coverage. We therefore use KRR-Poly in the main facility-location comparison. This result illustrates that downstream decision quality is not determined by prediction MSE alone.

For each uncertain objective or capacity constraint of the form $\bm{a}^{\top}\bm{d}\le b$, the Box, Ellipsoid, and Budget score sets are implemented through their standard support functions. The resulting robust counterparts are linear for Box and Budget sets and second-order conic for Ellipsoid sets; the same support-function forms are summarized in the inventory case study.

\begin{figure}[htb]
    \centering
    \caption{Performance gaps of different methods under various coverage levels.}
    \label{fig:performance_gap_cro}
    \includegraphics[width=0.7\textwidth]{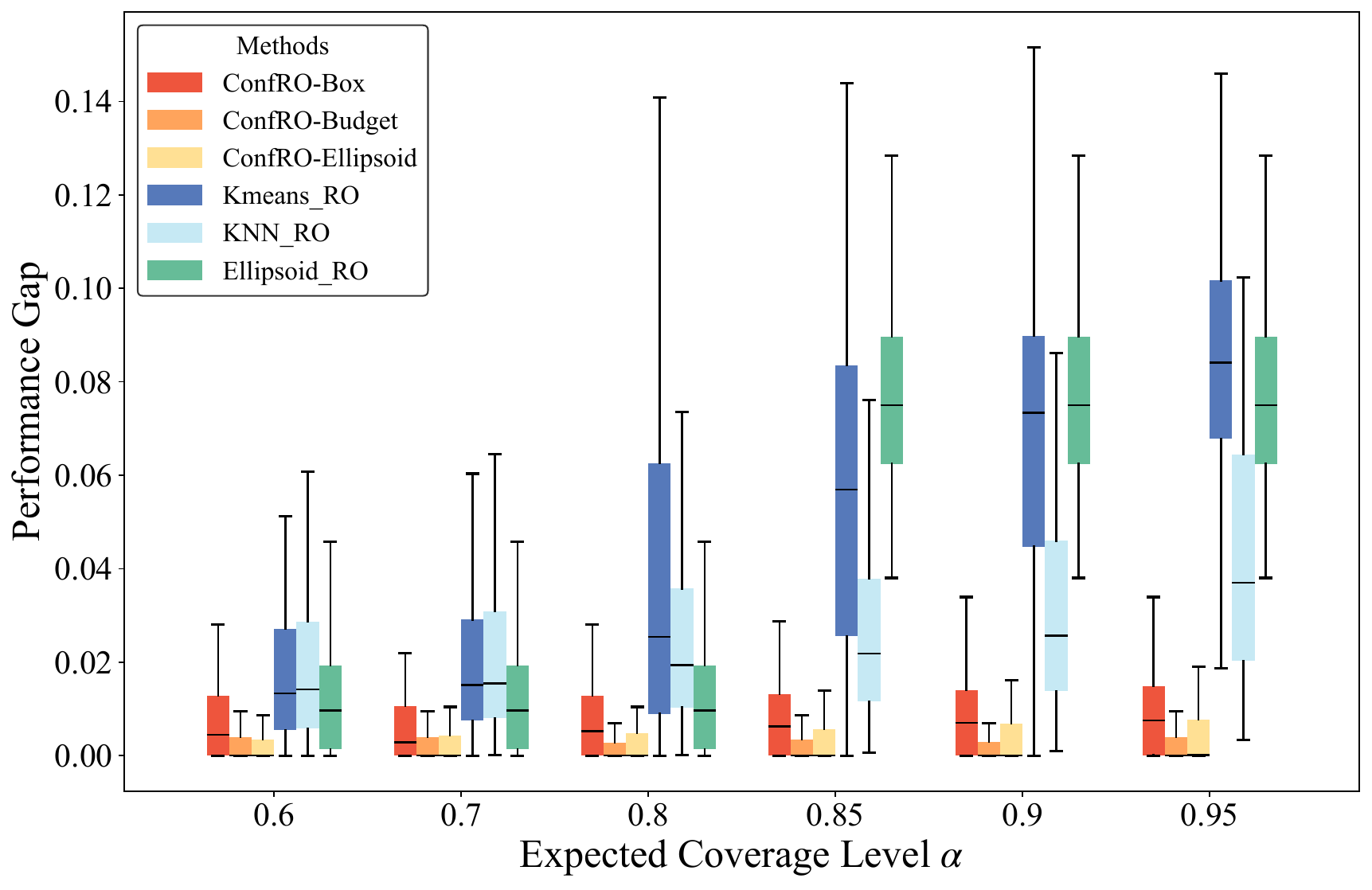}
\end{figure}

To compare decision performance, we define the suboptimality gap as $\Delta=(V_m-V_D)/V_D$, where $V_m$ is the realized cost of method $m$ and $V_D$ is the deterministic cost under realized demand. Figure~\ref{fig:performance_gap_cro} shows that \ConfRO{} achieves lower mean and dispersion of the suboptimality gap across the considered coverage levels.

\end{APPENDICES}

\end{document}